\documentclass{article} 
\usepackage{iclr2027_conference,times}

\usepackage{amsmath,amsfonts,bm}

\def\eqref#1{equation~\ref{#1}}

\def\1{\bm{1}}

\DeclareMathAlphabet{\mathsfit}{\encodingdefault}{\sfdefault}{m}{sl}
\SetMathAlphabet{\mathsfit}{bold}{\encodingdefault}{\sfdefault}{bx}{n}

\usepackage{hyperref}
\usepackage{url}
\usepackage{amsmath,amssymb}
\usepackage{amsthm}
\usepackage{xcolor}
\usepackage{tabularx}
\usepackage{booktabs}
\usepackage{graphicx}
\usepackage{subcaption}

\newcolumntype{Y}{>{\raggedright\arraybackslash}X}

\theoremstyle{plain}
\newtheorem{theorem}{Theorem}
\newtheorem{proposition}{Proposition}

\newtheorem*{prop1*}{Proposition 1}
\newtheorem*{prop2*}{Proposition 2}
\newtheorem*{thm1*}{Theorem 1}
\newtheorem*{prop3*}{Proposition 3}
\newtheorem*{prop4*}{Proposition 4}
\makeatletter
\newcommand{\fixedlabel}[2]{%
  \begingroup
  \def\@currentlabel{#1}%
  \label{#2}%
  \endgroup}
\makeatother
\theoremstyle{definition}

\theoremstyle{remark}
\newtheorem{remark}{Remark}

\graphicspath{{figures/}}

\title{Learning PDE Dynamics between Submanifolds Using Green's Observation Operators}

\author{
\textbf{Jan Tauberschmidt}\textsuperscript{1,2},\hspace{0.6em}
\textbf{Jephte Abijuru}\textsuperscript{2}\thanks{Equal contribution; ordered alphabetically.},\hspace{0.6em}
\textbf{Naukshatro Bose}\textsuperscript{2}\footnotemark[1],\hspace{0.6em}
\textbf{Samuel Okon}\textsuperscript{1,2}\footnotemark[1],
\\
\textbf{Sophie Fellenz}\textsuperscript{2},\hspace{0.6em}
\textbf{Marius Kloft}\textsuperscript{1,2},\hspace{0.6em}
\textbf{Jonas Latz}\textsuperscript{3},\hspace{0.6em}
\textbf{Sebastian Josef Vollmer}\textsuperscript{1,2}
\\
\textsuperscript{1}German Research Center for Artificial Intelligence (DFKI),
Kaiserslautern, Germany
\\
\textsuperscript{2}Rheinland-Pf\"alzische Technische Universit\"at (RPTU),
Kaiserslautern, Germany
\\
\textsuperscript{3}University of Manchester,
Manchester, United Kingdom
}

\iclrfinalcopy 
\begin{document}

\maketitle

\begin{abstract}
Many physical systems are driven and observed only on lower-dimensional submanifolds of a larger spatial domain, while their dynamics are governed by the ambient medium occupying that domain. Examples include laser-heated parts imaged by an infrared camera, and ground-level emissions measured on a sensor plane. Full-domain solvers, however, compute the entire volume for every new source although only the observation submanifold is needed, and black-box surrogates do not exploit that the ambient medium remains fixed. We introduce the \emph{Green's Observation Operator (GObO)}, which maps the ambient medium once to the Green's kernel of a linear PDE restricted to the source and observation submanifolds. New sources then cost one lower-dimensional integral and no network evaluation. Exponential rates in the kernel yield an exact finite streaming state with horizon-independent memory; we prove its stability and an approximation rate for the restricted heat kernel. On three-dimensional heat conduction and advection--diffusion with collocated and distinct source and observation geometries, GObO trained on static sources predicts responses to moving sources zero-shot with 4--8$\times$ lower error than black-box surrogates, at 1.4\,ms per query after a single conditioning pass.
The same kernel transfers across resolutions and admits corrections for mild nonlinearities, including radiative losses and temperature-dependent conductivity, without retraining, at the cost of lower in-distribution accuracy.
\end{abstract}

\section{Introduction}
\label{sec:intro}

Many systems in science and engineering are described by partial
differential equations (PDEs). Examples include heat flow in manufactured
parts, airflow around vehicles, blood flow in arteries, or acoustic wave
propagation. In many applications, however, the full state inside the
domain is not observed. Instead, measurements are available only on a
lower-dimensional part of the domain, such as a cross-section, a boundary face, or more generally, a \emph{submanifold}.
Likewise, the input that drives the system, represented by a forcing, may act only on a lower-dimensional region, such as a heated
surface, an area on the ground, or a boundary.

Laser-based metal additive manufacturing gives a concrete example of this setting, where laser locally heats and melts the material \citep{sajadi2025temperature} on a surface, while an infrared camera observes only a lower-dimensional temperature field. Thus, both the input and the observed output are lower-dimensional, but their relation still depends on the full material through which heat travels. Figure~\ref{fig:Fig1} illustrates this setting schematically, with a laser acting on a source submanifold and the resulting temperature being observed on an observation submanifold coupled through a heterogeneous ambient medium.

\begin{figure}
    \centering
    \includegraphics[width=\linewidth]{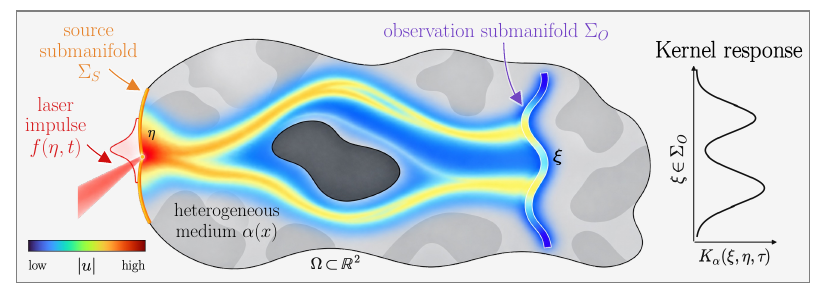}
    \caption{A laser acts on a source submanifold and drives an observed temperature response on the observation submanifold. The ambient medium affects the temperature flux.
    Classical PDE solvers model the transport through the full volume, whereas our GObO directly models the map between source and observation submanifolds.}
    \label{fig:Fig1}
\end{figure}

Similarly, in atmospheric transport, pollutants or
greenhouse gases may be released from an area on the ground and then
carried through the three-dimensional atmosphere by wind and diffusion.
Measurements, however, might only cover a two-dimensional concentration field on a vertical plane \citep{hashmonay2001field}.
Related settings include surface-pressure prediction in
aerodynamics~\citep{li2023gino}, wall-shear-stress prediction in
hemodynamics~\citep{suk2024mesh}, or nearfield acoustic
holography~\citep{olivieri2021physics}.

In these settings, the mapping between lower-dimensional source and observation submanifolds is still governed by the full ambient medium. The restricted state is therefore generally not closed and cannot be modeled independently of the surrounding domain. Classical PDE solvers retain this dependence by resolving the full ambient state, but consequently compute throughout the volume even when only the response on the observation submanifold is needed. 
This additional computation increases prediction time, which can limit its use in online prediction or planning settings.

We introduce the \emph{Green's Observation Operator} (GObO), which learns a time-dependent Green's kernel directly between the source and observation submanifolds, avoiding full-domain computation. GObO enables rapid prediction under many different source terms and generalizes across source configurations beyond those observed during training.
Our main contributions are:
\begin{itemize}
\item We introduce the \textbf{Green's Observation Operator (GObO)}, a time-dependent Green's kernel model acting directly between source and observation submanifolds conditioned on the ambient medium.

\item We develop a \textbf{structured kernel parametrization} and prove its exponential approximation rate for the restricted heat kernel (Theorem~\ref{thm:spectral-approx}), together with exact finite-state recurrent evaluation and its stability (Propositions~\ref{prop:streaming}--\ref{prop:streaming-stability}).

\item We \textbf{benchmark GObO} against PDE solvers, neural operators, and learned Green's function models. It achieves $4$--$8\times$ lower error under source generalization, up to $3.5\times$ lower error than the learned Green's function baseline, and over $100\times$ faster evaluation of new sources than black-box neural operators.

\end{itemize}

\section{Prediction between submanifolds needs a tailored approach}
\label{sec:submanifold-setting}
In this section, we formalize prediction between source and observation submanifolds and motivate a model that works directly on the submanifolds while retaining the ambient dynamics.

\subsection{Submanifolds of closed systems are not closed}
We study a map between \emph{lower-dimensional} submanifolds whose output, unlike the full PDE state, does not form a closed dynamical system.
We consider a time-dependent PDE on an ambient spatial domain
\(\Omega\subset\mathbb{R}^d\) over a time interval \([0,T]\). Its solution is
\(u_\alpha(x,t)\), while
\(\alpha:\Omega\to\mathbb{R}^p\) describes the known medium and \(f\) drives the system. With differential operator \(\mathcal{P}_{\alpha}\), we write the PDE as
\begin{equation}
\mathcal{P}_\alpha u_\alpha=f(x,t),
\qquad (x,t)\in\Omega\times(0,T],
\label{eq}
\end{equation}
together with suitable initial and boundary conditions.
The source \(f\) acts on a source submanifold
\(\Sigma_S \subseteq \Omega\), i.e. $f|_{\Omega \backslash \Sigma_S} \equiv 0$. The solution is observed on a possibly different
observation submanifold \(\Sigma_O \subseteq \Omega\). Denoting the observed trace by
\(v_\alpha:=u_\alpha|{\Sigma_O}\), we seek the map
\begin{equation}
\mathcal{S}_\alpha:\quad
f
\longmapsto
v_\alpha.
\label{eq:restricted-solution-map}
\end{equation}
Although both source and output are lower-dimensional, this map still depends on the full ambient medium.
Therefore, the observed trace on \(\Sigma_O\) is generally not closed. While the full state
\(u_\alpha(\cdot,t)\) determines its future, its restriction \(v_\alpha(\cdot,t)\) generally does not. Information can propagate through the unobserved domain and later influence \(\Sigma_O\) again. Consequently, the surrounding volume therefore acts as a hidden memory.
We next consider two approaches to this problem. Classical PDE solvers compute the solution in the full ambient domain, whereas data-driven methods seek to learn the source-to-observation map directly.

\subsection{PDE solvers require closed systems}



Classical PDE solvers obtain the observation trace by first solving the closed dynamical system in the full ambient domain. They typically compute \(u_\alpha\) by discretizing the governing equations throughout the full volume \citep{leveque2007finite}. The resulting state can then be restricted to \(\Sigma_O\), but this requires computing the full solution even when only the lower-dimensional trace is needed. Moreover, because the trace is not a closed state, each new source generally requires solving the ambient PDE again. This overhead accumulates across repeated predictions and, on the heat benchmark of Section~\ref{sec:experiments}, makes the finite-volume solver about \(2750\times\) slower than our model, as shown in Table~\ref{tab:timing-1k}.

\begin{table}[ht]
\centering
\footnotesize
\setlength{\tabcolsep}{4pt}
\begin{tabular}{lcc}
\toprule
 & Ours & Finite-volume solver \\
\midrule
Time for 1000 source predictions/solves & $\mathbf{1.7}$ s & $1.3$ h \\
\bottomrule
\end{tabular}
\caption{Classical solvers repeat a volumetric solve for every source although only the lower-dimensional trace is required, making repeated prediction expensive. See details in Section~\ref{sec:exp_reuse}.}
\label{tab:timing-1k}
\end{table}

\subsection{Black-box approaches miss key operator structure or dynamics}

General black-box surrogates such as DeepONet and the Fourier Neural Operator
(FNO) can learn mappings between function spaces, including on submanifolds,
but miss key structure of our setting
\citep{lu2021learning,li2021fourier,kovachki2023neural}. They can be conditioned
on the ambient medium \(\alpha\), but process \(\alpha\) and the source \(f\)
jointly through a nonlinear mapping. Each new source therefore requires another
network evaluation, even when the medium is unchanged, and generalization to
source configurations beyond those observed during training is not built into
the representation.

Neural Green's functions provide a more structured alternative by learning a
source-independent kernel that is applied to different sources
\citep{boulle2022learning,yoo2026neural}. This makes them well-suited to
generalization across sources. For transient problems, however, the Neural
Green's Operator (NGO) and DGNet learn a one-step kernel and roll it out
iteratively~\citep{melchers2026ngo,tan2026dgnet}. 
This requires a closed system
state, which \(v_\alpha\) does not provide, making these constructions not
directly applicable to our setting.

\subsection{Proposed GObO directly models non-closed submanifold dynamics}

The proposed Green's Observation Operator addresses these limitations by representing the effect of the ambient dynamics directly between the source and observation submanifolds. Its time-dependent kernel captures how the full medium couples the two submanifolds without reconstructing the volumetric state or requiring the observation trace to form a closed dynamical system. Once this medium-dependent representation is learned, new sources can be evaluated without another network pass.

We introduce a structured parametrization of this kernel in the following sections. Besides supporting efficient approximation, it enables a recurrent computation that compresses the system memory into a running state. We first introduce the Green's-function formulation underlying this construction.

\section{A primer on Green's functions}
\label{sec:background-green}
Green's functions provide a useful representation of the solution of linear PDEs.
We briefly review the ideas needed below and refer to
\citet{stakgold1979green} for a comprehensive treatment. In the notation of
Section~\ref{sec:submanifold-setting}, a Green's function is the response to
a unit impulse at location $y$ and time $s$:
\[
\mathcal P_{\alpha,(x,t)}
G_\alpha(x,t;y,s)
=
\delta(x-y)\delta(t-s).
\]
For a time-independent linear system, superposition gives the response to a
source $F$ as
\begin{equation}
\label{eq:green-volume}
    u_\alpha(x,t)
    =
    \int_0^t\int_\Omega
        G_\alpha(x,y,t-s)F(y,s)
    \,dy\,ds\,.
\end{equation}
This simplified display assumes zero initial data and homogeneous Neumann
boundary conditions. Appendix~\ref{app:green-proofs} develops the full
setting with initial and boundary inputs. The kernel $G_\alpha$ depends on
the medium $\alpha$, but not on $F$, so it can be reused across source
histories.

We now specialize the general problem in
Section~\ref{sec:submanifold-setting} to the parabolic PDE family studied
throughout this work:
\begin{equation}
\label{eq:parabolic}
    \partial_t u_\alpha
    +
    b(x)\cdot\nabla u_\alpha
    -
    \nabla\cdot
    \big(
        \kappa(x)\nabla u_\alpha
    \big)
    =
    F,
\end{equation}
Here $\alpha=(b,\kappa)$ collects the transport field $b$ and the spatially
varying diffusivity $\kappa$. Setting $b=0$ gives the heat equation, while
nonzero $b$ gives advection--diffusion.

\begin{proposition}[Existence of the Green's kernel, \emph{informal}]
\label{prop:green-existence}
Under the standard assumptions in Appendix~\ref{app:green-proofs}, for each
admissible $\alpha=(b,\kappa)$ the parabolic problem
\eqref{eq:parabolic} with homogeneous Neumann boundary conditions admits a
Green's kernel $G_\alpha$ that is continuous up to the boundary for positive
time.
\end{proposition}
The precise statement and proof are given in
Appendix~\ref{app:green-proofs}.
Under mild integrability assumptions on $F$, the integral in
\eqref{eq:green-volume} is finite \citep{daners2000heat}.
Because the PDE is linear, a single Green's kernel can predict the response
to any source in a fixed medium through \eqref{eq:green-volume}. However,
the kernel changes with the medium and is available analytically only in a few simple cases
\citep{davies1989heat}.
The next section therefore develops a method for
learning the map \(\alpha\mapsto G_\alpha\) in our setting of prediction between submanifolds.

\section{Green's Observation Operator}
\label{sec:method}
This section develops the Green's
Observation Operator (GObO).
Our construction first restricts the ambient Green kernel to the source and observation submanifolds and then introduces a tractable parametrization for learning and prediction.

\subsection{Restricting the ambient green kernel to submanifolds}

Rather than learning the ambient kernel $G_\alpha$ between arbitrary points in
$\Omega$, GObO learns only its restriction from sources on $\Sigma_S$ to
observations on $\Sigma_O$, as introduced in
Section~\ref{sec:submanifold-setting}.
To express this restricted map in
coordinates, we use bi-Lipschitz $C^1$ charts
\begin{equation}
\label{eq:manifolds}
    X_S:U_S\subset\mathbb{R}^{m_S}\to\overline{\Omega}
    \qquad\text{and}\qquad
    X_O:U_O\subset\mathbb{R}^{m_O}\to\overline{\Omega}
\end{equation}
that define a source manifold $\Sigma_S:=X_S(U_S)$ and an observation manifold $\Sigma_O:=X_O(U_O)$ respectively.
In these coordinates, restricting the observation variable is straightforward
and simply amounts to evaluating the solution on $\Sigma_O$, so the observed trace from
\eqref{eq:restricted-solution-map} becomes
\(v_\alpha(\xi,t) := u_\alpha(X_O(\xi),t)\).
Restricting the source variable requires a different construction, since
\(\Sigma_S\) generally is a null set in the ambient space and its contribution
would disappear under na\"ive integration in \eqref{eq:green-volume}.

Instead, we will follow the standard singular-source construction used in immersed-interface and immersed-boundary methods \citep{leveque1994immersed,peskin2002immersed,kandilarov2007immersed}.
First suppose that $\Sigma_S\subset\Omega$ and let $f_\Sigma(y,t)$ denote source
strength per unit surface measure. For smooth test functions
$\varphi \in C_{c}^\infty (\Omega)$, with $d\sigma$ denoting the surface
measure on $\Sigma_S$, we represent this source by the distribution
\begin{equation}
\label{eq:surface-source}
    F(\cdot,t)
    :=
    f_\Sigma(\cdot,t)\delta_{\Sigma_S},
    \qquad
    \langle F(\cdot,t),\varphi\rangle
    =
    \int_{\Sigma_S}
        f_\Sigma(y,t)\varphi(y)
    \,d\sigma(y),
\end{equation}
which can also be obtained as the distributional limit of smooth volumetric
sources supported in shrinking neighborhoods of $\Sigma_S$ \citep{friedlander1998introduction}.
Substituting \eqref{eq:surface-source} into \eqref{eq:green-volume} reduces the ambient to a surface integral.
In chart coordinates, and using the area formula, this can be written as
\begin{equation}
\label{eq:restricted-kernel}
    v_\alpha(\xi,t) = \int_0^t\int_{U_S} K_\alpha(\xi,\eta,t-s)\, f(\eta,s) \,d\eta\,ds,
\end{equation}
with \(f(\eta,t) := f_\Sigma(X_S(\eta),t)\) and surface element \(J_S(\eta) := \sqrt{\det\!\left(DX_S(\eta)^\top DX_S(\eta)\right)}\) directly absorbed into the kernel
\(K_\alpha(\xi,\eta,\tau):= G_\alpha\!\left( X_O(\xi),X_S(\eta),\tau \right)J_S(\eta).\)
Absorbing the surface element into \(K_\alpha\) keeps inference in the form of a standard chart integral.

If we instead consider \(\Sigma_S\subset\partial\Omega\) as an inward Neumann flux, we treat $F$ as a usual boundary condition.
Appendix~\ref{app:green-proofs} shows that, in the absence of volumetric forcing, this yields the same integral representation as in \eqref{eq:restricted-kernel}. Therefore, in the following, we use \eqref{eq:restricted-kernel} for both interior and boundary source manifolds.
The regularity result below provides the corresponding continuous extension of the kernel to the boundary, so we also allow \(\Sigma_O \subset \bar\Omega\).
It remains to verify that restricting the kernel to the measure-zero source
manifold preserves integrability.
\begin{proposition}[Integrability of the restricted kernel, \emph{informal}]
\label{prop:restricted-integrability}
For every finite $T>0$, the source integral of the restricted kernel is
bounded by $C_T\tau^{-1/2}$ for $0<\tau\leq T$, uniformly over the
observation manifold. In particular, the kernel is integrable over source
space and time.
\end{proposition}
The precise statement and proof are given in
Appendix~\ref{app:green-proofs}.
In particular, this justifies that the integral in \eqref{eq:restricted-kernel} exists in all settings considered below.

\subsection{Spectral separation of time and space in the restricted kernel}
\label{sec:kernel-approximation}
Spectral decomposition of the kernel provides a tractable parametrization,
which is needed because a naive pointwise parametrization is prohibitively
expensive. The latter would represent the kernel with a coordinate-based network
\(\bigl((\xi,\eta,\tau), \alpha\bigr)\mapsto K_\alpha(\xi,\eta,\tau)\in \mathbb R\)
which requires too
many network evaluations to compute the space--time integral.
For a 2D domain discretized with 64 points per
spatial axis and 64 time steps, for instance, a single prediction would already require
$64^2\times64^2\times64>10^9$ evaluations. Our
parametrization is motivated by a classical spectral decomposition, which also
separates time and space in the restricted kernel.

Indeed, for strongly elliptic second-order operators on bounded Lipschitz domains, spectral expansions in
eigenfunctions and generalized eigenfunctions are established by
\citet{agranovich2008spectral}. If the operator has no nontrivial
generalized eigenfunctions, this reduces to the biorthogonal
eigenfunction expansion
\begin{equation}
\label{eq:general-spectral}
    G_\alpha(x,y,\tau)
    =
    \sum_{k=1}^{\infty}
        e^{-\lambda_k\tau}
        r_k(x)\ell_k(y),
\end{equation}
where $r_k$ and $\ell_k$ are the corresponding right and left
eigenfunctions and eigenvalues \(\lambda_k\), which may be complex.
We restrict this expansion to the source and observation manifolds and,
to obtain a tractable parametrization for optimization, truncate the sum
after $M$ terms:
\begin{equation}
\label{eq:restricted-separable}
    K_\alpha(\xi,\eta,\tau)
    =
    \sum_{k=1}^{\infty}
        e^{-\lambda_k\tau}
        \psi_k(\xi)\varphi_k(\eta)
    \approx
    \sum_{k=1}^{M}
        e^{-\lambda_k\tau}
        \psi_k(\xi)\varphi_k(\eta)
    =:
    K_{\alpha,M}(\xi,\eta,\tau),
\end{equation}
with restricted spatial factors
\(\psi_k(\xi) := r_k(X_O(\xi))\)
and
\(\varphi_k(\eta) := \ell_k(X_S(\eta))J_S(\eta)\).
Thus, the source-side factor absorbs the surface element.

\subsection{Exponential convergence of spectral truncation for the heat kernel}

For the heat equation, this truncation is particularly well behaved, and its
approximation error decays exponentially in $M^{2/d}$ with the number $M$ of
retained terms. This behavior follows from the spectral structure of the heat
kernel. Its expansion holds pointwise, and the regular growth of the underlying
operator's eigenvalues causes higher modes to be rapidly suppressed at positive
times. Consequently, once we stay a positive time $\tau_0$ away from the
singularity at $\tau=0$, the restricted kernel satisfies the explicit
truncation bound stated next.

\begin{theorem}[Approximation of the restricted heat kernel, \emph{informal}]
\label{thm:spectral-approx}
Assume the setting of Propositions~\ref{prop:green-existence} and
\ref{prop:restricted-integrability} with $b=0$. For each fixed medium
$\alpha$ and every $\tau_0>0$, there are constants $C_{\tau_0},c>0$, with
$c$ independent of $\tau_0$, such that
\begin{equation}
\label{eq:spectral-bound}
    \sup_{\tau\geq\tau_0}
    \left\|
        K_\alpha(\cdot,\cdot,\tau)
        -
        K_{\alpha,M}(\cdot,\cdot,\tau)
    \right\|_{L^2(U_O\times U_S)}
    \leq
    C_{\tau_0}
    e^{-c\tau_0M^{2/d}}.
\end{equation}
\end{theorem}

The precise statement and proof are given in
Appendix~\ref{app:kernel-proof}.
This justifies the truncated representation in
\eqref{eq:restricted-separable} and leaves the spatial factors as the
remaining objects to represent.

\subsection{Finite-dimensional parametrization of the restricted kernel}
We approximate the remaining spatial factors of the truncated kernel in finite-dimensional
bases $\boldsymbol{\Phi}_S(\eta)\in\mathbb{R}^{D_S}$ and
$\boldsymbol{\Phi}_O(\xi)\in\mathbb{R}^{D_O}$ on the source and observation
submanifolds, respectively. For each retained mode, we represent the two spatial
factors by coefficient vectors
$\mathbf{v}_k(\alpha)\in\mathbb{R}^{D_S}$ and
$\mathbf{u}_k(\alpha)\in\mathbb{R}^{D_O}$. This gives
\begin{equation}
\label{eq:kernel-model}
    K_{\theta,\alpha}(\xi,\eta,\tau)
    =
    \sum_{k=1}^{M}
        e^{-\lambda_k\tau}
        \underbrace{
        \big(
            \boldsymbol{\Phi}_O(\xi)^\top
            \mathbf{u}_k(\alpha)
        \big)
        \big(
            \mathbf{v}_k(\alpha)^\top
            \boldsymbol{\Phi}_S(\eta)
        \big)
        }_{=:\,\boldsymbol{\Phi}_O(\xi)^\top A_k(\alpha) \boldsymbol{\Phi}_S(\eta)}
\end{equation}
where
$A_k(\alpha):=\mathbf{u}_k(\alpha)\mathbf{v}_k(\alpha)^\top$
collects the coefficients.

A conditioning network maps $\alpha$
to the factors $\mathbf{u}_k(\alpha)$ and $\mathbf{v}_k(\alpha)$ of the
rank-one matrices $\{A_k\}$. 
Thus, the network learns only the medium-dependent coefficients, while source predictions are then obtained by numerical integration.
Since the predicted response is linear in
each $A_k$ for fixed rates, a squared loss is convex in the coefficient
matrices, but not in their individual factors.
Results on factored low-rank optimization suggest, under assumptions that our problem does not fully satisfy, that rank-one factors can stall at
stationary points that do not minimize the convex loss. Additional factor columns can supply descent directions away from such
points~\citep{burer2005local,haeffele2015global,zhang2025improved}. Motivated by
these results, we use a higher-rank factorization as an optimization heuristic
and parametrize
\[
    A_k(\alpha)
    =
    U_k(\alpha)V_k(\alpha)^\top,
\]
with
$V_k(\alpha)\in\mathbb{R}^{D_S\times L}$ and
$U_k(\alpha)\in\mathbb{R}^{D_O\times L}$ to allow for higher rank. We
elaborate on this in Appendix~\ref{app:search-rank}, and the experiment
in Section~\ref{sec:exp-kernel-structure} demonstrates the effect empirically.

\subsection{Markovian streaming via an augmented kernel state}
\label{sec:streaming}
The kernel parametrization in \eqref{eq:kernel-model} also admits a Markovian
augmented-state representation that can be used for recurrent streaming
prediction.
Since the dynamics on the observation submanifold are not closed,
naively evaluating the Green's integral would require reintegration over the
full source history at every time step with growing cost over longer prediction horizons.
On a time grid where the source is piecewise constant, the augmented formulation is exact and replaces this growing history by a fixed-dimensional state whose size is independent of the prediction horizon. 
The observed state can then be reconstructed from the
augmented state.

\begin{proposition}[Exact streaming evaluation, \emph{informal}]
\label{prop:streaming}
On a uniform grid $t_i=i\Delta$, a piecewise-constant projected source
$\widehat f^i$ can be evaluated exactly with one recurrent state per temporal
mode, updated by
\begin{equation}
\label{eq:stream-recurrence}
    g_k^{i+1}
    =
    e^{-\lambda_k\Delta}g_k^i
    +
    \left(\int_0^\Delta e^{-\lambda_k s}\,ds\right)
    V_k(\alpha)^\top\widehat f^i.
\end{equation}
The observation is reconstructed as
$v_\alpha(\xi,t_i)=\boldsymbol{\Phi}_O(\xi)^\top
\sum_{k=1}^{M}U_k(\alpha)g_k^i$.
\end{proposition}
Furthermore, Proposition~\ref{prop:streaming-stability} shows that the exponential update makes this recurrence stable whenever the rates have nonnegative real parts.
For both propositions, precise statements and proofs are given in
Appendix~\ref{app:streaming-proof}.

\begin{proposition}[Stability of the streaming state, \emph{informal}]
\label{prop:streaming-stability}
If $\operatorname{Re}(\lambda_k)\geq0$, the recurrence in
Proposition~\ref{prop:streaming} does not amplify state perturbations. If the
inequality is strict, bounded projected sources yield uniformly bounded states.
\end{proposition}

\subsection{Reusing the linear kernel for mild nonlinearities}
\label{sec:nonlinear}

The learned kernel describes a linear system, but classical techniques allow
us to reuse it for certain mild nonlinearities. One such technique handles
source-like nonlinearities by alternating the learned linear
propagation with a nonlinear source update, in the spirit of Strang
splitting~\citep{strang1968construction}. When the nonlinear term can be evaluated from the predicted
trace, the streaming state makes these repeated updates inexpensive.
Appendix~\ref{app:radiative-feedback} develops this construction for
radiative heat loss.

A second technique handles state-dependent material laws through a change
of variables. For nonlinear heat conduction, the Kirchhoff transform either
reduces the problem exactly to the learned linear response or leaves a
transformed nonlinear system that can be advanced approximately using the
streaming state. Appendix~\ref{app:kirchhoff-construction} develops both
cases.

\section{Experiments}
\label{sec:experiments}

Our experiments show that GObO learns the dynamics between source and observation submanifolds while enabling efficient prediction and source generalization. We compare against black-box neural operators and learned Green's function approaches.

\subsection{Experimental setting}

We study three-dimensional heat conduction and advection--diffusion with different configurations of source and observation submanifolds. The datasets are obtained from finite-volume solves and specified in Appendix~\ref{app:TODO-datasets}. All results report the mean relative error and standard deviation over 3 seeds. 
Evaluation tests predictive accuracy across unseen inputs \(f\) and medium configurations \(\alpha\).
Our GObO implementation uses cosine basis functions whose coefficients are learned with a medium-conditioned U-FNO-style cosine spectral backbone \citep{wen2022ufno,liu2023spfno}. The temporal rates are not learned. Further details on the architecture, baselines and experimental details are given in Appendices~\ref{app:TODO-shared-architecture} and \ref{app:TODO-experiment-details}, with ablations of these choices in Appendix~\ref{app:results}.

\subsection{GObO learns a kernel that generalizes across sources}

The kernel learned by GObO remains accurate across different source classes,
although this comes at some cost on in-distribution accuracy. We
train all models on static source pulses and evaluate them on held-out static
pulses and on sources moving along a straight line. Figure~\ref{fig:exp1}
illustrates this transfer in the heat setting. While the medium-conditioned FNO-2D does not recover the location, shape, or magnitude of the moving response, GObO continues to track it closely.

\begin{table}[h]
\centering
\footnotesize
\setlength{\tabcolsep}{5pt}
\begin{tabular}{llcc}
\toprule
system & method & static & moving \\
\midrule
heat & GObO & $0.011 \pm 0.001$ & $\mathbf{0.012 \pm 0.000}$ \\
& adapted NGO, $S=1$ & $0.039 \pm 0.001$ & $0.043 \pm 0.000$ \\
& medium-conditioned FNO-2D & $\mathbf{0.001 \pm 0.000}$ & $0.101 \pm 0.004$ \\
\midrule
advection--diffusion & GObO, fixed rates & $0.114 \pm 0.018$ & $\mathbf{0.101 \pm 0.020}$ \\
& medium-conditioned FNO-2D & $\mathbf{0.104 \pm 0.009}$ & $0.430 \pm 0.015$ \\
\bottomrule
\end{tabular}
\caption{GObO transfers to unseen source configurations. All models are trained on static sources and evaluated on static and moving sources without retraining. Shown is the mean relative error and standard deviation.}
\label{tab:exp1}
\end{table}

Table~\ref{tab:exp1} quantifies this behavior for both heat and advection--diffusion. FNO-2D achieves the lowest error on static heat sources, but its error increases by about \(100\times\) for moving sources, whereas the error of GObO changes only slightly. 
The adapted NGO, another learned Green's function approach, shows similarly stable transfer across source classes, but with errors around \(4\times\) higher than GObO. 
This pattern is consistent across additional source and observation configurations, where the FNO-2D error increases by about \(4\)--\(7\times\) while GObO remains stable (Appendix~\ref{app:result-distinct-observations}). 
Additional black-box baselines, including FNO-3D and a backbone-matched spectral predictor, show similar behavior (Appendix~\ref{app:result-heat-transfer}).

\begin{figure}[t]
\centering
\includegraphics[width=0.92\linewidth]{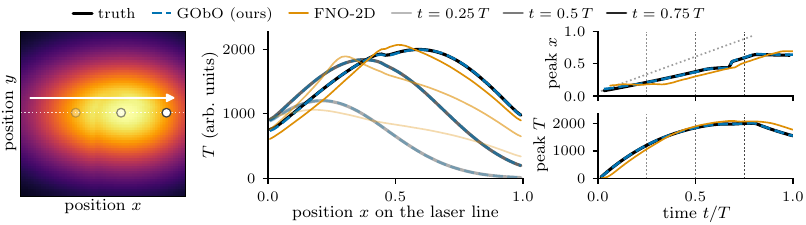}
\caption{GObO tracks the moving response, whereas FNO does not track it accurately. Left: true final temperature and source trajectory. Middle: temperature profiles at the three indicated times. Right: position and magnitude of the temperature peak.}
\label{fig:exp1}
\end{figure}

\subsection{The augmented kernel state captures non-closed submanifold dynamics}
\label{sec:exp-augmented-state}


Autoregressive prediction from the restricted state requires information about its history, whereas GObO carries the missing information in an augmented kernel state. We first verify this in a controlled regression setting by predicting the next state from different amounts of observation history. As shown in Figure~\ref{fig:impl-onestep-err}(a), additional history reduces the error by about \(3\times\) for the restricted trace, while a closed two-dimensional heat equation on the same manifold is predicted to machine precision from a single state. This isolates the effect of the non-closed dynamics.
The same issue appears in recurrent Green's models. Transient NGO learns a purely spatial kernel in a one-step recurrence with a spectral cap \(S\) for stabilization, but yields about \(3.5\times\) higher rollout error than GObO and degrades over longer horizons (Figure~\ref{fig:impl-onestep-err}(b)). To separate this effect from architectural differences, we also place a purely spatial GObO kernel in the same one-step setting, matched to the full model except for the temporal exponential rates. Its predictions again improve with additional history, but remain less accurate than full GObO (Figure~\ref{fig:impl-onestep-err}(c)). These results show that the non-closed dynamics are consequential and that GObO's augmented kernel state captures the information needed to propagate them.
Experimental details and complete results are given in Appendices~\ref{app:details-kernel-state} and~\ref{app:result-one-step}.

\begin{figure}[h]
\centering
\begin{subfigure}[b]{0.32\textwidth}
  \centering\includegraphics[width=0.8\linewidth]{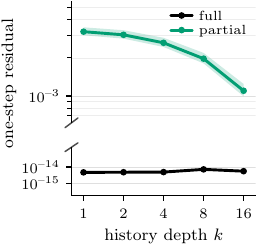}
  \caption{full and partial observations}
\end{subfigure}\hfill
\begin{subfigure}[b]{0.32\textwidth}
  \centering\includegraphics[width=0.8\linewidth]{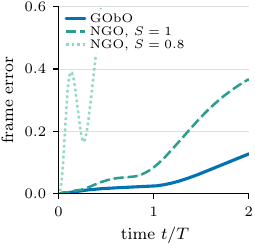}
  \caption{GObO vs. NGO}
\end{subfigure}\hfill
\begin{subfigure}[b]{0.32\textwidth}
  \centering\includegraphics[width=0.8\linewidth]{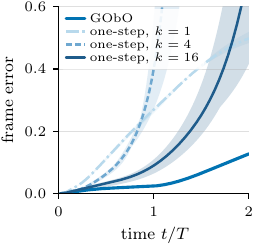}
  \caption{one-step GObO variant}
\end{subfigure}
\caption{The restricted observation requires additional state information for recurrent prediction. (a) Next-step regression error with increasing observation history, compared with a closed two-dimensional heat system. (b) Rollout error of GObO and transient NGO with spectral cap \(S\). (c)~History dependence of the purely spatial one-step GObO variant.}
\label{fig:impl-onestep-err}
\end{figure}

\subsection{The kernel parametrization captures the spatiotemporal structure}
\label{sec:exp-kernel-structure}

The proposed spectral kernel parametrization is important for capturing the joint spatial and temporal structure of the response. Replacing it with a naive tensor decomposition over space and time increases the error by \(3\)--\(4\times\) (Appendix~\ref{app:result-tensor}). This difference is also visible in the impulse response of Figure~\ref{fig:impl-lag}, where GObO captures the exponential temporal structure while the tensor ablation does not (Appendix~\ref{app:result-impulse-response}).
We further isolate the factor search rank \(L\) in a controlled example with curved manifolds, a homogeneous medium, and an analytically known kernel. Here, the task reduces to direct regression of the kernel coefficients. Increasing \(L\) from \(1\) to \(8\) reduces the fitting objective by about \(14\times\), with little improvement afterward, which motivates our choice of \(L=16\) as a safety margin. This agrees with the full-setting ablation in Appendix~\ref{app:result-gop-ablation}; the controlled recovery and full rank sweep are reported in Appendices~\ref{app:result-curved-recovery} and~\ref{app:result-rank}.

\subsection{GObO enables rapid source evaluation and nonlinear corrections}
\label{sec:exp_reuse}

Once learned, the GObO kernel supports inexpensive source evaluation and extensions of the underlying linear model. Constructing the kernel requires one neural-network pass per medium, comparable in cost to one FNO-2D prediction, while subsequent source evaluations are \(134\times\) faster on CPU and \(5.6\times\) faster on GPU. Table~\ref{tab:timing} summarizes these costs for batch size \(1\), with similar amortization for batched and streaming evaluation in Appendix~\ref{app:result-timing}.

\begin{table}[h]
\centering
\footnotesize
\setlength{\tabcolsep}{4pt}
\begin{tabular}{llcc}
\toprule
method & frequency & CPU & GPU, $b=1$ \\
\midrule
FV solver, matched accuracy & per source & $4.6$ s & n/a \\
FNO-2D forward & per source & $194.1$ ms & $2.50$ ms \\
GObO conditioning & per medium & $218.3$ ms & $10.1$ ms \\
GObO query & per source & $\mathbf{1.45}$ ms & $\mathbf{0.45}$ ms \\
\bottomrule
\end{tabular}
\caption{GObO constructs one kernel per medium and
then evaluates new sources faster than FNO-2D or the numerical solver.}
\label{tab:timing}
\end{table}

Its basis-function representation also supports arbitrary evaluation locations on the observation submanifold, with numerically negligible additional error at \(2\times\) the training resolution. The complete resolution-transfer study is given in Appendix~\ref{app:result-resolution-transfer}.
Finally, the same kernel further supports the nonlinear corrections of Section~\ref{sec:nonlinear} without retraining, improving accuracy by more than \(11\times\) for radiative feedback and up to \(29\times\) for temperature-dependent conductivity (Appendix~\ref{app:result-radiative-feedback}). The corrected traces closely follow the corresponding nonlinear reference solutions in Figures~\ref{fig:exp2-app} and~\ref{fig:exp3}.


\section{Related Work}
\label{sec:related}
We summarize the most directly related work here and provide an extended discussion of related work in Appendix~\ref{app:related-ext}.
\paragraph{Neural operators and surrogates.}
Neural operators can learn maps between function spaces, including maps from
sources on one submanifold to observations on another. However,
they usually treat this relation as a black box and do not explicitly impose
linearity in the source. GObO instead encodes this structure through a reusable
restricted Green kernel.
\paragraph{PDEs on manifolds.}
Recent methods solve PDEs posed intrinsically on manifolds or learn their
solution operators~\citep{burbulla2024pinn,chen2024norm,liao2025curvature}. In
our setting, the PDE instead evolves in the ambient domain and the observed
submanifold is not closed. GObO therefore models how the surrounding medium
couples the source and observation submanifolds.
\paragraph{Learning Green's functions.}
Learned Green functions often identify a fixed operator
\citep{boulle2022learning}. Recent methods condition steady-state
representations on geometry~\citep{yoo2026neural} or Green operators on PDE
coefficients, while transient approaches rely on fixed time-stepping schemes.
GObO instead learns
the continuous time-dependent kernel between submanifolds.

\section{Conclusion}
\label{sec:conclusion}

We introduced the Green's Observation Operator, a method for learning a Green's kernel that maps directly between source and observation submanifolds of a larger ambient domain. GObO preserves source linearity and causality and admits a horizon-independent streaming form, enabling rapid predictions and generalization to new sources. Across heat conduction and advection--diffusion, this structure trades some in-distribution accuracy for reliable source transfer and cheap reuse. The same kernel also supports resolution transfer, longer rollouts, and extensions to mild nonlinearities without retraining the linear kernel.
Future work could extend the approximation and stability guarantees to broader classes of PDEs and submanifold settings, and explore more general nonlinear dynamics.

\subsubsection*{Acknowledgments}
The main part of this work was conducted within the BMFTR awards 16IS24071A and 16IS24071B (PIAD). The authors further acknowledge support under BMBF award 01IW23005 (Eventful) and by the DFG through DFG FOR 5359 (ID 459419731) , TRR 375 (ID 511263698), SPP 2298 (ID 464252197), and SPP 2331 (IDs 441958259, 553345933, and 466468799) and by the Carl-Zeiss Foundation through the initiatives AI4ChemRisk and AI-Care.

\subsection*{Declaration of Generative AI Assistance}
As part of this work, generative AI tools were used for literature reviews, to assist in writing and editing texts, to implement code for training and evaluating numerical experiments, and to create scientific figures.
All AI-generated results were reviewed and verified by the authors.
Literature suggestions were cross-checked against the original sources, the code was inspected and tested, and the figures were validated based on the underlying numerical results.
The scientific ideas, methodological decisions, mathematical statements, experimental design, interpretation of the results, and conclusions were developed and evaluated by the authors.

\bibliography{iclr2027_conference}

@book{stakgold1979green,
  title={Green's functions and boundary value problems},
  author={Stakgold, Ivar and Holst, Michael J},
  volume={9},
  year={1979},
  publisher={Wiley New York}
}

@book{leveque2007finite,
  title={Finite difference methods for ordinary and partial differential equations: steady-state and time-dependent problems},
  author={LeVeque, Randall J},
  year={2007},
  publisher={SIAM}
}

@article{aronson1968non,
  title={Non-negative solutions of linear parabolic equations},
  author={Aronson, Donald G},
  journal={Annali della Scuola Normale Superiore di Pisa-Scienze Fisiche e Matematiche},
  volume={22},
  number={4},
  pages={607--694},
  year={1968}
}

@book{davies1989heat,
  title={Heat kernels and spectral theory},
  author={Davies, Edward Brian},
  number={92},
  year={1989},
  publisher={Cambridge university press}
}

@book{friedlander1998introduction,
  title={Introduction to the Theory of Distributions},
  author={Friedlander, Friedrich Gerard},
  year={1998},
  publisher={Cambridge University Press}
}

@article{agranovich2008spectral,
  title={Spectral boundary value problems in Lipschitz domains for strongly elliptic systems in Banach spaces H p $\sigma$ and B p $\sigma$},
  author={Agranovich, Mikhail Semenovich},
  journal={Functional Analysis and Its Applications},
  volume={42},
  number={4},
  pages={249--267},
  year={2008},
  publisher={Springer}
}

@article{daners2000heat,
  title={Heat kernel estimates for operators with boundary conditions},
  author={Daners, Daniel},
  journal={Mathematische Nachrichten},
  volume={217},
  number={1},
  pages={13--41},
  year={2000},
  publisher={Wiley Online Library}
}

@article{choi2013green,
  title={Green's function for second order parabolic systems with Neumann boundary condition},
  author={Choi, Jongkeun and Kim, Seick},
  journal={Journal of Differential Equations},
  volume={254},
  number={7},
  pages={2834--2860},
  year={2013},
  publisher={Elsevier}
}

@article{nittka2011regularity,
  title={Regularity of solutions of linear second order elliptic and parabolic boundary value problems on Lipschitz domains},
  author={Nittka, Robin},
  journal={Journal of Differential Equations},
  volume={251},
  number={4-5},
  pages={860--880},
  year={2011},
  publisher={Elsevier}
}

@article{necati1993heat,
  title={Heat conduction},
  author={Necati, Ozisik M and others},
  journal={ed: John Wiley \& Sons, Inc},
  year={1993}
}

@article{benner2015survey,
  title={A survey of projection-based model reduction methods for parametric dynamical systems},
  author={Benner, Peter and Gugercin, Serkan and Willcox, Karen},
  journal={SIAM review},
  volume={57},
  number={4},
  pages={483--531},
  year={2015},
  publisher={SIAM}
}

@article{carlberg2017galerkin,
  title={Galerkin v. least-squares Petrov--Galerkin projection in nonlinear model reduction},
  author={Carlberg, Kevin and Barone, Matthew and Antil, Harbir},
  journal={Journal of Computational Physics},
  volume={330},
  pages={693--734},
  year={2017},
  publisher={Elsevier}
}

@book{ouhabaz2009analysis,
  title={Analysis of heat equations on domains.(LMS-31)},
  author={Ouhabaz, El-Maati},
  year={2009},
  publisher={Princeton University Press}
}

@book{carslaw1906introduction,
  title={Introduction to the Mathematical Theory of the Conduction of Heat in Solids},
  author={Carslaw, Horatio Scott},
  year={1906},
  publisher={Macmillan and Company, limited}
}

@article{hochbruck2010exponential,
  title={Exponential integrators},
  author={Hochbruck, Marlis and Ostermann, Alexander},
  journal={Acta Numerica},
  volume={19},
  pages={209--286},
  year={2010},
  publisher={Cambridge University Press}
}

@article{lu2021learning,
  title={Learning nonlinear operators via DeepONet based on the universal approximation theorem of operators},
  author={Lu, Lu and Jin, Pengzhan and Pang, Guofei and Zhang, Zhongqiang and Karniadakis, George Em},
  journal={Nature machine intelligence},
  volume={3},
  number={3},
  pages={218--229},
  year={2021},
  publisher={Nature Publishing Group UK London}
}

@inproceedings{li2021fourier,
    title={Fourier Neural Operator for Parametric Partial Differential Equations},
    author={Zongyi Li and Nikola Borislavov Kovachki and Kamyar Azizzadenesheli and Burigede liu and Kaushik Bhattacharya and Andrew Stuart and Anima Anandkumar},
    booktitle={International Conference on Learning Representations},
    year={2021},
}

@article{wen2022ufno,
  title={U-FNO—An enhanced Fourier neural operator-based deep-learning model for multiphase flow},
  author={Wen, Gege and Li, Zongyi and Azizzadenesheli, Kamyar and Anandkumar, Anima and Benson, Sally M},
  journal={Advances in Water Resources},
  volume={163},
  pages={104180},
  year={2022},
  publisher={Elsevier}
}

@article{liu2023spfno,
  title={Spfno: Spectral operator learning for pdes with dirichlet and neumann boundary conditions},
  author={Liu, Ziyuan and Wu, Yuhang and Huang, Daniel Zhengyu and Zhang, Hong and Qian, Xu and Song, Songhe},
  journal={arXiv preprint arXiv:2312.06980},
  year={2023}
}

@article{kovachki2023neural,
  title={Neural operator: Learning maps between function spaces with applications to pdes},
  author={Kovachki, Nikola and Li, Zongyi and Liu, Burigede and Azizzadenesheli, Kamyar and Bhattacharya, Kaushik and Stuart, Andrew and Anandkumar, Anima},
  journal={Journal of Machine Learning Research},
  volume={24},
  number={89},
  pages={1--97},
  year={2023}
}

@inproceedings{brandstetter2022message,
    title={Message Passing Neural {PDE} Solvers},
    author={Johannes Brandstetter and Daniel E. Worrall and Max Welling},
    booktitle={International Conference on Learning Representations},
    year={2022},
}

@article{boussif2022magnet,
  title={MAgnet: Mesh agnostic neural PDE solver},
  author={Boussif, Oussama and Bengio, Yoshua and Benabbou, Loubna and Assouline, Dan},
  journal={Advances in Neural Information Processing Systems},
  volume={35},
  pages={31972--31985},
  year={2022}
}

@article{cho2025learning,
  title={Learning time-dependent PDE via graph neural networks and deep operator network for robust accuracy on irregular grids},
  author={Cho, Sung Woong and Lee, Jae Yong and Hwang, Hyung Ju},
  journal={Journal of Computational Physics},
  pages={114430},
  year={2025},
  publisher={Elsevier}
}

@article{boulle2022learning,
  title={Learning Green's functions associated with time-dependent partial differential equations},
  author={Boull{\'e}, Nicolas and Kim, Seick and Shi, Tianyi and Townsend, Alex},
  journal={Journal of Machine Learning Research},
  volume={23},
  number={218},
  pages={1--34},
  year={2022}
}

@article{yoo2026neural,
  title={Neural Green's functions},
  author={Yoo, Seungwoo and Yeo, Kyeongmin and Hwang, Jisung and Sung, Minhyuk},
  journal={Advances in Neural Information Processing Systems},
  volume={38},
  pages={47961--47982},
  year={2026}
}

@article{ayed2019learning,
  title={Learning dynamical systems from partial observations},
  author={Ayed, Ibrahim and de B{\'e}zenac, Emmanuel and Pajot, Arthur and Brajard, Julien and Gallinari, Patrick},
  journal={arXiv preprint arXiv:1902.11136},
  year={2019}
}

@article{ayed2022modelling,
  title={Modelling spatiotemporal dynamics from Earth observation data with neural differential equations},
  author={Ayed, Ibrahim and de B{\'e}zenac, Emmanuel and Pajot, Arthur and Gallinari, Patrick},
  journal={Machine Learning},
  volume={111},
  number={6},
  pages={2349--2380},
  year={2022},
  publisher={Springer}
}

@article{iakovlev2023learning,
  title={Learning space-time continuous latent neural PDEs from partially observed states},
  author={Iakovlev, Valerii and Heinonen, Markus and L{\"a}hdesm{\"a}ki, Harri},
  journal={Advances in Neural Information Processing Systems},
  volume={36},
  pages={26372--26395},
  year={2023}
}

@article{stephany2024weak,
  title={Weak-PDE-LEARN: A weak form based approach to discovering PDEs from noisy, limited data},
  author={Stephany, Robert and Earls, Christopher},
  journal={Journal of Computational Physics},
  volume={506},
  pages={112950},
  year={2024},
  publisher={Elsevier}
}

@article{kontolati2024learning,
  title={Learning nonlinear operators in latent spaces for real-time predictions of complex dynamics in physical systems},
  author={Kontolati, Katiana and Goswami, Somdatta and Em Karniadakis, George and Shields, Michael D},
  journal={Nature Communications},
  volume={15},
  number={1},
  pages={5101},
  year={2024},
  publisher={Nature Publishing Group UK London}
}

@article{li2025latent,
  title={Latent neural PDE solver: A reduced-order modeling framework for partial differential equations},
  author={Li, Zijie and Patil, Saurabh and Ogoke, Francis and Shu, Dule and Zhen, Wilson and Schneier, Michael and Buchanan Jr, John R and Farimani, Amir Barati},
  journal={Journal of Computational Physics},
  volume={524},
  pages={113705},
  year={2025},
  publisher={Elsevier}
}

@article{wang2024latent,
  title={Latent neural operator for solving forward and inverse pde problems},
  author={Wang, Tian and Wang, Chuang},
  journal={Advances in Neural Information Processing Systems},
  volume={37},
  pages={33085--33107},
  year={2024}
}

@article{seidman2022nomad,
  title={NOMAD: Nonlinear manifold decoders for operator learning},
  author={Seidman, Jacob and Kissas, Georgios and Perdikaris, Paris and Pappas, George J},
  journal={Advances in Neural Information Processing Systems},
  volume={35},
  pages={5601--5613},
  year={2022}
}

@article{mou2026neuralpod,
  title={Neural-POD: A plug-and-play neural operator framework for infinite-dimensional functional nonlinear proper orthogonal decomposition},
  author={Mou, Changhong and Lu, Binghang and Lin, Guang},
  journal={arXiv preprint arXiv:2602.15632},
  year={2026}
}

@article{zhong2025physics,
  title={Physics-informed geometry-aware neural operator},
  author={Zhong, Weiheng and Meidani, Hadi},
  journal={Computer Methods in Applied Mechanics and Engineering},
  volume={434},
  pages={117540},
  year={2025},
  publisher={Elsevier}
}

@article{wen2026geometry,
  title={Geometry aware operator transformer as an efficient and accurate neural surrogate for pdes on arbitrary domains},
  author={Wen, Shizheng and Kumbhat, Arsh and Lingsch, Levi and Mousavi, Sepehr and Zhao, Yizhou and Chandrashekar, Praveen and Mishra, Siddhartha},
  journal={Advances in Neural Information Processing Systems},
  volume={38},
  pages={155423--155501},
  year={2026}
}

@article{li2023geofno,
  author  = {Zongyi Li and Daniel Zhengyu Huang and Burigede Liu and Anima Anandkumar},
  title   = {Fourier Neural Operator with Learned Deformations for PDEs on General Geometries},
  journal = {Journal of Machine Learning Research},
  year    = {2023},
  volume  = {24},
  number  = {388},
  pages   = {1--26},
}

@article{zhao2025diffeomorphism,
  title={Diffeomorphism neural operator for various domains and parameters of partial differential equations},
  author={Zhao, Zhiwei and Liu, Changqing and Li, Yingguang and Chen, Zhibin and Liu, Xu},
  journal={Communications Physics},
  volume={8},
  number={1},
  pages={15},
  year={2025},
  publisher={Nature Publishing Group UK London}
}

@article{li2023gino,
  title={Geometry-informed neural operator for large-scale 3d pdes},
  author={Li, Zongyi and Kovachki, Nikola and Choy, Chris and Li, Boyi and Kossaifi, Jean and Otta, Shourya and Nabian, Mohammad Amin and Stadler, Maximilian and Hundt, Christian and Azizzadenesheli, Kamyar and others},
  journal={Advances in Neural Information Processing Systems},
  volume={36},
  pages={35836--35854},
  year={2023}
}

@article{chen2024norm,
  title={Learning neural operators on riemannian manifolds},
  author={Chen, Gengxiang and Liu, Xu and Meng, Qinglu and Chen, Lu and Liu, Changqing and Li, Yingguang},
  journal={National Science Open},
  volume={3},
  number={6},
  year={2024},
  publisher={Science China Press}
}

@article{burbulla2024pinn,
  title={Physics-informed neural networks for transformed geometries and manifolds},
  author={Burbulla, Samuel},
  journal={arXiv preprint arXiv:2311.15940},
  year={2023}
}

@inproceedings{liao2025curvature,
  title={Curvature-aware graph attention for PDEs on manifolds},
  author={Liao, Yunfeng and Guan, Jiawen and Li, Xiucheng},
  booktitle={Forty-second International Conference on Machine Learning},
  year={2025}
}

@article{suk2024mesh,
  title={Mesh neural networks for SE (3)-equivariant hemodynamics estimation on the artery wall},
  author={Suk, Julian and de Haan, Pim and Lippe, Phillip and Brune, Christoph and Wolterink, Jelmer M},
  journal={Computers in Biology and Medicine},
  volume={173},
  pages={108328},
  year={2024},
  publisher={Elsevier}
}

@article{olivieri2021physics,
  title={A physics-informed neural network approach for nearfield acoustic holography},
  author={Olivieri, Marco and Pezzoli, Mirco and Antonacci, Fabio and Sarti, Augusto},
  journal={Sensors},
  volume={21},
  number={23},
  pages={7834},
  year={2021},
  publisher={MDPI}
}

@article{sajadi2025temperature,
  title={Two-dimensional temperature field prediction with in-situ data in metal additive manufacturing using physics-informed neural networks},
  author={Sajadi, Pouyan and Dehaghani, Mostafa Rahmani and Tang, Yifan and Wang, G Gary},
  journal={Engineering Applications of Artificial Intelligence},
  volume={150},
  pages={110636},
  year={2025},
  publisher={Elsevier}
}

@article{hashmonay2001field,
  title={Field evaluation of a method for estimating gaseous fluxes from area sources using open-path Fourier transform infrared},
  author={Hashmonay, Ram A and Natschke, David F and Wagoner, Keith and Harris, D Bruce and Thompson, Edgar L and Yost, Michael G},
  journal={Environmental science \& technology},
  volume={35},
  number={11},
  pages={2309--2313},
  year={2001},
  publisher={ACS Publications}
}

@techreport{kim1975thermophysical,
  title={Thermophysical properties of stainless steels},
  author={Kim, Choong S},
  year={1975},
  institution={Argonne National Lab., Ill.(USA)}
}

@article{mukherjee2016printability,
  title={Printability of alloys for additive manufacturing},
  author={Mukherjee, Tuhin and Zuback, JS and De, A and DebRoy, Tarasankar},
  journal={Scientific reports},
  volume={6},
  number={1},
  pages={19717},
  year={2016},
  publisher={Nature Publishing Group UK London}
}

@misc{specialmetals2007in718,
  title={Special Metals Corporation},
  author={Alloys, High-Performance},
  year={2008}
}

@article{strang1968construction,
  title={On the construction and comparison of difference schemes},
  author={Strang, Gilbert},
  journal={SIAM journal on numerical analysis},
  volume={5},
  number={3},
  pages={506--517},
  year={1968},
  publisher={SIAM}
}

@article{ascher1997implicit,
  title={Implicit-explicit Runge-Kutta methods for time-dependent partial differential equations},
  author={Ascher, Uri M and Ruuth, Steven J and Spiteri, Raymond J},
  journal={Applied Numerical Mathematics},
  volume={25},
  number={2-3},
  pages={151--167},
  year={1997},
  publisher={Elsevier}
}

@article{mclachlan2002splitting,
  title={Splitting methods},
  author={McLachlan, Robert I and Quispel, G Reinout W},
  journal={Acta Numerica},
  volume={11},
  pages={341--434},
  year={2002},
  publisher={Cambridge University Press}
}

@article{leveque1994immersed,
  title={The immersed interface method for elliptic equations with discontinuous coefficients and singular sources},
  author={LeVeque, Randall J and Li, Zhilin},
  journal={SIAM Journal on Numerical Analysis},
  volume={31},
  number={4},
  pages={1019--1044},
  year={1994},
  publisher={SIAM}
}

@article{kandilarov2007immersed,
  title={The immersed interface method for two-dimensional heat-diffusion equations with singular own sources},
  author={Kandilarov, Juri D and Vulkov, Lubin G},
  journal={Applied numerical mathematics},
  volume={57},
  number={5-7},
  pages={486--497},
  year={2007},
  publisher={Elsevier}
}

@article{peskin2002immersed,
  title={The immersed boundary method},
  author={Peskin, Charles S},
  journal={Acta numerica},
  volume={11},
  pages={479--517},
  year={2002},
  publisher={Cambridge University Press}
}

@article{zhang2025improved,
  title={Improved global guarantees for the nonconvex Burer--Monteiro factorization via rank overparameterization: RY Zhang},
  author={Zhang, Richard Y},
  journal={Mathematical Programming},
  volume={213},
  number={1},
  pages={1009--1038},
  year={2025},
  publisher={Springer}
}

@article{ha2020equivalence,
  title={An equivalence between critical points for rank constraints versus low-rank factorizations},
  author={Ha, Wooseok and Liu, Haoyang and Barber, Rina Foygel},
  journal={SIAM Journal on Optimization},
  volume={30},
  number={4},
  pages={2927--2955},
  year={2020},
  publisher={SIAM}
}

@article{burer2005local,
  title={Local minima and convergence in low-rank semidefinite programming},
  author={Burer, Samuel and Monteiro, Renato DC},
  journal={Mathematical programming},
  volume={103},
  number={3},
  pages={427--444},
  year={2005},
  publisher={Springer}
}

@inproceedings{bhojanapalli2018smoothed,
  title={Smoothed analysis for low-rank solutions to semidefinite programs in quadratic penalty form},
  author={Bhojanapalli, Srinadh and Boumal, Nicolas and Jain, Prateek and Netrapalli, Praneeth},
  booktitle={Conference on learning theory},
  pages={3243--3270},
  year={2018},
  organization={PMLR}
}

@article{haeffele2015global,
  title={Global optimality in tensor factorization, deep learning, and beyond},
  author={Haeffele, Benjamin D and Vidal, Ren{\'e}},
  journal={arXiv preprint arXiv:1506.07540},
  year={2015}
}

@inproceedings{ge2017nospurious,
  title={No spurious local minima in nonconvex low rank problems: A unified geometric analysis},
  author={Ge, Rong and Jin, Chi and Zheng, Yi},
  booktitle={International conference on machine learning},
  pages={1233--1242},
  year={2017},
  organization={PMLR}
}

@article{zhang2018howmuch,
  title={How much restricted isometry is needed in nonconvex matrix recovery?},
  author={Zhang, Richard and Josz, C{\'e}dric and Sojoudi, Somayeh and Lavaei, Javad},
  journal={Advances in neural information processing systems},
  volume={31},
  year={2018}
}

@article{tong2021accelerating,
  title={Accelerating ill-conditioned low-rank matrix estimation via scaled gradient descent},
  author={Tong, Tian and Ma, Cong and Chi, Yuejie},
  journal={Journal of Machine Learning Research},
  volume={22},
  number={150},
  pages={1--63},
  year={2021}
}

@inproceedings{saxe2014exact,
    title     = {Exact Solutions to the Nonlinear Dynamics of Learning in Deep Linear Neural Networks},
    author    = {Saxe, Andrew M. and McClelland, James L. and Ganguli, Surya},
    booktitle = {International Conference on Learning Representations},
    year      = {2014}
}

@article{du2017gradient,
  title={Gradient descent can take exponential time to escape saddle points},
  author={Du, Simon S and Jin, Chi and Lee, Jason D and Jordan, Michael I and Singh, Aarti and Poczos, Barnabas},
  journal={Advances in neural information processing systems},
  volume={30},
  year={2017}
}

@article{lesnic1995note,
  title={A note on the determination of the thermal properties of a material in a transient nonlinear heat conduction problem},
  author={Lesnic, D and Elliott, L and Ingham, DB},
  journal={International communications in heat and mass transfer},
  volume={22},
  number={4},
  pages={475--482},
  year={1995},
  publisher={Elsevier}
}

@article{fabre2017integral,
  title={On the integral-balance approach to the transient heat conduction with linearly temperature-dependent thermal diffusivity},
  author={Fabre, Antoine and Hristov, Jordan},
  journal={Heat and Mass Transfer},
  volume={53},
  number={1},
  pages={177--204},
  year={2017},
  publisher={Springer}
}

@article{boulle2022data,
  title={Data-driven discovery of Green’s functions with human-understandable deep learning},
  author={Boull{\'e}, Nicolas and Earls, Christopher J and Townsend, Alex},
  journal={Scientific reports},
  volume={12},
  number={1},
  pages={4824},
  year={2022},
  publisher={Nature Publishing Group UK London}
}

@article{zhang2022modnet,
  author = {Zhang, Lulu and Luo, Tao and Zhang, Yaoyu and E, Weinan and John Xu, Zhi-Qin and Ma, Zheng},
  title = {MOD-Net: A Machine Learning Approach via Model-Operator-Data Network for Solving PDEs},
  journal = {Communications in Computational Physics},
  year = {2022},
  volume = {32},
  number = {2},
  pages = {299-335},
  doi = {10.4208/cicp.OA-2021-0257},
}

@article{lin2023bigreennet,
  title={Bi-greennet: learning green’s functions by boundary integral network},
  author={Lin, Guochang and Chen, Fukai and Hu, Pipi and Chen, Xiang and Chen, Junqing and Wang, Jun and Shi, Zuoqiang},
  journal={Communications in Mathematics and Statistics},
  volume={11},
  number={1},
  pages={103--129},
  year={2023},
  publisher={Springer}
}

@article{gu2025greensonet,
  title={An explainable operator approximation framework under the guideline of Green’s function},
  author={Gu, Jianghang and Wen, Ling and Chen, Yuntian and Chen, Shiyi},
  journal={Journal of Computational Physics},
  pages={114244},
  year={2025},
  publisher={Elsevier}
}

@article{tang2022ngf,
  title={Neural green’s function for laplacian systems},
  author={Tang, Jingwei and Azevedo, Vinicius C and Cordonnier, Guillaume and Solenthaler, Barbara},
  journal={Computers \& Graphics},
  volume={107},
  pages={186--196},
  year={2022},
  publisher={Elsevier}
}

@article{praveen2025chebgreen,
  title={chebgreen: Learning and interpolating continuous Empirical Green's Functions from data},
  author={Praveen, Harshwardhan and Brown, Jacob and Earls, Christopher},
  journal={Computer Physics Communications},
  pages={109867},
  year={2025},
  publisher={Elsevier}
}

@article{yang2026gffo,
  title={Green’s function fourier operator: a spectral learning framework for linear differential equations with arbitrary boundary conditions},
  author={Yang, Yang and Zhang, Xihong and Kou, Kunpang and Lam, Chichiu},
  journal={Engineering Analysis with Boundary Elements},
  volume={191},
  pages={106948},
  year={2026},
  publisher={Elsevier}
}

@article{aldirany2024greenonet,
  title={Operator approximation of the wave equation based on deep learning of Green's function},
  author={Aldirany, Ziad and Cottereau, R{\'e}gis and Laforest, Marc and Prudhomme, Serge},
  journal={Computers \& Mathematics with Applications},
  volume={159},
  pages={21--30},
  year={2024},
  publisher={Elsevier}
}

@article{melchers2026ngo,
  title={Neural green’s operators for parametric partial differential equations},
  author={Melchers, HA and Prins, JHM and Abdelmalik, Michael RA},
  journal={Computer Methods in Applied Mechanics and Engineering},
  volume={455},
  pages={118893},
  year={2026},
  publisher={Elsevier}
}

@inproceedings{tan2026dgnet,
 author = {Tan, Yingjie and Yao, Quanming and Wang, Yaqing},
 booktitle = {International Conference on Learning Representations},
 editor = {C. Vondrick and B. Hariharan and C. Raffel and L. Pinto and D. Yang and A. Faust},
 pages = {26748--26773},
 title = {DGNet: Discrete Green Networks for Data-Efficient Learning of Spatiotemporal PDEs},
 volume = {2026},
 year = {2026}
}
\bibliographystyle{iclr2027_conference}

\appendix
\renewcommand{\theHequation}{app.\arabic{equation}}

\section*{Overview of the Appendix}

\begin{itemize}
    \item Appendix~\ref{app:related-ext} places GObO among operator-learning,
    manifold, and learned Green-function methods.
    \item Appendix~\ref{app:theory} provides the proofs and derivations used
    by the method.
    \item Appendix~\ref{app:datasets} defines the equations, media, sources,
    observation manifolds, and numerical reference solutions.
    \item Appendix~\ref{app:shared-architecture} describes the GObO
    architecture and the protocol shared across experiments.
    \item Appendix~\ref{app:details} gives the experiment-specific models and
    evaluation procedures.
    \item Appendix~\ref{app:results} provides the complete experimental
    results and ablations.
\end{itemize}

\section{Extended Related Work}
\label{app:related-ext}

\subsection{Operator learning for PDEs}
Neural operators learn mappings between function spaces and have become a
common tool for data-driven models of PDEs.
DeepONet and FNO use branch--trunk approximation and
spectral convolution. \citet{kovachki2023neural} give a general
function-space framework. Later work extends these ideas to irregular
meshes and graphs. Examples include message-passing PDE
solvers~\citep{brandstetter2022message},
mesh-agnostic neural PDE solvers~\citep{boussif2022magnet}, and
graph-enhanced DeepONets~\citep{cho2025learning}. We use the same
function-space view but take the output to be a measured submanifold.

\subsection{Partial observations and reduced representations}
PDE dynamics can also be learned from limited measurements through
forecasting models, latent states, weak forms, and reduced models.
Partially observed PDE models use neural networks, recurrent models, or
learned bases to forecast from limited
measurements~\citep{ayed2019learning,ayed2022modelling}, while latent neural
PDE models infer continuous spatio-temporal dynamics from partial and noisy
observations~\citep{iakovlev2023learning}. Weak-form methods discover PDE
structure from sparse or noisy measurements~\citep{stephany2024weak}.
Reduced and latent operator-learning methods, including Latent
DeepONet~\citep{kontolati2024learning}, latent neural PDE
solvers~\citep{li2025latent}, Latent Neural
Operators~\citep{wang2024latent}, nonlinear manifold
decoders~\citep{seidman2022nomad}, and neural POD
approaches~\citep{mou2026neuralpod}, learn compact representations of
high-dimensional states. GObO instead makes the observed submanifold the
prediction space and uses the ambient medium to determine its response.
The nonlinear construction in Section~\ref{sec:nonlinear} also connects
the learned kernel to projection-based model
reduction~\citep{benner2015survey,carlberg2017galerkin}.

\subsection{Geometry and physical context}
Many learned PDE models depend on geometry, coefficients, and boundary
conditions. These quantities can enter through domain encodings,
coordinate maps, graph representations,
signed-distance functions, and residual-based losses. Examples include
PI-GANO~\citep{zhong2025physics}, geometry-aware operator
transformers~\citep{wen2026geometry}, Geo-FNO~\citep{li2023geofno}, and
diffeomorphism-based neural operators~\citep{zhao2025diffeomorphism}.
GINO~\citep{li2023gino} is especially relevant because it predicts on
surfaces in large three-dimensional systems. It uses full-field
supervision and full-domain latent states. GObO uses only observations on
the manifold and represents their response with an explicit kernel.

\subsection{Nonlinear methods built on linear solvers}
Operator splitting originates with \citet{strang1968construction} and is
surveyed by \citet{mclachlan2002splitting}. Implicit--explicit methods
treat stiff and non-stiff terms differently~\citep{ascher1997implicit}.
Exponential integrators propagate a linear part exactly and approximate
the remaining increment~\citep{hochbruck2010exponential}. The Kirchhoff
transform is classical for temperature-dependent heat conduction
\citep{carslaw1906introduction,necati1993heat}. Each method assumes access
to a linear solver. Section~\ref{sec:nonlinear} uses the learned kernel as
that solver. Its streaming form advances the linear state exactly, while
source linearity allows it to evaluate the intermediate forcings created
by the nonlinear method.

\subsection{Green's-function and kernel-based operators}
\label{app:green-learning}

We separate methods by whether the problem is steady or transient and by how
the learned kernel is represented. Pointwise methods evaluate a coordinate
network at each source, observation, and possibly time coordinate. This is
flexible, but resolving a dense space--time kernel requires many network
evaluations and becomes impractical in the three-dimensional transient setting
considered here. Structured methods instead produce a reusable finite
representation that can be applied to many sources after 1 inference or
identification step.

\begin{table}[ht]
\centering
\small
\setlength{\tabcolsep}{4pt}
\begin{tabularx}{\linewidth}{p{0.14\linewidth}YY}
\toprule
 & pointwise kernel values & reusable representation \\
\midrule
steady
& GreenLearning, MOD-Net, BI-GreenNet, and GreensONet
\citep{boulle2022data,zhang2022modnet,lin2023bigreennet,gu2025greensonet}
& discrete and spectral Green representations
\citep{tang2022ngf,praveen2025chebgreen,yoo2026neural,yang2026gffo} \\
transient
& GreenONet \citep{aldirany2024greenonet}
& full-kernel identification \citep{boulle2022learning}, one-step NGO,
DGNet, and GObO \\
\bottomrule
\end{tabularx}
\caption{Representative learned Green-function methods. Pointwise methods
query individual kernel values. Reusable methods produce a finite operator
representation that is applied without another network query at each point.}
\label{tab:green-taxonomy}
\end{table}

Most prior work treats steady problems. Among transient approaches,
GreenONet represents time through pointwise coordinate queries
\citep{aldirany2024greenonet}. The full space--time recovery of
\mbox{\citet{boulle2022learning}} avoids these repeated queries, but identifies the
kernel of 1 fixed system rather than predicting a kernel for a new medium.
NGO and DGNet learn or construct a one-step update. Repeated application then requires the
updated quantity to contain the full state needed for the next step. This is
not true for the partial observations considered here.

GObO instead maps each new medium once to a structured kernel over temporal
lag between the source and observation submanifolds. It does not evaluate a
network at every kernel point, form a dense Green matrix, reconstruct the
ambient state, or assume that the observation is a closed state. The same
representation supports direct source convolution and the finite streaming
state derived in Section~\ref{sec:streaming}. NGO is therefore the closest
baseline because it also preserves source linearity and conditions its Green
representation on PDE coefficients, while differing in its one-step treatment
of time.

\section{Proofs and Derivations}
\label{app:theory}

This appendix collects the arguments that support the statements in
Section~\ref{sec:method}. We follow the order of the method section, so
that each argument can be read alongside the corresponding construction.

\subsection{Green representations for initial, interior, and boundary data}
\label{app:green-proofs}

For an autonomous linear parabolic problem, the initial condition and
the interior forcing are propagated by the same solution semigroup.
Boundary data enter through a boundary control operator and may therefore
have a different kernel. Writing these contributions separately gives
the general representation
\begin{align}
\label{eq:full-green-representation}
    u_\alpha(x,t)
    & =
    \underbrace{
    \int_\Omega
        G_\alpha(x,y,t)u_0(y)
    \,dy}_{\text{initial condition}}
    \nonumber\\
    &\quad+
    \underbrace{
    \int_0^t\int_\Omega
        G_\alpha(x,y,t-s)F(y,s)
    \,dy\,ds}_{\text{interior forcing}}
    \nonumber\\
    &\quad+
    \underbrace{
    \int_0^t\int_{\partial\Omega}
        G_\alpha^{\partial}(x,z,t-s)g(z,s)
    \,d\sigma(z)\,ds}_{\text{boundary input}}.
\end{align}
Thus the initial and forcing kernels coincide for the time-independent
systems considered here. The boundary kernel $G_\alpha^{\partial}$
depends on how the boundary data enter the PDE. For an inward Neumann
flux, the weak load is
\begin{equation}
\label{eq:weak-neumann-load}
    \ell_s(\phi)
    =
    \int_{\partial\Omega}
        g(z,s)\phi(z)
    \,d\sigma(z).
\end{equation}
Relative to the homogeneous Neumann problem, its response kernel is the
boundary trace of the same Green kernel,
$G_\alpha^{\partial}(x,z,\tau)=G_\alpha(x,z,\tau)$, with the sign fixed
by the convention that positive $g$ injects flux into the domain. Other
boundary conditions can produce a different boundary kernel. Keeping the
superscript in \eqref{eq:full-green-representation} makes this distinction
explicit.

We now restrict the two source terms. The area formula is
\begin{equation}
\label{eq:area-formula}
    \int_{\Sigma_S}
        h(y)\,d\sigma(y)
    =
    \int_{U_S}
        h(X_S(\eta))
        \sqrt{
            \det(
                DX_S(\eta)^\top DX_S(\eta)
            )
        }
    \,d\eta .
\end{equation}
For an interior manifold source $F=f_\Sigma\delta_{\Sigma_S}$, the
forcing term in \eqref{eq:full-green-representation} becomes
\begin{align}
\label{eq:interior-manifold-response}
    v_\alpha^F(\xi,t)
    & =
    \int_0^t\int_{U_S}
        K_\alpha^F(\xi,\eta,t-s)f(\eta,s)
    \,d\eta\,ds,
\\
    K_\alpha^F(\xi,\eta,\tau)
    & :=
    G_\alpha\!\left(
        X_O(\xi),X_S(\eta),\tau
    \right)J_S(\eta).
\end{align}
For a Neumann flux supported on a boundary patch
$\Sigma_S\subset\partial\Omega$, the boundary term instead becomes
\begin{align}
\label{eq:boundary-manifold-response}
    v_\alpha^\partial(\xi,t)
    & =
    \int_0^t\int_{U_S}
        K_\alpha^\partial(\xi,\eta,t-s)f(\eta,s)
    \,d\eta\,ds,
\\
    K_\alpha^\partial(\xi,\eta,\tau)
    & :=
    G_\alpha^\partial\!\left(
        X_O(\xi),X_S(\eta),\tau
    \right)J_S(\eta).
\end{align}
For a Neumann input, $K_\alpha^F$ and $K_\alpha^\partial$ are restrictions
of the same Green kernel to different source sets. A more general
boundary operator may instead change the boundary kernel. In either
case, the source-to-observation map has the same form. The spatial
integration is over $U_S$, the kernel depends on the full ambient medium, and no
observation Jacobian appears because the state is evaluated pointwise on
$\Sigma_O$. This is the common form written in
\eqref{eq:restricted-kernel}. If initial, interior, and boundary inputs
are present together, their three responses simply add. The initial
condition can also be viewed as an interior impulse at time zero.

The interior distribution and the weak Neumann load also share a local
flux interpretation. For a codimension-one surface, consider a
divergence-form diffusion equation written locally as
\begin{equation}
\label{eq:flux-distribution}
    -\nabla\cdot q
    =
    f_\Sigma\delta_{\Sigma_S},
    \qquad
    q=\kappa\nabla u.
\end{equation}
If $\Sigma_S\subset\Omega$ separates two regions and $n$ points from the
minus side to the plus side, integration over a shrinking pillbox gives
\begin{equation}
\label{eq:flux-jump}
    [q\cdot n]_{\Sigma_S}
    :=
    q^+\cdot n-q^-\cdot n
    =
    -f_\Sigma .
\end{equation}
Thus the singular source prescribes a jump in normal flux across an
interior interface, as in the classical immersed-interface construction
\citep{leveque1994immersed}. If
$\Sigma_S\subset\partial\Omega$, only the interior side remains and
\eqref{eq:flux-jump} becomes the Neumann condition
\begin{equation}
\label{eq:neumann-source}
    -\kappa\nabla u\cdot n
    =
    f_\Sigma ,
\end{equation}
when positive $f_\Sigma$ denotes flux injected into the domain. The
boundary and interior flux balances are therefore the one-sided and
two-sided versions of the same local construction.

We now state the analytic assumptions on the ambient problem. The result
below is the full version of Proposition~\ref{prop:green-existence} in the
main text.

\begin{prop1*}[Existence of the Green's kernel]
\phantomsection\fixedlabel{1}{prop:green-existence-app}
Let $\Omega\subset\mathbb R^d$ be bounded and Lipschitz. Suppose that
$\kappa:\Omega\to\mathbb R^{d\times d}$ is bounded, measurable, symmetric,
and uniformly elliptic, so that for some $0<\kappa_-\leq\kappa_+<\infty$,
\begin{equation}
\label{eq:ellipticity-assumption}
    \kappa_-|z|^2
    \leq
    z^\top\kappa(x)z
    \leq
    \kappa_+|z|^2
\end{equation}
for almost every $x\in\Omega$ and every $z\in\mathbb R^d$. Let
$b\in L^\infty(\Omega;\mathbb R^d)$. For the advection--diffusion systems
considered here, assume additionally that $b$ is divergence-free in
$\Omega$ and has zero normal trace on $\partial\Omega$.
Then each admissible $\alpha$ in \eqref{eq:parabolic} admits a Green's
kernel $G_\alpha$ with a representative that is continuous on
$\overline\Omega\times\overline\Omega\times(0,\infty)$.
\end{prop1*}

\begin{proof}
For scalar divergence-form parabolic equations with bounded measurable
coefficients, \citet{choi2013green} construct the Neumann Green's function.
The boundary regularity result of \citet{nittka2011regularity} gives the
continuous representative at positive times. The drift
$b\cdot\nabla u$ is a bounded lower-order term. It is covered by the
non-symmetric kernel theory of \citet{daners2000heat}.
\end{proof}

The kernel not only exists but also satisfies an Aronson-type Gaussian bound
\citep{aronson1968non,daners2000heat,choi2013green,ouhabaz2009analysis}. For every
finite $T>0$, there are constants $C_G,c>0$, which may depend on $T$, such that
\begin{equation}
\label{eq:gaussian}
    |G_\alpha(x,y,\tau)|
    \leq
    C_G\tau^{-d/2}
    \exp\!\left(-c\frac{|x-y|^2}{\tau}\right)
\end{equation}
for all $x,y\in\overline\Omega$ and $0<\tau\leq T$. Integrating over $y$
and enlarging the domain to $\mathbb R^d$ gives
\begin{equation}
    \sup_{x\in\overline\Omega}
    \int_\Omega |G_\alpha(x,y,\tau)|\,dy
    \leq
    C_G\left(\frac{\pi}{c}\right)^{d/2}
    =: C_T.
\end{equation}
The factor $\tau^{-d/2}$ is therefore canceled by integration over the full
$d$-dimensional domain. This makes \eqref{eq:green-volume} finite, for
example, whenever $F\in L^1(0,T;L^\infty(\Omega))$. Restricting the kernel
to a $(d-1)$-dimensional source manifold leaves the integrable singularity
$\tau^{-1/2}$ instead. The following result gives the corresponding bound
uniformly over all observation points. It is the full version of
Proposition~\ref{prop:restricted-integrability} in the main text.

\begin{prop2*}[Integrability of the restricted kernel]
\phantomsection\fixedlabel{2}{prop:restricted-integrability-app}
Assume the setting of Proposition~\ref{prop:green-existence-app}, with
$d\geq2$. Let $U_S\subset\mathbb R^{d-1}$ be compact and let
$X_S:U_S\to\Sigma_S\subset\overline\Omega$ be bi-Lipschitz with constant
$L\geq1$, that is,
\begin{equation}
\label{eq:bilipschitz}
    L^{-1}|\eta-\eta'|
    \leq
    |X_S(\eta)-X_S(\eta')|
    \leq
    L|\eta-\eta'|
    \qquad\text{for all }\eta,\eta'\in U_S,
\end{equation}
and let $K_\alpha$ include the surface Jacobian $J_S$ as in
\eqref{eq:restricted-kernel}. Then for every finite $T>0$ there is a constant
$C_T$, independent of the observation point, such that for every
$x=X_O(\xi)\in\overline\Omega$
\begin{equation}
\label{eq:surface-singularity}
    \int_{U_S}
        |K_\alpha(\xi,\eta,\tau)|
    \,d\eta
    =
    \int_{\Sigma_S}
        |G_\alpha(x,y,\tau)|
    \,d\sigma(y)
    \leq
    C_T\tau^{-1/2},
    \qquad 0<\tau\leq T.
\end{equation}
Consequently,
\begin{equation}
\label{eq:surface-time-integrability}
    \sup_{\xi\in U_O}
    \int_0^T\int_{U_S}
        |K_\alpha(\xi,\eta,\tau)|
    \,d\eta\,d\tau
    \leq
    2C_T\sqrt T
    <\infty.
\end{equation}
\end{prop2*}

\begin{proof}
The equality in \eqref{eq:surface-singularity} is the area formula
\eqref{eq:area-formula}. Fix $x\in\overline\Omega$. Since $\Sigma_S$ is
compact, there is
$\eta_0\in U_S$ with $|x-X_S(\eta_0)|=\operatorname{dist}(x,\Sigma_S)$. For
every $\eta\in U_S$ the triangle inequality and the choice of $\eta_0$ give
\[
    |X_S(\eta)-X_S(\eta_0)|
    \leq
    |X_S(\eta)-x|+|x-X_S(\eta_0)|
    \leq
    2\,|X_S(\eta)-x|,
\]
and the lower bound in \eqref{eq:bilipschitz} turns this into
\begin{equation}
\label{eq:chart-distance}
    |X_S(\eta)-x|
    \geq
    \frac{1}{2L}\,|\eta-\eta_0|
    \qquad\text{for all }\eta\in U_S.
\end{equation}
If $x\in\Sigma_S$, then $x=X_S(\eta_0)$ and the factor $2$ can be dropped.

The upper Lipschitz bound implies that $DX_S$ exists almost everywhere and
has operator norm at most $L$. Therefore,
\begin{equation}
\label{eq:jacobian-bound}
    J_S(\eta)
    =
    \sqrt{\det\!\big(DX_S(\eta)^\top DX_S(\eta)\big)}
    \leq
    L^{d-1}.
\end{equation}

The Gaussian bound \eqref{eq:gaussian} holds pointwise on
$\overline\Omega\times\overline\Omega$, so it may be evaluated at
$y=X_S(\eta)$. Combining it with \eqref{eq:chart-distance} gives
\[
    |G_\alpha(x,X_S(\eta),\tau)|
    \leq
    C_G\tau^{-d/2}
    \exp\!\Big(-\frac{c\,|\eta-\eta_0|^2}{4L^2\tau}\Big).
\]
Multiplying by \eqref{eq:jacobian-bound}, integrating over $U_S$, and
enlarging the domain of integration to $\mathbb R^{d-1}$ gives
\begin{align}
    \int_{U_S}|K_\alpha(\xi,\eta,\tau)|\,d\eta
    &\leq
    C_GL^{d-1}\tau^{-d/2}
    \int_{\mathbb R^{d-1}}
        \exp\!\Big(-\frac{c\,|\eta-\eta_0|^2}{4L^2\tau}\Big)
    \,d\eta
    \nonumber\\
    &=
    C_GL^{d-1}\tau^{-d/2}
    \Big(\frac{4\pi L^2\tau}{c}\Big)^{(d-1)/2}
    \nonumber\\
    &=
    C_GL^{d-1}\Big(\frac{4\pi L^2}{c}\Big)^{(d-1)/2}
    \tau^{-1/2}.
\end{align}
Here we used the Gaussian integral
$\int_{\mathbb R^{n}}e^{-|z|^2/(4s)}\,dz=(4\pi s)^{n/2}$ with $n=d-1$ and
$s=L^2\tau/c$. Defining
$C_T=C_GL^{d-1}(4\pi L^2/c)^{(d-1)/2}$ proves
\eqref{eq:surface-singularity}. This constant does not depend on $x$.

It remains to integrate the bound in time. Uniformly in $\xi$,
\[
    \int_0^T\int_{U_S}|K_\alpha(\xi,\eta,\tau)|\,d\eta\,d\tau
    \leq
    C_T\int_0^T\tau^{-1/2}\,d\tau
    =
    2C_T\sqrt T,
\]
which is \eqref{eq:surface-time-integrability}.
\end{proof}

When the manifolds are separated, the exponential decay in the Gaussian
bound removes even the remaining short-time singularity.

\begin{remark}[Separated source and observation manifolds]
\label{rem:separated-manifolds}
If $\operatorname{dist}(\Sigma_S,\Sigma_O)=\delta>0$, then for
$x\in\Sigma_O$, $y\in\Sigma_S$, and $0<\tau\leq T$, \eqref{eq:gaussian}
gives
\begin{equation}
\label{eq:separated-bound}
    |G_\alpha(x,y,\tau)|
    \leq
    C_G\tau^{-d/2}e^{-c\delta^2/\tau}
    \leq
    C_G\Big(\frac{d}{2ec\delta^2}\Big)^{d/2},
\end{equation}
where the second inequality is $s^{d/2}e^{-s}\leq(d/2e)^{d/2}$ with
$s=c\delta^2/\tau$. Hence
\[
    \sup_{0<\tau\leq T}\;
    \sup_{\xi\in U_O}
    \int_{U_S}|K_\alpha(\xi,\eta,\tau)|\,d\eta
    \leq
    C_G|\Sigma_S|\Big(\frac{d}{2ec\delta^2}\Big)^{d/2}
    <\infty .
\]
The short-time singularity of Proposition~\ref{prop:restricted-integrability-app}
therefore appears only where the observation manifold touches the source
manifold.
\end{remark}

\subsection{Approximation of the restricted kernel}
\label{app:kernel-proof}

For heat, the abstract descriptor is $\alpha=\kappa$, and the spatial
operator $A_\kappa$ is the variational Neumann realization of
$-\nabla\cdot(\kappa\nabla)$. Its eigenpairs solve
\begin{equation}
\label{eq:eigenproblem}
    -\nabla\cdot\big(\kappa(x)\nabla e_k(x)\big)
    =\lambda_k e_k(x),
    \qquad
    (\kappa\nabla e_k)\cdot n=0
    \quad\text{on }\partial\Omega.
\end{equation}
For symmetric uniformly elliptic $\kappa$ on a bounded Lipschitz domain, this
operator is nonnegative and self-adjoint, and its resolvent is compact. The
spectral theorem therefore gives an orthonormal basis
$\{e_k\}_{k\geq1}$ of $L^2(\Omega)$, with
$0=\lambda_1\leq\lambda_2\leq\cdots$ and $\lambda_k\to\infty$. For every
$\tau>0$, the corresponding heat kernel has the expansion
\begin{equation}
\label{eq:heat-spectral-app}
    G_\kappa(x,y,\tau)
    =
    \sum_{k=1}^{\infty}
        e^{-\lambda_k\tau}e_k(x)e_k(y),
\end{equation}
with convergence in $L^2(\Omega\times\Omega)$. These standard spectral facts
for elliptic Neumann operators and their heat kernels are given, for example,
by \citet{davies1989heat} and \citet{ouhabaz2009analysis}.

Before restricting this series to the manifolds, we strengthen its
convergence. The Gaussian bound implies that the eigenfunctions grow at
most polynomially in $\lambda_k$ and that the eigenvalues grow at least
like $k^{2/d}$. The exponential factor then gives uniform convergence for
every time lag bounded away from zero. This makes the pointwise restriction
well defined and reduces the truncation estimate to a tail sum. The
following is the full version of Theorem~\ref{thm:spectral-approx} in the
main text. It gives a uniform pointwise bound, from which the stated
$L^2$ bound follows.

\begin{thm1*}[Approximation of the restricted heat kernel]
\phantomsection\fixedlabel{1}{thm:spectral-approx-app}
Assume the setting of Proposition~\ref{prop:green-existence-app} with
$b=0$. Let $X_S:U_S\to\Sigma_S$ and $X_O:U_O\to\Sigma_O$ be bi-Lipschitz
$C^1$ charts on compact parameter sets, with $X_S$ satisfying the assumptions
of Proposition~\ref{prop:restricted-integrability-app}, and let $\psi_k$,
$\varphi_k$, and $K_{\alpha,M}$ be defined as in
\eqref{eq:restricted-separable}. Then the eigenfunctions $e_k$ are
continuous on $\overline\Omega$, and for every $\tau>0$ the series in
\eqref{eq:restricted-separable} converges uniformly on $U_O\times U_S$.
Moreover, there is a constant $c>0$ depending only on $\Omega$,
$\kappa_-$, $\kappa_+$, and the Gaussian constants in
\eqref{eq:gaussian}. For every $\tau_0>0$, there is a constant
$C_{\tau_0}$ such that, for all $M\geq1$,
\begin{equation}
\label{eq:spectral-bound-sup}
    \sup_{\tau\geq\tau_0}
    \sup_{\xi\in U_O,\,\eta\in U_S}
    \frac{
        |K_\alpha(\xi,\eta,\tau)-K_{\alpha,M}(\xi,\eta,\tau)|
    }{J_S(\eta)}
    \leq
    C_{\tau_0}e^{-c\tau_0M^{2/d}}.
\end{equation}
In particular,
\begin{equation*}
\tag*{\ref{eq:spectral-bound}}
    \sup_{\tau\geq\tau_0}
    \left\|
        K_\alpha(\cdot,\cdot,\tau)-K_{\alpha,M}(\cdot,\cdot,\tau)
    \right\|_{L^2(U_O\times U_S)}
    \leq
    |U_O|^{1/2}\|J_S\|_{L^2(U_S)}
    C_{\tau_0}e^{-c\tau_0M^{2/d}}.
\end{equation*}
\end{thm1*}

\begin{proof}
Let $C_G$ and $c_G$ denote the constants in \eqref{eq:gaussian} for
$0<t\leq1$. We use the variational Neumann operator $A_\kappa$ from
\eqref{eq:eigenproblem}, with orthonormal eigenpairs
$(\lambda_k,e_k)$ ordered as
$0=\lambda_1<\lambda_2\leq\cdots$.

We first use the Gaussian bound to control the heat operator from $L^2$ to
$L^\infty$. For $f\in L^2(\Omega)$, $x\in\overline\Omega$, and
$0<t\leq1$, the Cauchy--Schwarz inequality gives
\begin{align}
    |(e^{-tA_\kappa}f)(x)|
    &=
    \left|\int_\Omega G_\alpha(x,y,t)f(y)\,dy\right|
    \nonumber\\
    &\leq
    C_Gt^{-d/2}
    \left(\int_{\mathbb R^d}e^{-2c_G|x-y|^2/t}\,dy\right)^{1/2}
    \|f\|_{L^2(\Omega)}
    \nonumber\\
    &=
    C_G\left(\frac{\pi}{2c_G}\right)^{d/4}
    t^{-d/4}\|f\|_{L^2(\Omega)}
    =:
    C_1t^{-d/4}\|f\|_{L^2(\Omega)}.
\label{eq:ultracontractive}
\end{align}
The kernel is bounded and continuous on
$\overline\Omega\times\overline\Omega$ for fixed $t>0$. Dominated
convergence therefore shows that $e^{-tA_\kappa}f$ is continuous on
$\overline\Omega$.

Since $e^{-tA_\kappa}e_k=e^{-\lambda_kt}e_k$, equation
\eqref{eq:ultracontractive} yields
\[
    \|e_k\|_{L^\infty(\Omega)}
    \leq
    C_1t^{-d/4}e^{\lambda_kt},
    \qquad 0<t\leq1.
\]
It also shows that each $e_k$ has a continuous representative on
$\overline\Omega$. When $\lambda_k\geq1$, we choose
$t=1/\lambda_k$. Otherwise, we choose $t=1$. In both cases,
\begin{equation}
\label{eq:eigen-sup}
    \|e_k\|_{L^\infty(\Omega)}
    \leq
    eC_1(1+\lambda_k)^{d/4}.
\end{equation}
Thus, pointwise restriction of an eigenfunction introduces only polynomial
growth in $\lambda_k$.

The Gaussian bound also gives the required growth of the eigenvalues. The
squared Hilbert--Schmidt norm of $e^{-(t/2)A_\kappa}$ equals the squared
$L^2(\Omega\times\Omega)$ norm of its kernel. Hence, for $0<t\leq1$,
\begin{align}
    \sum_{j=1}^\infty e^{-\lambda_jt}
    &=
    \int_\Omega\int_\Omega
        |G_\alpha(x,y,t/2)|^2
    \,dy\,dx
    \nonumber\\
    &\leq
    C_G^2\left(\frac{t}{2}\right)^{-d}
    \int_\Omega\int_{\mathbb R^d}
        e^{-4c_G|x-y|^2/t}
    \,dy\,dx
    \nonumber\\
    &=
    C_G^2|\Omega|2^d
    \left(\frac{\pi}{4c_G}\right)^{d/2}t^{-d/2}
    =:
    C_3t^{-d/2}.
\end{align}
Because the eigenvalues are nondecreasing, the sum on the left is at least
$ke^{-\lambda_kt}$. If $\lambda_k\geq1$, choosing $t=1/\lambda_k$ gives
\[
    ke^{-1}
    \leq
    C_3\lambda_k^{d/2},
    \qquad
    \lambda_k
    \geq
    (eC_3)^{-2/d}k^{2/d}.
\]
Only finitely many indices $k\geq2$ satisfy $\lambda_k<1$. Decreasing the
constant to include these indices gives some $c_0>0$ such that
\begin{equation}
\label{eq:weyl}
    \lambda_k\geq c_0k^{2/d}
    \qquad\text{for every }k\geq2.
\end{equation}

We now combine these two estimates. Fix $\tau_0>0$ and let
$\tau\geq\tau_0$. Equation~\eqref{eq:eigen-sup} implies
\begin{align}
    e^{-\lambda_k\tau}|e_k(x)e_k(y)|
    &\leq
    e^2C_1^2(1+\lambda_k)^{d/2}e^{-\lambda_k\tau_0}
    \nonumber\\
    &\leq
    C_{\tau_0}'e^{-\lambda_k\tau_0/2},
\label{eq:term-bound}
\end{align}
uniformly for $x,y\in\overline\Omega$, where
\[
    C_{\tau_0}'
    :=
    e^2C_1^2
    \sup_{\lambda\geq0}
        (1+\lambda)^{d/2}e^{-\lambda\tau_0/2}
    <\infty.
\]
Together with \eqref{eq:weyl}, this shows that the series
\[
    S(x,y,\tau)
    :=
    \sum_{k=1}^\infty e^{-\lambda_k\tau}e_k(x)e_k(y)
\]
converges absolutely and uniformly on
$\overline\Omega\times\overline\Omega\times[\tau_0,\infty)$. Its sum is
therefore continuous. By \eqref{eq:heat-spectral-app}, the same series
converges to $G_\alpha(\cdot,\cdot,\tau)$ in
$L^2(\Omega\times\Omega)$. Since both functions are continuous and agree
almost everywhere, they agree everywhere. As $\tau_0>0$ was arbitrary,
\begin{equation}
\label{eq:pointwise-spectral}
    G_\alpha(x,y,\tau)
    =
    S(x,y,\tau)
    \qquad
    \text{for all }x,y\in\overline\Omega\text{ and }\tau>0.
\end{equation}

Evaluating \eqref{eq:pointwise-spectral} at $x=X_O(\xi)$ and
$y=X_S(\eta)$, then multiplying by $J_S(\eta)$, gives
\begin{equation}
\label{eq:restricted-spectral-app}
    K_\alpha(\xi,\eta,\tau)
    =
    \sum_{k=1}^\infty e^{-\lambda_k\tau}
        \underbrace{e_k(X_O(\xi))}_{\psi_k(\xi)}
        \underbrace{e_k(X_S(\eta))J_S(\eta)}_{\varphi_k(\eta)},
\end{equation}
with uniform convergence on $U_O\times U_S$.

It remains to bound the omitted terms. Every index in the tail satisfies
$k\geq2$ when $M\geq1$. Using \eqref{eq:term-bound} and
\eqref{eq:weyl}, with $a=c_0\tau_0/2$, gives
\begin{align}
    \frac{
        |K_\alpha(\xi,\eta,\tau)-K_{\alpha,M}(\xi,\eta,\tau)|
    }{J_S(\eta)}
    &\leq
    \sum_{k>M}e^{-\lambda_k\tau}
        |e_k(X_O(\xi))e_k(X_S(\eta))|
    \nonumber\\
    &\leq
    C_{\tau_0}'\sum_{k>M}e^{-ak^{2/d}}.
\end{align}
For $k>M$, we have $k^{2/d}\geq M^{2/d}$ and therefore
\[
    e^{-ak^{2/d}}
    \leq
    e^{-aM^{2/d}/2}e^{-ak^{2/d}/2}.
\]
Summing this inequality yields
\begin{equation}
\label{eq:tail-sum}
    \sum_{k>M}e^{-ak^{2/d}}
    \leq
    e^{-aM^{2/d}/2}
    \sum_{k=1}^\infty e^{-ak^{2/d}/2}
    =:
    S_a e^{-aM^{2/d}/2},
    \qquad S_a<\infty.
\end{equation}
Combining these estimates proves \eqref{eq:spectral-bound-sup} with
$C_{\tau_0}=C_{\tau_0}'S_a$ and $c=c_0/4$. In particular, $c$ is
independent of $\tau_0$. Finally, multiplying the pointwise estimate by
$J_S(\eta)$ and taking its $L^2(U_O\times U_S)$ norm gives
\eqref{eq:spectral-bound}.
\end{proof}

\subsection{Optimization with additional search rank}
\label{app:search-rank}

We provide some intuition for the effect of search rank observed in
Table~\ref{tab:rank-full} by examining the optimization landscape of the
factorized objective. We treat the factors as free variables and hold the
rates fixed. In the conditional model, however, the factors are outputs
of a neural network, so the connection to our experiments is only qualitative.

For fixed rates, the observations depend linearly on the coefficient
matrices $A=(A_1,\ldots,A_M)$. The squared loss
\begin{equation}
\label{eq:app-loss}
\mathcal L(A)
=
\sum_i\lVert\mathcal P_i(A)-v_i\rVert_2^2,
\qquad \mathcal P_i\ \text{linear},
\end{equation}
is therefore a convex quadratic in $A$. Introducing the factorizations
$A_k=U_kV_k^\top$ gives
\begin{equation}
\label{eq:app-factored}
f(U,V)
=
\mathcal L
\left(
U_1V_1^\top,\ldots,U_MV_M^\top
\right),
\end{equation}
which is generally nonconvex in
$U_k\in\mathbb R^{D_O\times L}$ and
$V_k\in\mathbb R^{D_S\times L}$. In coefficient space, this resembles
asymmetric low-rank matrix sensing with one measurement block per
temporal mode. Our instance is highly ill-conditioned. The temporal
design matrix in Appendix~\ref{app:toy-details} has condition number
$5.8\times10^{14}$, and the objective has no balancing regularizer.

Let $G_k:=\nabla_{A_k}\mathcal L$. The first-order stationary conditions
for the factored objective are
\begin{equation}
\label{eq:app-stationary}
G_kV_k=0,
\qquad
G_k^\top U_k=0
\qquad
\text{for every }k.
\end{equation}
These conditions are weaker than $G_k=0$. For example, they hold
for a given slot when $U_k=V_k=0$, regardless of its coefficient
gradient. The factored objective can therefore be stationary even
when the corresponding coefficient matrices do not minimize
\eqref{eq:app-loss} \citep[eq.~(8)]{ha2020equivalence}.

Rank deficiency provides a useful connection between the two
objectives. For a twice differentiable function $g$ and a
factorization $X=AB^\top$ with $L$ columns, every second-order
stationary point of $g(AB^\top)$ satisfying
$\operatorname{rank}(AB^\top)<L$ also satisfies
$\nabla g(AB^\top)=0$ \citep[Theorem~1]{ha2020equivalence}.
Related arguments appear for Burer--Monteiro factorizations and
regularized factorization models
\citep{burer2005local,bhojanapalli2018smoothed,haeffele2015global}.
Intuitively, unused factor directions allow the second-order
behavior to detect coefficient gradients that are hidden from
the first-order conditions.

Applying this result to one rate slot while holding the others fixed, a
second-order stationary point with $\operatorname{rank}(A_k)<L$ must
satisfy $G_k=0$. If every slot is rank deficient, then
$\nabla_A\mathcal L=0$, and convexity makes $A$ a global minimizer of
\eqref{eq:app-loss}. At $L=1$, this mechanism applies only to zero-rank
slots, whereas additional columns allow unused directions even for nonzero
coefficient matrices.

Related theoretical motivation for increasing search rank comes
from positive-semidefinite Burer--Monteiro factorizations of
smooth, strongly convex objectives. Sufficient rank
overparameterization makes every second-order stationary point
globally optimal, whereas the corresponding guarantee without
overparameterization requires an objective with condition number
below three \citep{zhang2025improved}. These results suggest
potential optimization benefits from increasing the search rank.
However, they do not cover our rectangular, multi-rate factorization,
and our coefficient-space loss need not satisfy the required
strong-convexity assumptions. Nor is there a guarantee that our
optimizer reaches a rank-deficient second-order stationary point.

The optimization landscape also depends on properties of the measurement
operator. Under suitable restricted-isometry assumptions, factored
matrix-sensing problems have no spurious local minima and their
saddle points are strict \citep{ge2017nospurious}. Counterexamples
show that weaker isometry conditions can still permit spurious
local minima \citep{zhang2018howmuch}. We have not established
these assumptions for our highly ill-conditioned temporal design,
so neither result directly characterizes our loss. Optimization
may also be slow without convergence to a spurious minimum.
Gradient descent can be sensitive to ill-conditioning in low-rank
estimation \citep{tong2021accelerating}. In related factored
linear models, singular directions may be learned sequentially
with plateaus between them \citep{saxe2014exact}, and even strict
saddle points can take a long time to escape \citep{du2017gradient}.
The final loss of a finite training run therefore cannot
distinguish these plateaus from convergence to a poor local minimum.

The results in Table~\ref{tab:rank-full} are consistent with the theoretical motivation for increasing search rank. At $L=1$, the attained loss remains more than an order of magnitude above the convex coefficient-space optimum, whereas for $L\geq8$, it approaches that optimum to within approximately one percent. Although increasing $L$ affects both the set of representable coefficient matrices and the optimization landscape, the observed improvement provides empirical support for the search rank used in the remaining experiments. A theoretical characterization of the observed dependence on search rank is left to future work.

\subsection{Exact streaming evaluation}
\label{app:streaming-proof}

The result below restates Proposition~\ref{prop:streaming} from the main
text. The proof separates the accumulated history from the final time
interval.

\begin{prop3*}[Exact streaming evaluation]
\phantomsection\fixedlabel{3}{prop:streaming-app}
Let $t_i=i\Delta$ be a uniform time grid and define
\begin{align}
    \widehat f(t)
    &:={}
    \int_{U_S}
        \boldsymbol{\Phi}_S(\eta)f(\eta,t)
    \,d\eta,
\\
    g_k(t)
    &:={}
    \int_0^t
        e^{-\lambda_k(t-s)}
        V_k(\alpha)^\top\widehat f(s)
    \,ds,
    \qquad k=1,\ldots,M.
\end{align}
Suppose that the source is piecewise constant in time, with projected value
$\widehat f^i$ on $[t_i,t_{i+1})$. Then, writing
$g_k^i:=g_k(t_i)$,
\begin{equation*}
\tag*{\ref{eq:stream-recurrence}}
    g_k^{i+1}
    =
    e^{-\lambda_k\Delta}g_k^i
    +
    \left(\int_0^\Delta e^{-\lambda_k s}\,ds\right)
    V_k(\alpha)^\top\widehat f^i.
\end{equation*}
The observation at $t_i$ is
\begin{equation}
    v_\alpha(\xi,t_i)
    =
    \boldsymbol{\Phi}_O(\xi)^\top
    \sum_{k=1}^{M}U_k(\alpha)g_k^i.
\end{equation}
Thus the recurrence and readout are exactly equal to batch evaluation of the
kernel for a piecewise-constant source.
\end{prop3*}

\begin{proof}
We begin from the definition of the state at $t_{i+1}$ and split its time
integral at $t_i$:
\begin{align}
    g_k(t_{i+1})
    &={}
    \int_0^{t_{i+1}}
        e^{-\lambda_k(t_{i+1}-s)}
        V_k(\alpha)^\top\widehat f(s)
    \,ds
    \nonumber\\
    &={}
    \int_0^{t_i}
        e^{-\lambda_k(t_{i+1}-s)}
        V_k(\alpha)^\top\widehat f(s)
    \,ds
    +
    \int_{t_i}^{t_{i+1}}
        e^{-\lambda_k(t_{i+1}-s)}
        V_k(\alpha)^\top\widehat f(s)
    \,ds.
\end{align}
Since $t_{i+1}=t_i+\Delta$, the first integral becomes
\begin{align}
    &\int_0^{t_i}
        e^{-\lambda_k(t_{i+1}-s)}
        V_k(\alpha)^\top\widehat f(s)
    \,ds
    \nonumber\\
    &\quad={}
    \int_0^{t_i}
        e^{-\lambda_k\Delta}
        e^{-\lambda_k(t_i-s)}
        V_k(\alpha)^\top\widehat f(s)
    \,ds
    \nonumber\\
    &\quad={}
    e^{-\lambda_k\Delta}
    \int_0^{t_i}
        e^{-\lambda_k(t_i-s)}
        V_k(\alpha)^\top\widehat f(s)
    \,ds
    \nonumber\\
    &\quad={}
    e^{-\lambda_k\Delta}g_k(t_i)
    =e^{-\lambda_k\Delta}g_k^i.
\end{align}
On the interval $[t_i,t_{i+1})$, piecewise constancy gives
$\widehat f(s)=\widehat f^i$. With the change of variables
$r=t_{i+1}-s$, the second integral therefore becomes
\begin{align}
    &\int_{t_i}^{t_{i+1}}
        e^{-\lambda_k(t_{i+1}-s)}
        V_k(\alpha)^\top\widehat f(s)
    \,ds
    \nonumber\\
    &\quad={}
    V_k(\alpha)^\top\widehat f^i
    \int_{t_i}^{t_{i+1}}
        e^{-\lambda_k(t_{i+1}-s)}
    \,ds
    \nonumber\\
    &\quad={}
    V_k(\alpha)^\top\widehat f^i
    \int_{\Delta}^{0}e^{-\lambda_k r}(-\,dr)
    \nonumber\\
    &\quad={}
    \left(\int_0^\Delta e^{-\lambda_k s}\,ds\right)
    V_k(\alpha)^\top\widehat f^i.
\end{align}
Combining the two parts gives
\begin{align}
    g_k^{i+1}
    &={}
    e^{-\lambda_k\Delta}g_k^i
    +
    \left(\int_0^\Delta e^{-\lambda_k s}\,ds\right)
    V_k(\alpha)^\top\widehat f^i,
\end{align}
which is \eqref{eq:stream-recurrence}.

It remains to verify the readout. Substituting the kernel model
\eqref{eq:kernel-model} into the source-to-observation integral gives
\begin{align}
    v_\alpha(\xi,t_i)
    &={}
    \int_0^{t_i}\int_{U_S}
        K_{\theta,\alpha}(\xi,\eta,t_i-s)f(\eta,s)
    \,d\eta\,ds
    \nonumber\\
    &={}
    \int_0^{t_i}\int_{U_S}
        \boldsymbol{\Phi}_O(\xi)^\top
        \sum_{k=1}^{M}
            e^{-\lambda_k(t_i-s)}
            U_k(\alpha)V_k(\alpha)^\top
            \boldsymbol{\Phi}_S(\eta)
            f(\eta,s)
    \,d\eta\,ds
    \nonumber\\
    &={}
    \boldsymbol{\Phi}_O(\xi)^\top
    \sum_{k=1}^{M}U_k(\alpha)
    \int_0^{t_i}
        e^{-\lambda_k(t_i-s)}V_k(\alpha)^\top
        \left(
            \int_{U_S}\boldsymbol{\Phi}_S(\eta)f(\eta,s)\,d\eta
        \right)
    \,ds
    \nonumber\\
    &={}
    \boldsymbol{\Phi}_O(\xi)^\top
    \sum_{k=1}^{M}U_k(\alpha)
    \int_0^{t_i}
        e^{-\lambda_k(t_i-s)}V_k(\alpha)^\top\widehat f(s)
    \,ds
    \nonumber\\
    &={}
    \boldsymbol{\Phi}_O(\xi)^\top
    \sum_{k=1}^{M}U_k(\alpha)g_k^i.
\end{align}
This is the stated observation formula.
\end{proof}

The next result restates Proposition~\ref{prop:streaming-stability} from the
main text.

\begin{prop4*}[Stability of the streaming state]
\phantomsection\fixedlabel{4}{prop:streaming-stability-app}
Let $\rho_k=e^{-\lambda_k\Delta}$. If
$\operatorname{Re}(\lambda_k)\geq0$, perturbations evolving under the same
source satisfy
\begin{equation}
    \lVert\delta g_k^{i+1}\rVert
    \leq
    |\rho_k|\,\lVert\delta g_k^i\rVert
    \leq
    \lVert\delta g_k^i\rVert.
\end{equation}
If $\operatorname{Re}(\lambda_k)>0$ and
$\sup_i\lVert V_k(\alpha)^\top\widehat f^i\rVert<\infty$, then the states
$\{g_k^i\}_{i\geq0}$ are uniformly bounded.
\end{prop4*}

\begin{proof}
Let $g_k^i$ and $\widetilde g_k^i$ be two states driven by the same source,
and set $\delta g_k^i:=\widetilde g_k^i-g_k^i$. Their updates are
\begin{align}
    g_k^{i+1}
    &={}
    \rho_k g_k^i
    +
    \left(\int_0^\Delta e^{-\lambda_k s}\,ds\right)
    V_k(\alpha)^\top\widehat f^i,
    \nonumber\\
    \widetilde g_k^{i+1}
    &={}
    \rho_k\widetilde g_k^i
    +
    \left(\int_0^\Delta e^{-\lambda_k s}\,ds\right)
    V_k(\alpha)^\top\widehat f^i.
\end{align}
Subtracting the first equation from the second cancels the common input and
gives
\begin{align}
    \delta g_k^{i+1}
    &={}
    \widetilde g_k^{i+1}-g_k^{i+1}
    \nonumber\\
    &={}
    \rho_k\widetilde g_k^i-\rho_k g_k^i
    \nonumber\\
    &={}
    \rho_k\bigl(\widetilde g_k^i-g_k^i\bigr)
    =\rho_k\delta g_k^i.
\end{align}
Moreover,
\begin{align}
    |\rho_k|
    &=|e^{-\lambda_k\Delta}|
    =e^{-\operatorname{Re}(\lambda_k)\Delta}
    \leq 1,
\end{align}
because $\operatorname{Re}(\lambda_k)\geq0$. It follows that
\begin{align}
    \lVert\delta g_k^{i+1}\rVert
    &=\lVert\rho_k\delta g_k^i\rVert
    =|\rho_k|\,\lVert\delta g_k^i\rVert
    \leq\lVert\delta g_k^i\rVert.
\end{align}

For the second claim, let $a_k:=\operatorname{Re}(\lambda_k)>0$ and
$B_k:=\sup_i\lVert V_k(\alpha)^\top\widehat f^i\rVert$. Since the source is
piecewise constant, the same bound holds at almost every time. The definition
of the state then yields
\begin{align}
    \lVert g_k(t)\rVert
    &=
    \left\lVert
        \int_0^t
            e^{-\lambda_k(t-s)}
            V_k(\alpha)^\top\widehat f(s)
        \,ds
    \right\rVert
    \nonumber\\
    &\leq
    \int_0^t
        |e^{-\lambda_k(t-s)}|
        \lVert V_k(\alpha)^\top\widehat f(s)\rVert
    \,ds
    \nonumber\\
    &\leq
    B_k\int_0^t e^{-a_k(t-s)}\,ds
    \nonumber\\
    &=
    \frac{B_k}{a_k}\left(1-e^{-a_k t}\right)
    \leq
    \frac{B_k}{a_k}.
\end{align}
Setting $t=t_i$ proves that the sequence $\{g_k^i\}_{i\geq0}$ is uniformly
bounded.
\end{proof}

\subsection{Extensions to mild nonlinearities}
\label{app:nonlinear-proofs}

This section develops the two nonlinear extensions summarized in
Section~\ref{sec:nonlinear}. The first treats a nonlinear effect as a feedback
source and alternates it with the learned linear propagation. The second uses
a Kirchhoff transformation to rewrite a class of nonlinear heat equations in
terms of a transformed field. Both constructions reuse the learned linear
response. The reference medium is fixed throughout, so we omit the common
argument $\alpha_0$ from the factors and states below.

\subsubsection{Radiative boundary feedback}
\label{app:radiative-feedback}

Radiative heat loss provides a source-like nonlinearity whose strength grows
rapidly with temperature and can therefore alter the peak response even when
conduction is well described by a linear model. Given an applied flux $q$, we
write the effective source as
\begin{equation}
\label{eq:radiative-feedback}
    f=q+\mathcal N(v_\alpha),
    \qquad
    \mathcal N(v_\alpha)
    :=-\sigma
    \big((v_\alpha+u_\infty)^4-u_\infty^4\big),
\end{equation}
where $u_\infty$ is the ambient temperature and $\mathcal N$ is the
Stefan--Boltzmann radiative flux~\citep{necati1993heat}. If the source and
observation manifolds coincide, this flux can be evaluated from the predicted
temperature and fed back as a source.

In the spirit of operator splitting~\citep{strang1968construction}, we
alternate linear propagation with this nonlinear source update. We use a
semi-implicit step based on a linearization at the previous prediction
\citep{ascher1997implicit}. Evaluating the full convolution at each step would
revisit the entire source history; the streaming representation instead
reduces the update to a fixed-cost state recurrence and a small solve in
coefficient space.

For the coefficient-space construction, we assume that the coincident source
and observation manifolds use the same orthonormal cosine basis. We retain the
$S/O$ subscripts to show the direction of each projection.

Let
\begin{equation}
    w^i:=\sum_{k=1}^{M}U_kg_k^i,
    \qquad
    v^i(\xi):=\boldsymbol{\Phi}_O(\xi)^\top w^i.
\end{equation}
We write the projected applied flux and radiative flux as
\begin{align}
    \widehat q^{\,i}
    &:={}
    \int_{U_S}\boldsymbol{\Phi}_S(\xi)q(\xi,t_i)\,d\xi,
    \\
    \widehat{\mathcal N}(w)
    &:={}
    \int_{U_S}
        \boldsymbol{\Phi}_S(\xi)
        \mathcal N\!\left(\boldsymbol{\Phi}_O(\xi)^\top w\right)
    \,d\xi.
\end{align}
At time $t_i$, set $\widehat{\mathcal N}^{\,i}:=
\widehat{\mathcal N}(w^i)$. Its derivative with respect to $w$ is
\begin{equation}
    J^i
    :={}
    \int_{U_S}
        \boldsymbol{\Phi}_S(\xi)\mathcal N'(v^i(\xi))
        \boldsymbol{\Phi}_O(\xi)^\top
    \,d\xi.
\end{equation}

To include the unknown radiative flux at $t_{i+1}$, we linearly interpolate
the source over $[t_i,t_{i+1}]$. With $\nu_k:=\lambda_k\Delta$, define
\begin{align}
    \gamma_0(\nu)
    &:={}
    \frac{1-(1+\nu)e^{-\nu}}{\nu^2},
    &
    \gamma_1(\nu)
    &:={}
    \frac{\nu-1+e^{-\nu}}{\nu^2},
    \\
    B_\Delta
    &:={}
    \sum_{k=1}^{M}\gamma_1(\nu_k)U_kV_k^\top,
\end{align}
where $\gamma_0(0)=\gamma_1(0)=1/2$ by continuity. All terms that are known
at the beginning of the step are collected in
\begin{align}
    \bar g_k^{i+1}
    &:={}
    e^{-\nu_k}g_k^i
    +\Delta\gamma_0(\nu_k)V_k^\top
        \left(\widehat q^{\,i}+\widehat{\mathcal N}^{\,i}\right)
    +\Delta\gamma_1(\nu_k)V_k^\top\widehat q^{\,i+1},
    \\
    \bar w^{i+1}
    &:={}
    \sum_{k=1}^{M}U_k\bar g_k^{i+1}.
\end{align}

\begin{proposition}[Semi-implicit radiative update]
\label{prop:radiation}
A first-order expansion of the radiative flux at $v^i$ gives the coefficient
update
\begin{equation}
\label{eq:radiation-solve}
    \left(I-\Delta B_\Delta J^i\right)w^{i+1}
    =\bar w^{i+1}
    +\Delta B_\Delta
        \left(\widehat{\mathcal N}^{\,i}-J^iw^i\right).
\end{equation}
Thus each time step requires one linear solve in the observation-basis
coefficients.
\end{proposition}

\begin{proof}
For $0\leq r\leq\Delta$, approximate the projected source by
\begin{equation}
    \widehat f(t_i+r)
    \approx
    \left(1-\frac{r}{\Delta}\right)\widehat f^i
    +\frac{r}{\Delta}\widehat f^{i+1}.
\end{equation}
Splitting the state integral at $t_i$ as in Proposition~\ref{prop:streaming-app}
then gives
\begin{align}
    g_k^{i+1}
    &\approx
    e^{-\nu_k}g_k^i
    +V_k^\top\widehat f^i
        \int_0^\Delta e^{-\lambda_k(\Delta-r)}
        \left(1-\frac{r}{\Delta}\right)dr
    \nonumber\\
    &\qquad
    +V_k^\top\widehat f^{i+1}
        \int_0^\Delta e^{-\lambda_k(\Delta-r)}
        \frac{r}{\Delta}\,dr
    \nonumber\\
    &=
    e^{-\nu_k}g_k^i
    +\Delta\gamma_0(\nu_k)V_k^\top\widehat f^i
    +\Delta\gamma_1(\nu_k)V_k^\top\widehat f^{i+1}.
\end{align}
For radiative feedback,
\begin{equation}
    \widehat f^i=\widehat q^{\,i}+\widehat{\mathcal N}^{\,i},
    \qquad
    \widehat f^{i+1}
    =\widehat q^{\,i+1}+\widehat{\mathcal N}(w^{i+1}).
\end{equation}
Using the definition of $\bar g_k^{i+1}$ therefore yields
\begin{equation}
    g_k^{i+1}
    =\bar g_k^{i+1}
    +\Delta\gamma_1(\nu_k)V_k^\top
        \widehat{\mathcal N}(w^{i+1}).
\end{equation}
After multiplication by $U_k$ and summation over $k$, this becomes
\begin{equation}
\label{eq:radiation-w-proof}
    w^{i+1}
    =\bar w^{i+1}
    +\Delta B_\Delta\widehat{\mathcal N}(w^{i+1}).
\end{equation}
The first-order expansion at $w^i$ is
\begin{align}
    \widehat{\mathcal N}(w^{i+1})
    &\approx
    \widehat{\mathcal N}(w^i)
    +J^i(w^{i+1}-w^i)
    \nonumber\\
    &=
    \widehat{\mathcal N}^{\,i}+J^i(w^{i+1}-w^i).
\end{align}
Substituting this expression into \eqref{eq:radiation-w-proof} and collecting
the terms containing $w^{i+1}$ gives \eqref{eq:radiation-solve}.
\end{proof}

This feedback construction uses the linearity of the kernel in the source and
could also be evaluated from the full convolution. The streaming form avoids
repeating that convolution at every nonlinear step.

\subsubsection{Kirchhoff transformation}
\label{app:kirchhoff-construction}

The second construction concerns state-dependent medium properties, as arise
when a material's conductivity changes with temperature. Given a reference
medium $\alpha_0$, write
\begin{equation}
\label{eq:state_dep_properties}
    \alpha_u(x,t)=\mathcal A\big(\alpha_0(x),u(x,t)\big),
    \qquad
    \alpha_0:\Omega\to\mathbb R^p,
    \quad
    \mathcal A:\mathbb R^p\times\mathbb R\to\mathbb R^p.
\end{equation}
If $\alpha_u$ were known throughout the volume, we could evaluate a newly
conditioned kernel $K_{\alpha_u}$. In our setting, however, the constitutive
law $\mathcal A$ is known while $u$ is predicted only on the observation
manifold, so the updated volumetric medium cannot generally be recovered.

For heat conduction, useful special cases arise when the spatial and
temperature dependence separate. Let $\alpha_0(x)$ be the spatial reference
conductivity and let $\widetilde k(u)$ and $c(u)$ be positive scalar
conductivity and heat-capacity factors. We consider
\begin{equation}
\label{eq:temp_dep_conduct}
    \begin{gathered}
        \partial_tu
        =\frac{1}{c(u)}\nabla\cdot
        \left(\alpha_0(x)\widetilde k(u)\nabla u\right),
        \qquad
        \alpha_0:\Omega\to\mathbb R_{>0},
        \quad
        \widetilde k,c:\mathbb R\to\mathbb R_{>0},
        \\
        r(u):=\frac{\widetilde k(u)}{c(u)},
        \qquad
        \text{(A)}\quad r(u)\equiv1,
        \qquad
        \text{(B)}\quad r(u)=\widetilde k(u).
    \end{gathered}
\end{equation}
Setting~A describes conductivity and heat capacity that vary by the same
factor, so heat spreads at a temperature-independent rate
\citep{lesnic1995note}. Setting~B fixes $c\equiv1$ and varies only the
conductivity, so the propagation speed depends on the local temperature
\citep{fabre2017integral}.

Both cases can be treated through the classical Kirchhoff construction
\citep{carslaw1906introduction,necati1993heat}. Define
\begin{equation}
\label{eq:kirchhoff-transform}
    \Theta_{\widetilde k}(r):=\int_0^r\widetilde k(s)\,ds,
    \qquad
    \theta:=\Theta_{\widetilde k}(u).
\end{equation}
It is strictly increasing and can therefore be inverted pointwise. This
change of variables itself does not depend on the streaming formulation.

\begin{proposition}[Kirchhoff transformation]
\label{prop:kirchhoff}
The transformed field satisfies
\begin{equation}
\label{eq:kirchhoff-general}
    \partial_t\theta
    =r(u)\nabla\cdot\left(\alpha_0(x)\nabla\theta\right),
    \qquad
    r(u):=\frac{\widetilde k(u)}{c(u)}.
\end{equation}
The imposed normal flux is unchanged. If $c(u)=\widetilde k(u)$, the equation
for $\theta$ is exactly the linear heat equation. If $c(u)\equiv1$, it becomes
\begin{equation}
\label{eq:kirchhoff-clock}
    \partial_t\theta
    =m(\theta)\nabla\cdot\left(\alpha_0(x)\nabla\theta\right),
    \qquad
    m(\theta):=r\!\left(\Theta_{\widetilde k}^{-1}(\theta)\right)
    =\widetilde k\!\left(\Theta_{\widetilde k}^{-1}(\theta)\right).
\end{equation}
\end{proposition}

\begin{proof}
Differentiating the transform gives
\begin{equation}
    \Theta_{\widetilde k}'(u)=\widetilde k(u).
\end{equation}
The chain rule therefore gives
\begin{align}
    \nabla\theta
    &=\widetilde k(u)\nabla u,
    \label{eq:kirchhoff-gradient}
    \\
    \partial_t\theta
    &=\widetilde k(u)\partial_tu.
    \label{eq:kirchhoff-time}
\end{align}
It follows that
\begin{align}
    \nabla\cdot\left(\alpha_0(x)\widetilde k(u)\nabla u\right)
    &=\nabla\cdot\left(\alpha_0(x)\nabla\theta\right),
    \\
    \partial_tu
    &=\frac{1}{\widetilde k(u)}\partial_t\theta.
\end{align}
Substitution into \eqref{eq:temp_dep_conduct} gives
\begin{equation}
    \frac{1}{\widetilde k(u)}\partial_t\theta
    =\frac{1}{c(u)}\nabla\cdot
    \left(\alpha_0(x)\nabla\theta\right).
\end{equation}
Multiplying by $\widetilde k(u)$ proves \eqref{eq:kirchhoff-general}. The same
gradient identity gives
\begin{equation}
    -\alpha_0(x)\widetilde k(u)\nabla u\cdot n
    =-\alpha_0(x)\nabla\theta\cdot n,
\end{equation}
so the prescribed normal flux is unchanged. Setting
$c(u)=\widetilde k(u)$ or $c(u)=1$ gives the two stated special cases.
\end{proof}

The two settings now lead to different uses of the learned response. Writing
$v_\alpha^{\mathrm{lin}}$ for the raw prediction of the linear kernel, both
can be summarized as
\begin{equation}
\label{eq:kirchhoff-input-output}
    \big(v_\alpha^{\mathrm{lin}},\widetilde k,c\big)
    \xrightarrow{\ \mathcal T\ }
    v_\alpha.
\end{equation}
In setting~A, the transformation reduces the problem exactly to the linear
heat equation. The raw prediction is therefore the transformed trace,
$v_\alpha^{\mathrm{lin}}=\theta|_{\Sigma_O}$, and pointwise inversion gives
\begin{equation}
    v_\alpha
    =\Theta_{\widetilde k}^{-1}(v_\alpha^{\mathrm{lin}}).
\end{equation}
Thus $\mathcal T$ is exact and does not require the streaming state. In
setting~B, the multiplier $m(\theta)$ remains, so the transformed system does
not fully linearize. Because this multiplier is defined throughout the volume
while the model predicts only a trace, we introduce a trace-based
approximation and use the streaming state to advance it in time.

For the conductivity used in the experiments, we set
\begin{equation}
\label{eq:conductivity-law}
    \widetilde k(u)=1+\beta u,
\end{equation}
and assume that $1+\beta u>0$ on the relevant solution range. The transform,
its inverse, and the remaining multiplier when $c(u)=1$ are
\begin{align}
\label{eq:kirchhoff-explicit}
    \Theta_{\widetilde k}(u)
    &=u+\frac{\beta}{2}u^2,
    &
    \Theta_{\widetilde k}^{-1}(s)
    &=\frac{-1+\sqrt{1+2\beta s}}{\beta},
    \\
\label{eq:clock-explicit}
    m(s)
    &=\sqrt{1+2\beta s}.
\end{align}
The inverse formula is written for $\beta\neq0$. For $\beta=0$, the transform
is the identity and $m\equiv1$.

We evaluate $m$ on the predicted trace and use it to modify the corresponding
time derivative. To express this update, we first collect the streaming
variables in a single state vector.

Let
\begin{equation}
    z:=\begin{bmatrix}g_1^\top&\cdots&g_M^\top\end{bmatrix}^{\!\top},
    \qquad
    \dot z=-\Lambda z+B_f\widehat f,
\label{eq:stacked-linear}
\end{equation}
where
\begin{equation}
    \Lambda:=\operatorname{diag}(\lambda_1I_L,\ldots,\lambda_MI_L),
    \qquad
    C:=\begin{bmatrix}U_1&\cdots&U_M\end{bmatrix},
    \qquad
    B_f:=\begin{bmatrix}V_1^\top&\cdots&V_M^\top\end{bmatrix}^{\!\top}.
\end{equation}
The transformed trace reconstructed from this state is
\begin{equation}
    v^{\mathrm{lin}}(\xi)
    :=\Theta_{\widetilde k}(u(X_O(\xi)))
    =\boldsymbol{\Phi}_O(\xi)^\top Cz.
\end{equation}
Define the projected multiplication matrix and a regularized map from
observation coefficients to state variables by
\begin{align}
\label{eq:clock-matrix}
    M(z)
    &:={}
    \int_{U_O}
        \boldsymbol{\Phi}_O(\xi)m(v^{\mathrm{lin}}(\xi))
        \boldsymbol{\Phi}_O(\xi)^\top
    \,d\xi,
    \\
\label{eq:left-inverse}
    C_\varepsilon^\dagger
    &:={}
    (C^\top C+\varepsilon I)^{-1}C^\top.
\end{align}

\begin{proposition}[Trace-based reduced approximation]
\label{prop:surface-clock}
Replacing the volumetric multiplier in \eqref{eq:kirchhoff-clock} by its value
on the predicted trace gives
\begin{equation}
\label{eq:clocked-reduced}
    \dot z=C_\varepsilon^\dagger M(z)C
    \left(-\Lambda z+B_f\widehat f\right).
\end{equation}
If $\beta=0$, $\varepsilon=0$, and $C$ has full column rank, this reduces to
the original linear streaming dynamics.
\end{proposition}

\begin{proof}
The linear streaming model predicts the observation coefficients and their
time derivative as
\begin{align}
    w&=Cz,
    &
    \dot w&=C\dot z=C(-\Lambda z+B_f\widehat f).
\end{align}
Multiplying the reconstructed derivative by $m(v^{\mathrm{lin}}(\xi))$ and projecting
it back onto the observation basis gives the coefficient vector
\begin{equation}
    M(z)C(-\Lambda z+B_f\widehat f).
\end{equation}
We choose the state derivative by regularized least squares,
\begin{equation}
    \dot z
    =\underset{y}{\operatorname{argmin}}
    \left\lVert Cy-M(z)C(-\Lambda z+B_f\widehat f)\right\rVert^2
    +\varepsilon\lVert y\rVert^2.
\end{equation}
Its solution is exactly \eqref{eq:clocked-reduced}. If $\beta=0$, then
$m\equiv1$ and orthonormality of the observation basis gives $M(z)=I$.
For $\varepsilon=0$ and full column rank, $C^\dagger C=I$, so
\begin{equation}
    \dot z=-\Lambda z+B_f\widehat f.
\end{equation}
\end{proof}

Evolving \eqref{eq:clocked-reduced} and then applying the inverse Kirchhoff
transform gives the approximate map $\mathcal T$ for case~B.

\section{Datasets}
\label{app:datasets}
\label{app:TODO-datasets}
\label{app:exp}

This section defines the equations, media, sources, observation manifolds,
and numerical reference solutions used in the experiments. Model definitions
and experiment-specific protocols follow in later sections.

\subsection{Heat-Conduction Data}
\label{app:inst}

The temperature satisfies
\begin{equation}
\label{eq:heat-instance}
    \partial_t u=\nabla\!\cdot\!\bigl(\alpha(x)\nabla u\bigr)+f,
    \qquad u(\cdot,0)=0,
\end{equation}
on the cube $\Omega=[0,0.1]^3$ with homogeneous Neumann conditions. A
surface flux on the top face is represented as a volumetric source in the top
cell layer. The default observation is the temperature in that layer.

\begin{table}[ht]
\centering
\small
\begin{tabular}{ll}
\toprule
quantity & value \\
\midrule
spatial grid & $64^3$ cells \\
time interval & $[0,1]$ \\
stored frames & $N_t=64$ \\
time integrator & Crank--Nicolson \\
substeps per frame & $4$ for linear, $32$ for radiative, adaptive for configuration~B \\
linear solver & preconditioned conjugate gradients \\
relative tolerance & $10^{-8}$ \\
maximum iterations & $200$ \\
solver precision & float64 \\
\bottomrule
\end{tabular}
\caption{Reference discretization for the heat data.}
\label{tab:heat-data-settings}
\end{table}

The conservative finite-volume discretization uses harmonic averages of
$\alpha$ at cell faces. Crank--Nicolson time stepping uses 4 substeps per
stored frame for linear data and 32 for radiative data. Each linear system is
solved by conjugate gradients with a constant-coefficient DCT preconditioner.
The solver uses float64 and the learned models receive float32 fields. The
64 reported frames follow the zero initial condition, which is not stored as
an additional frame.

\begin{table}[ht]
\centering
\small
\begin{tabular}{p{0.17\linewidth}p{0.72\linewidth}}
\toprule
medium & construction \\
\midrule
two-phase & Threshold $g_{\mathrm{med}}$ at its median. Set
$\alpha_{\mathrm{solid}}=5\times10^{-4}$ and
$\alpha_{\mathrm{powder}}=5\times10^{-6}$. The top 2 cell layers are
set to $\alpha_{\mathrm{powder}}$. \\
graded & Set
$\varphi=\Phi(g_{\mathrm{med}}/\operatorname{std}(g_{\mathrm{med}}))$
and $\alpha=\alpha_{\mathrm{powder}}+
(\alpha_{\mathrm{solid}}-\alpha_{\mathrm{powder}})\varphi$.
For the interior-plane data, the endpoints are $10^{-4}$ and $10^{-2}$
and no powder cap is used. \\
\bottomrule
\end{tabular}
\caption{Media used in the heat experiments. The collocated experiments
use the two-phase medium. The interior-plane experiments use the graded
medium.}
\label{tab:heat-media}
\end{table}

The two-phase medium is obtained by smoothing a standard Gaussian field
$g_{\mathrm{med}}$ with a Gaussian filter of standard deviation $0.02$,
thresholding it at its median, and assigning
$\alpha_{\mathrm{solid}}=5\times10^{-4}$ and
$\alpha_{\mathrm{powder}}=5\times10^{-6}$. The top 2 cell layers are set to
the powder value. The interior-plane experiments instead use a graded medium
between $10^{-4}$ and $10^{-2}$ without the powder cap. Every learned model
receives the normalized coefficient
\begin{equation}
    \alpha_{\mathrm{norm}}
    =\frac{\alpha-\alpha_{\min}}{\alpha_{\max}-\alpha_{\min}}.
\end{equation}
For the two-phase data,
$(\alpha_{\min},\alpha_{\max})
=(\alpha_{\mathrm{powder}},\alpha_{\mathrm{solid}})$. The graded data use
their corresponding conductivity endpoints.

Every heat source is a Gaussian flux
\begin{equation}
\label{eq:heat-source-family}
    g(x,t)=A(t)\exp\!\left(-\frac{\lVert x-c(t)\rVert^2}{2\sigma^2}\right),
    \qquad \sigma=0.025,
\end{equation}
with peak amplitude $10$. Training sources are static pulses with centers
uniformly distributed inside the surface margins and durations between 6 and
64 frames. Their centers are uniform in $[0.15,0.85]^2L_\Omega$, and the
amplitude uses smooth $\sin^2$ ramps. Straight scans move at
constant speed along a random line segment of length $0.3L_\Omega$ to
$0.6L_\Omega$ and switch off at the end of the path. Their speed is sampled
between $0.5L_\Omega$ and $1.5L_\Omega$ per unit time. Serpentine scans
traverse 3 to 6 alternating horizontal lines in boxes whose widths range from
$0.25L_\Omega$ to $0.5L_\Omega$ and whose heights range from $0.15L_\Omega$
to $0.4L_\Omega$. Half of these paths include a small sinusoidal perturbation.
Only the static pulses are used for training.

Two further static-pulse test sets support appendix diagnostics. The rollout
set contains 512 frames on $[0,8]$ at the training frame spacing, while the
frame-spacing set contains 128 frames on $[0,1]$.

\begin{table}[ht]
\centering
\small
\begin{tabular}{lc}
\toprule
data set & samples \\
\midrule
heat training & $1024$ \\
static test & $128$ \\
straight scans & $128$ \\
serpentine scans & $128$ \\
radiative test & $128$ \\
Kirchhoff data & $1024/128/128$ per setting \\
interior-plane data & $1024/128/128/128$ per plane \\
resolution pairs & $64$ \\
\bottomrule
\end{tabular}
\caption{Heat data used for training and evaluation. Multiple numbers denote
training, static test, and the moving or shifted test sets.}
\label{tab:heat-splits}
\end{table}

The reported error is the mean of the per-sample relative errors
\begin{equation}
\label{eq:evaluation-metric}
    \mathcal E_{\mathrm{rel}}
    =\frac{1}{B}\sum_{b=1}^{B}
    \frac{\lVert \widehat v_b-v_b\rVert_2}
         {\lVert v_b\rVert_2+10^{-8}}.
\end{equation}
Tables report the mean and population standard deviation over training seeds
$0$, $1$, and $2$ unless stated otherwise.

\subsection{Distinct Observation Surfaces}
\label{app:interior-data}

The source and governing equation remain unchanged while the observation is
moved into the volume.

\begin{table}[ht]
\centering
\small
\begin{tabular}{lll}
\toprule
observation & grid location & relation to the source \\
\midrule
horizontal plane & $i_z=48$ & separated by about $L_\Omega/4$ \\
vertical plane & $i_x=32$ & mid-plane whose top edge lies in the source face \\
\bottomrule
\end{tabular}
\caption{Interior observation surfaces. Both cases are covered by
Proposition~\ref{prop:restricted-integrability}. The horizontal case also
satisfies the uniform bound in Remark~\ref{rem:separated-manifolds}.}
\label{tab:interior-planes}
\end{table}

The horizontal and vertical data use identical media and sources, which
isolates the effect of observation geometry. The selected graded-medium
contrast produces a visible response throughout both planes while retaining
meaningful dependence on the medium. Each plane has 1024 training samples and
128 samples for each static, straight, and serpentine test class.

\subsection{Advection--Diffusion Data}
\label{app:advection-data}

A ground-level release is observed on an orthogonal sensor curtain. The
concentration satisfies
\begin{equation}
\label{eq:advdiff-instance}
    \partial_t c+b\!\cdot\!\nabla c-\nabla\!\cdot\!(\kappa\nabla c)=S
\end{equation}
on $[0,9\,{\rm km}]^2\times[0,2\,{\rm km}]$ with homogeneous Neumann
conditions. The grid has $64^3$ cells and 64 stored frames over $2700$ s.
The diffusivity is
\begin{equation}
    \kappa(x)=\bar\kappa\exp(0.4g(x)),
    \qquad \bar\kappa\sim U(300,700)\ {\rm m^2s^{-1}},
\end{equation}
where $g$ is a band-limited cosine field with zero mean and unit variance,
clipped to $|g|\le2.5$. The velocity is the sum of a divergence-free circulation with
amplitude $W\sim U(8,12)$ m/s and a divergence-free perturbation whose RMS
speed is $0.4W$. Both components are tangent to the boundary.

The source manifold is the ground plane and the observation manifold is the
wall at $x=9$ km. The emission enters as a volumetric source with a fixed
Gaussian vertical profile centered at $50$ m with width $300$ m. The stored
observation is the cell layer adjacent to the wall. Static training fumes have random location, radius between
$400$ and $1200$ m, amplitude between $0.5$ and $1$, and exponential decay
time between $T/2$ and $2T$. Moving tests use either 1 crosswind track or 2 to
4 alternating crosswind legs. Moving emitters keep a constant amplitude and
radius and travel crosswind at $1.2$ to $2.4$ m/s. Static fumes are required to produce a curtain
RMS of at least $0.005$. Moving sources are not screened by their curtain
response.

The finite-volume reference uses 32 substeps per frame. Vertical diffusion
is implicit, while horizontal diffusion and first-order upwind advection are
explicit. Negative concentrations are clipped after each step.

\begin{table}[ht]
\centering
\small
\begin{tabular}{lc}
\toprule
data set & samples \\
\midrule
training & $1024$ \\
static test & $128$ \\
straight moving source & $128$ \\
complex moving source & $128$ \\
\bottomrule
\end{tabular}
\caption{Advection--diffusion data used for training and evaluation.}
\label{tab:advection-splits}
\end{table}

\subsection{Nonlinear and Resolution-Transfer Data}

The radiative set contains 128 static-pulse samples and stores both the
nonlinear reference and the corresponding linear solution. The boundary loss
is $-10^{-12}u^4$. The temperature-dependent conductivity experiments use
$\widetilde k(u)=1+\beta u$ in the 2 configurations developed in
Appendix~\ref{app:kirchhoff-construction}. The scale is set by the median peak
temperature $u_{\max}=1989.8$ of 64 linear training samples. The resulting
values are $1.0051\times10^{-4}$, $2.5129\times10^{-4}$, and
$5.0257\times10^{-4}$ for the mild, strong, and severe settings. Each setting
contains 1024 training samples, 128 static test samples, and 128 shifted test
samples. Shifted sources are scans, rasters, or sums of 2 static pulses.
Configuration~B uses $\lceil 4(1+\beta u^{(s)}_{\max})\rceil$ substeps per
frame, where $u^{(s)}_{\max}$ is the peak of a linear solve for the same sample.

The resolution study uses 64 paired samples. A medium drawn on the $64^3$
grid is prolonged as the same piecewise-constant field on the $128^3$ grid.
The same physical source is evaluated on each grid, and the numerical solver
is run independently at both resolutions. The fine target averages the top 2
layers so that it represents the same physical slab as the coarse target.

\section{GObO Architecture and Common Protocol}
\label{app:shared-architecture}
\label{app:TODO-shared-architecture}
\label{app:implementation}

\subsection{Cosine Basis and Kernel Representation}
\label{app:heat-model}
\label{app:impl-basis}

All source and observation charts are rectangular. The outer boundaries of
the physical domains have homogeneous Neumann conditions, so cosine functions
are a natural basis because their normal derivatives vanish there. Their
two-dimensional tensor products are indexed by 2 spatial frequencies and
therefore form a regular coefficient grid that can be processed efficiently.
We use the cell-centered DCT-II basis
\begin{equation}
\label{eq:impl-dct1d}
    \phi_a(i)=c_a\cos\!\left(\frac{\pi a(2i+1)}{2N}\right),
    \qquad c_0=N^{-1/2},
    \qquad c_a=(2/N)^{1/2}\ \text{for }a>0.
\end{equation}
The first 32 functions in each direction give a $32\times32$ coefficient
grid and $D_S=D_O=1024$. Source projection and observation synthesis are
\begin{equation}
    \widehat F^m=Df^mD^\top,
    \qquad v^m=D^\top\widehat V^mD.
\end{equation}
The constant midpoint quadrature weight is absorbed into the learned source
factors. Source and observation charts use the same function set. For
collocated data, 1 transform matrix therefore serves both projection and
synthesis. The basis is defined in continuous chart
coordinates, so the same coefficients can be evaluated on another grid. This
property is used in the resolution-transfer experiment.

The finite kernel used for heat is
\begin{equation}
\label{eq:heat-model-app}
    K_{\theta,\alpha}(\xi,\eta,\tau)
    =\sum_{r=1}^{R}e^{-\mu_r\tau}
    \boldsymbol\Phi_O(\xi)^\top
    U_r(\alpha)V_r(\alpha)^\top
    \boldsymbol\Phi_S(\eta).
\end{equation}
Here $\mu_r$ and $R$ correspond to $\lambda_k$ and $M$ in the main text. The
main model uses $R=16$ rates and factor search rank $L=16$. It shares
source and observation factors for collocated heat data and uses separate
factors for distinct manifolds. Each factor is normalized to unit Frobenius
norm and a learned scalar amplitude sets its magnitude. Thus the coefficient
matrix at rate $r$ is
\begin{equation}
\label{eq:impl-factor-matrix}
    A_r(\alpha)=\sum_{\ell=1}^{L}a_{r\ell}(\alpha)
    \widehat u_{r\ell}(\alpha)\widehat v_{r\ell}(\alpha)^\top.
\end{equation}
This normalization is used in every reported heat run.

The fixed rates are the lowest analytic decay rates of the homogeneous
Neumann heat Green kernel on the two-dimensional chart. They provide a
physical temporal dictionary for the dominant decay scales, while the learned
factors in Eq.~(\ref{eq:impl-factor-matrix}) represent the heterogeneous
medium. The rates follow
\begin{equation}
\label{eq:impl-ladder}
    \mu_{a,b}=\alpha_{\mathrm{solid}}\frac{\pi^2}{L_\Omega^2}(a^2+b^2).
\end{equation}
The 16 retained values of $a^2+b^2$ are
$0,1,1,2,4,4,5,5,8,9,9,10,10,13,13,16$. The zero rate represents the
conserved mean mode. Repeated analytic rates retain independent learned
spatial factors. In the learned-rate ablation, a pooled head predicts
$d_r(\alpha)$ and replaces each nonzero ladder value by
\begin{equation}
\label{eq:impl-learned-rates}
    \mu_r(\alpha)=\mu_r^{(0)}
    \exp\!\left(6\tanh d_r(\alpha)\right).
\end{equation}
The head is initialized at $d_r=0$ and is regularized toward the analytic
ladder. The conserved rate remains zero.

\subsection{Medium-Conditioning Network}
\label{app:impl-backbone}

The conditioning network receives only the normalized three-dimensional medium. For
each surface location, the corresponding depth column is augmented with a
fixed 4-channel depth encoding. 3 one-dimensional convolutions reduce its
64 depth values to 8 positions. Averaging over them gives 16 features. 2 planar
convolutions then produce a 32-channel field on the manifold grid. We append
16 fixed spatial coordinate channels and apply a channel map
$48\to256\to64$.

The resulting field passes through 4 blocks of width 64. Each block combines
a cosine spectral path, a pointwise linear path, and a small U-Net path. Its update is
\begin{equation}
\label{eq:impl-cosine-block}
    h^{p+1}=\operatorname{GELU}\!\left(
    \mathcal N_p(\mathcal S_p h^p)+W_ph^p+\mathcal U_ph^p\right).
\end{equation}
Here $\mathcal S_p$ is a cosine spectral operator that retains the lowest
$32\times32$ modes and mixes channels independently at each mode,
$\mathcal N_p$ is normalization, $W_p$ acts separately at each grid point,
and $\mathcal U_p$ captures local features at several scales. The spectral
path uses the same cosine transform as the kernel basis. After the 4 blocks, a
separate truncated spectral layer applies the forward cosine transform and
mixes channels separately at each mode. It stops directly on the
$32\times32$ coefficient grid without applying the inverse transform.

\begin{table}[ht]
\centering
\small
\begin{tabular}{ll}
\toprule
quantity & setting \\
\midrule
retained modes & $32\times32$ in each chart \\
source and observation dimensions & $D_S=D_O=1024$ \\
temporal rates & $R=16$ \\
factor rank per rate & $L=16$ \\
hidden width & $64$ \\
conditioning blocks & 4 cosine spectral blocks with U-Nets and 1 truncated layer \\
collocated coupling & shared source and observation factors \\
distinct-manifold coupling & separate source and observation factors \\
parameter count & $22.16$ million \\
\bottomrule
\end{tabular}
\caption{Settings for the heat GObO used in the main experiments.}
\label{tab:heat-model-settings}
\end{table}
A coefficient head produces the source and observation factors in
Eq.~(\ref{eq:impl-factor-matrix}). A pooled head produces their amplitudes.
For collocated heat data the same factor is used on both sides, while distinct
manifolds use separate factors. The source is not an input to this network. It
is projected and applied only after the medium-specific kernel has been
formed, so later sources in the same medium require no further neural-network
evaluation.

Figure~\ref{fig:impl-arch} places this architecture beside the adapted NGO
system network. NGO predicts a dense coefficient matrix for its 2 linear
temporal basis functions. GObO instead keeps the manifold grid through its conditioning blocks and
produces low-rank factors on the coefficient grid. It therefore never forms
the dense $D_O\times D_S$ matrix.

\begin{figure}[ht]
\centering
\includegraphics[width=\linewidth]{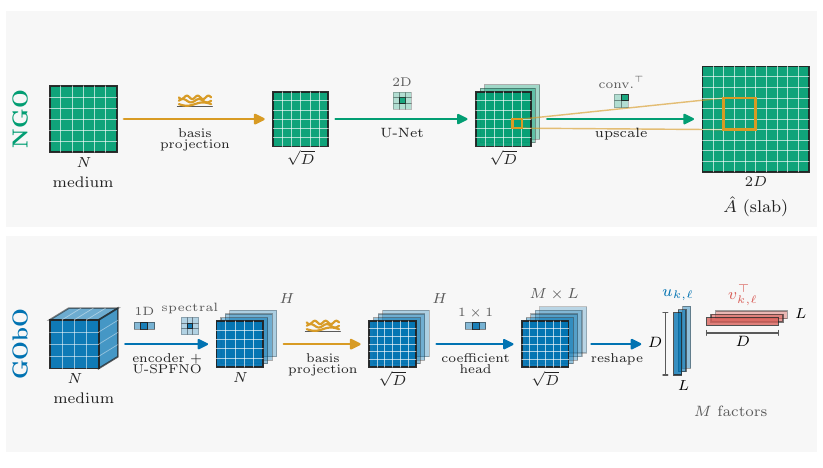}
\caption{\textbf{Conditioning architectures.} The adapted NGO system network
maps medium projections to a dense coefficient matrix. GObO maps the medium
volume to low-rank factors for each analytic decay rate. The GObO source is
applied only after these factors have been produced.}
\label{fig:impl-arch}
\end{figure}

\subsection{Temporal Propagation and Streaming State}

The reported heat and advection--diffusion models treat the source as
piecewise linear between frames. For $z=\mu\Delta$, define
\begin{equation}
\label{eq:impl-gamma}
    \gamma_0(z)=\frac{1-(1+z)e^{-z}}{z^2},
    \qquad
    \gamma_1(z)=\frac{z-1+e^{-z}}{z^2},
\end{equation}
with their continuous limits at $z=0$. The contribution of source frames
$f^{m-1}$ and $f^m$ is then integrated exactly against each exponential.
The equivalent dense form uses the lower-triangular propagator
\begin{equation}
    P_r[m,s]=\Delta\Big(\gamma_1(\mu_r\Delta)e^{-\mu_r\Delta(m-s)}\mathbb 1[s\le m]
    +\gamma_0(\mu_r\Delta)e^{-\mu_r\Delta(m-s-1)}\mathbb 1[s\le m-1]\Big)
\end{equation}
acting on the projected source coefficients $V_r(\alpha)^\top\widehat f^s$.
Training evaluates the resulting convolution densely over all frames. At
inference, the same computation is written as
\begin{equation}
\label{eq:impl-stream}
\begin{aligned}
    g_r^m={}&e^{-\mu_r\Delta}g_r^{m-1}
    +\Delta\gamma_0(\mu_r\Delta)V_r(\alpha)^\top\widehat f^{m-1}\\
    &+\Delta\gamma_1(\mu_r\Delta)V_r(\alpha)^\top\widehat f^m,\\
    v^m={}&\boldsymbol\Phi_O^\top\sum_r U_r(\alpha)g_r^m.
\end{aligned}
\end{equation}
with $g_r^1=\Delta\gamma_1(\mu_r\Delta)V_r(\alpha)^\top\widehat f^1$ because
the zero initial frame is not stored. The source coefficients are contracted
with the source factors before propagation, and only the observation factors
enter the synthesis. The state has $RL=256$ scalar coefficients for the heat
model. Changing the frame spacing only changes the analytical propagation
weights. Dense and streaming evaluations agree to floating-point precision.

\subsection{Advection--Diffusion Variant}

The advection--diffusion model receives the diffusivity and 3 velocity
components. The channels are scaled to order 1 using
$\kappa/\kappa_{\mathrm{ref}}$, where
$\kappa_{\mathrm{ref}}=500\,\mathrm{m^2s^{-1}}$, and by multiplying every
velocity component by $T/L_x=0.3$ s/m. One input path pools depth to the ground chart and a second pools
the along-wind direction to the sensor curtain. Each path has its own channel
map, 4 of the same conditioning blocks, and a final truncated spectral block.
Both retain a $32\times32$ cosine coefficient grid and produce separate source
and observation factors. Both factor grids are complex valued. Including the
conjugate rates gives the real curtain field
\begin{equation}
v^m=\operatorname{Re}\!\left(2\sum_{r=1}^{8}\sum_{\ell=1}^{8}
U_{r\ell}G^m_{r\ell}\right),
\qquad
G^m_{r\ell}=\sum_sP_r[m,s]
\langle V_{r\ell},\widehat f^{\,s}\rangle.
\end{equation}
The factor 2 accounts for the conjugate rate.

The kernel uses rank $L=8$ and the 8 lowest square modes $(a,b)$. It
parametrizes 8 complex rates together with their conjugates. In normalized
time, the fixed rates satisfy
\begin{equation}
\label{eq:impl-adv-ladder}
\operatorname{Re}\widetilde\mu_{a,b}
=\kappa_{\mathrm{ref}}T\frac{\pi^2(a^2+b^2)}{L_x^2},
\qquad
\operatorname{Im}\widetilde\mu_{a,b}=\widehat u(\alpha)\pi a,
\end{equation}
where $\kappa_{\mathrm{ref}}=500\,\mathrm{m^2s^{-1}}$ and $\widehat u$ is
the root-mean-square horizontal velocity scaled by $T/L_x$. The propagator
runs in normalized time with $\Delta=1/N_t$, so the rates are dimensionless.
The real part
represents diffusive decay and the imaginary part represents advective phase
transport. Adding the conjugate terms gives a real-valued output. A
learned-rate ablation perturbs both parts of each rate. The real part follows
Eq.~(\ref{eq:impl-learned-rates}), while the imaginary part is multiplied by
$1+\tanh e_r$. This variant does not
use the heat model's amplitude head or factor normalization. A scalar
calibration sets its initial output scale.

\subsection{Training and Evaluation}
\label{app:impl-training}
\label{app:impl-eval}

\begin{table}[ht]
\centering
\small
\begin{tabular}{ll}
\toprule
quantity & setting \\
\midrule
loss & relative $L_2$ in \eqref{eq:evaluation-metric} \\
optimizer & AdamW with default moment parameters \\
learning rate & $10^{-3}$ \\
weight decay & $10^{-4}$ \\
schedule & cosine decay to zero \\
training steps & $30{,}000$ \\
batch size & $64$ \\
precision & float32 \\
validation frequency & every $500$ steps \\
validation samples & $32$ held-out training samples \\
reported checkpoint & final checkpoint \\
training seeds & $0,1,2$ \\
gradient clipping & none \\
\bottomrule
\end{tabular}
\caption{Shared training settings. FNO-3D uses batch size 8. The
interior-plane models use the changes described below.}
\label{tab:training-settings}
\end{table}

Collocated heat data use all rotations and reflections of the square for data
augmentation. Horizontal-plane data use the same transformations. The
vertical-plane and advection--diffusion data use only the reflection that
preserves their geometry. Inputs are normalized as stated in the dataset
definitions. A source that enters a neural network is divided by its
training-set standard deviation. GObO instead receives the source in physical
units because it is applied only after conditioning. Targets remain in physical units. FNO-3D uses
batch size 8. Interior-plane models use learning rate $3\times10^{-4}$ and
gradient clipping at $1$. Adapted NGO uses Adam with learning rate $10^{-3}$
and no weight decay. All evaluations use Eq.~(\ref{eq:evaluation-metric}) on
the same splits.

\section{Experimental Details}
\label{app:details}
\label{app:TODO-experiment-details}

This section gives only the choices that differ across experiments. Dataset,
architecture, training, and metric definitions are given above. Numerical
outcomes are collected in Appendix~\ref{app:results}.

\subsection{Experimental setting}

\subsubsection{Baselines}
\label{app:heat-training}
\label{app:impl-baselines}

FNO-2D receives a two-dimensional medium encoding together with all 64 source
frames and predicts all 64 observation frames directly. It has width 64, 32
Fourier modes per direction, and 4 spectral layers. A 5-layer version tests
additional depth. The backbone-matched spectral model uses GObO's medium encoder
and 4 or 5 of the same cosine spectral blocks but replaces the kernel by direct regression of the
trajectory. FNO-3D receives the conductivity volume and the surface source
$f(x,y,t)$ arranged as a $64^3$ array with time as the third axis,
uses width 64 with 16 modes in each direction, and predicts the observed
trajectory. The advection--diffusion FNO stacks the diffusivity and velocity
fields as depth channels and maps the ground-source array directly to the
sensor curtain. The nonlinear FNO baselines are retrained independently for
each material law and contrast.

All baselines use the same data splits, loss, and evaluation metric as GObO.
Model sizes are reported in Table~\ref{tab:params}.

\subsection{GObO learns a kernel that generalizes across sources}

\subsubsection{Heat and Distinct-Surface Protocols}
\label{app:interior-details}

Every heat model is trained on static pulses and evaluated without retraining
on held-out static pulses, straight scans, and serpentine scans. The headline
experiment observes the source face. The distinct-surface experiments use
separate source and observation factors. The horizontal plane retains all
square rotations and reflections. The vertical plane uses reflection in the shared
horizontal coordinate because rotations do not preserve its geometry.

\subsubsection{Advection--Diffusion Protocol}
\label{app:advection-details}
\label{app:advdiff}

The fixed-rate GObO is the primary model. Learned rates provide an ablation,
and FNO-2D is the direct baseline. Training uses stationary fumes. Evaluation
uses held-out fumes, straight crosswind motion, and alternating crosswind
motion. Reflection in the crosswind coordinate is applied to the medium
volume, source, and observation while reversing the corresponding velocity
component.

\subsection{The augmented kernel state captures non-closed submanifold dynamics}
\label{app:details-kernel-state}

\subsubsection{Adapted Transient NGO}
\label{app:TODO-ngo-details}
\label{app:ngo-details}
\label{app:impl-ngo}

The baseline follows the published time-dependent NGO, with only the minimal adaptations that
our setting requires. Its conditioning network is shown next to GObO's in
Figure~\ref{fig:impl-arch}.

NGO leaves the choice of spatial basis open. Its authors use splines,
whereas we use the cosine basis of this paper so that the baseline is
comparable with the other experiments. This is the only design choice we
make differently. The remaining changes are not differences but adaptations
that our setting necessitates. The stepped state is the observed trace because
the full field that NGO advances is not available. The volumetric medium is
projected onto the spatial basis at each depth, and the resulting maps are
concatenated as channels for the system network.

NGO stabilizes its rollout by capping the norm of the step matrix at a
threshold $S$. We report their published value $S=0.8$ together with $S=1$,
because the norm of the true step map in our setting is close to 1.

\subsubsection{One-Step GObO Variant}
\label{app:impl-history}

This diagnostic keeps GObO's medium encoder and spatial factors but replaces
the lag-dependent kernel by a purely spatial recurrence. With $k$ previous
frames it predicts
\begin{equation}
\label{eq:impl-history}
    \widehat v^m
    =\sum_{j=1}^{k}S_j(\alpha)\widehat v^{m-j}
    +B_0(\alpha)q_L^m+B_1(\alpha)q_R^m.
\end{equation}
We use $k=1$, $4$, and $16$. Each spatial map is represented by a rank-256
factorization with separate input and output factors. The concatenated state
block $[S_1\cdots S_k]$ is normalized by the same norm cap as NGO with
$S=1$. During training it receives the true previous frames and predicts the
next frame. At evaluation it starts from the zero trace and is then advanced
using its own predictions.

\subsubsection{Best Linear One-Step Map}
\label{app:oracle-details}

This fit removes neural-network approximation from the test of whether the
observation contains enough state for a one-step update. For each of 3 fixed
test media, each of the 64 modes in an $8\times8$ cosine block is excited once.
Reported values are means over these 3 media.
The reference responses define the restricted discrete Green operator. For
history depth $k$, the best linear map is
\begin{equation}
\label{eq:impl-oracle-fit}
    W_k=\arg\min_W\sum_{p,m}
    \left\|W^\top
    \begin{bmatrix}
    \widehat h_p^{m-1} & \cdots & \widehat h_p^{m-k} &
    \widehat f_p^{m-1} & \widehat f_p^m
    \end{bmatrix}^{\!\top}
    -\widehat h_p^m\right\|_2^2.
\end{equation}
The fit is evaluated on 24 additional band-limited pulses by 1-step residual
and rollout error. The same procedure is repeated for a two-dimensional heat
equation on the top layer of the same medium, observed over its whole domain.
This control has access to the complete state and separates missing state in the
partial observation from basis truncation and regression error.

\subsection{The kernel parametrization captures the spatiotemporal structure}
\label{app:details-kernel-structure}

\subsubsection{Naive Tensor Decomposition}
\label{app:impl-spacetime}

The naive alternative replaces the lag law by independent cosine bases for
source time and observation time
\begin{equation}
\label{eq:impl-spacetime-kernel}
    K(\xi,\eta,t,s)
    =\sum_{c,c'=0}^{C-1}\tau_c(t)\tau_{c'}(s)
    \boldsymbol\Phi_O(\xi)^\top A_{cc'}(\alpha)
    \boldsymbol\Phi_S(\eta).
\end{equation}
It uses the same medium encoder and rank-16 spatial factors as GObO. The main
comparison uses $C=16$. A size-matched control uses $C=4$. This representation
does not enforce zero response before a source is applied. It also has no
finite streaming state and is defined only on the fixed time interval used to
build its temporal bases.

\subsubsection{Architecture and Parametrization Ablations}
\label{app:impl-medium}

The retained ablations compare fixed and learned rates, shared and separate
factors, and search ranks. A further control replaces GObO's learned
volume-to-manifold encoder by the lowest $16^3$ cosine projections of the
medium followed by a parameter-matched U-Net. The kernel, rate ladder, factors,
and recurrence remain unchanged. Reduced $16\times16$ and $8\times8$ basis
controls test whether differences can be explained by basis size.

\subsubsection{Controlled Curved-Surface Recovery}
\label{app:toy-details}
\label{app:impl-toy}

The controlled problem uses the homogeneous Neumann heat equation on the unit
cube with diffusivity 1. For chart coordinates $(p,q)$, set
$a=p-1/2$ and $b=q-1/2$. The source and observation surfaces are
\begin{equation}
\label{eq:toy-manifolds}
\begin{aligned}
X_S(p,q)&=\left(\tfrac12+0.7a+0.08\sin\!(2\pi q),
\tfrac12+0.8b,0.26-0.15a+0.15b+0.3ab\right),\\
X_O(p,q)&=\left(\tfrac12+0.7a+0.05\cos\!(2\pi q),
\tfrac12+0.45b+0.1\sin\!(2\pi p),0.78-0.25a+0.2b-0.2a^2\right).
\end{aligned}
\end{equation}
The restricted reference kernel is computed from 343 ambient cosine modes and
includes the source-surface area factor.

\begin{table}[ht]
\centering
\small
\begin{tabular}{ll}
\toprule
quantity & value \\
\midrule
diffusivity & $\kappa=1$ \\
time step and frames & $\Delta t=0.0125$ and $N_t=64$ \\
first supervised frame & $4$, corresponding to $\tau_{\min}=0.05$ \\
chart grid & $64\times64$ cell-centered points \\
retained chart modes & $16\times16$ cosine modes \\
ambient spectral cutoff & $n_j\leq6$, giving $343$ modes \\
displayed lags & $0.05$, $0.075$, and $0.125$ \\
held-out sources & $64$ Gaussian chart bumps \\
\bottomrule
\end{tabular}
\caption{Settings for the controlled curved-manifold example.}
\label{tab:toy-settings}
\end{table}

The recovery figure uses 256 coefficient-space impulse probes, 16 distinct
rates, and a truncated singular-value decomposition on frames with lag at
least $0.05$. Each fitted
coefficient matrix is truncated to rank 8. The rank sweep is a separate fit
with 63 rate slots, including repeated ambient eigenvalues. It minimizes the
relative response error plus a factor norm penalty and is compared with a
convex low-rank reference based on the nuclear norm. The table reports only
the data-fit term for both methods. Recovered rank uses threshold
$10^{-4}\sigma_1$.

\subsection{GObO enables rapid source evaluation and nonlinear corrections}

\subsubsection{Timing Protocol}
\label{app:timing-details}
\label{app:impl-timing}

\begin{table}[ht]
\centering
\small
\begin{tabular}{ll}
\toprule
quantity & setting \\
\midrule
CPU & AMD EPYC 7742 with 32 allocated cores \\
GPU & NVIDIA A100-SXM4 with 80 GB \\
software & Python 3.12.3 and PyTorch 2.13 \\
numerical libraries & NumPy 2.1.0 and SciPy 1.17.1 with OpenBLAS \\
solver repetitions & 1 warmup and 3 measurements \\
learned-model repetitions & 3 warmups and 8 measurements \\
reported statistic & median \\
\bottomrule
\end{tabular}
\caption{Environment used for the timing measurements.}
\label{tab:timing-environment}
\end{table}

GObO conditioning contains the only neural-network evaluation for a medium.
Later all-frames and streaming queries use only the explicit kernel. FNO-2D
requires a network evaluation for every source. Batch size 1 measures latency.
Batch size 64 measures throughput and shares 1 conditioned GObO across the
batch. GPU measurements are synchronized. The matched-accuracy solver is the
cheapest tested setting whose error against the production solver is below
GObO's in-distribution error of $0.011$.
The latency measurements use 1 test medium and 1 source. The finite-volume
rows use the production configuration with 1 thread. Each learned CPU row uses
the fastest result obtained with 1, 8, or 32 threads.

\subsubsection{Resolution-Transfer Protocol}
\label{app:res}
\label{app:impl-res}

\begin{table}[ht]
\centering
\small
\begin{tabular}{lp{0.72\linewidth}}
\toprule
mode & computation \\
\midrule
M0 & Compare the 2 numerical reference solutions. \\
M1 & Run the complete learned model on the fine grid. \\
M2 & Run the learned model on the coarse grid and evaluate its final
continuous spectral representation on the fine grid. \\
M3 & Run only the GObO medium encoder on the coarse grid, then project
the fine-grid source and evaluate the operator on the fine grid. \\
\bottomrule
\end{tabular}
\caption{Resolution-transfer modes used in Section~\ref{sec:exp-res}.}
\label{tab:resolution-modes}
\end{table}

M1 computes the positional encodings on the 64-cell training grid and
resamples them to the fine grid.
The positional-encoding ablation instead uses raw fine-grid indices. M2
resynthesizes GObO coefficients with the fine-grid cosine basis and uses
zero-padded Fourier synthesis for FNO. M3 rebuilds the source and observation
basis matrices at the fine resolution while reusing factors conditioned on the
coarse medium.

\subsubsection{Radiative Feedback Protocol}
\label{app:radiation}

The radiative law uses coefficient $10^{-12}$ and ambient temperature 0. We
evaluate an explicit update using the previous-frame radiative flux and a
linearized implicit update in coefficient space. Both use the unchanged GObO
kernel and the exact propagation weights in the state update. The reference
uses 32 finite-volume substeps per stored frame.
In the implicit correction matrix, $\gamma_1(\nu_k)$ is replaced by its
small-$\nu$ limit $1/2$. The state update itself retains the exact weights.

\subsubsection{Temperature-Dependent Conduction Protocol}
\label{app:kirchhoff}
\label{app:impl-nonlinear}

\begin{table}[ht]
\centering
\small
\begin{tabular}{lcc}
\toprule
setting & $\beta u_{\max}$ & conductivity ratio \\
\midrule
mild & $0.2$ & $1.2$ \\
strong & $0.5$ & $1.5$ \\
severe & $1.0$ & $2.0$ \\
\bottomrule
\end{tabular}
\caption{Conductivity contrasts used in the nonlinear experiments. The
range covers reported process windows for 316L, IN718, and Ti-6Al-4V
\citep{kim1975thermophysical,specialmetals2007in718,mukherjee2016printability}.}
\label{tab:kirchhoff-settings}
\end{table}

Configuration A applies the exact Kirchhoff transform
$\Theta=u+\beta u^2/2$, evaluates the unchanged linear checkpoint, and applies
the inverse transform pointwise. Configuration B advances the reduced kernel
state with the trace-based state equation from
Appendix~\ref{app:kirchhoff-construction}. The regularized left inverse uses
$\varepsilon=10^{-8}\lVert S^\top S\rVert_2$. The learned state update uses 4
substeps per frame. The local clock is frozen within a substep and the
resulting affine system is advanced by its matrix exponential.

The fixed-point comparison uses damping $0.5$, residual scale 1, tolerance
$10^{-3}$, and at most 20 iterations. A sample is marked as diverged after 2
successive residual increases. The reported iterate is the one with the
smallest residual and does not use the truth for selection.
The time-derivative rows obtain $\partial_t\theta$ by finite differences over
frames, while the generator row uses the exact derivative of the modal
representation. Statistics pool the 128 samples from all 3 seeds.

\section{Additional Experimental Results}
\label{app:results}

The results follow the same order as the main experiments. Protocols are not
repeated here.

\subsection{Experimental setting}

\subsubsection{Model-Size Comparison}
\label{app:result-model-size}

\begin{table}[!ht]
\centering
\footnotesize
\setlength{\tabcolsep}{3pt}
\begin{tabular}{lrl}
\toprule
model & parameters & scope \\
\midrule
GObO & $22.16$M & heat, fixed rates \\
FNO-2D & $16.91$M & heat \\
matched spectral model & $17.89$M & heat \\
FNO-3D & $67.18$M & heat \\
advection GObO, fixed rates & $44.01$M & advection--diffusion \\
advection GObO, learned rates & $44.14$M & advection--diffusion \\
advection FNO-2D & $16.96$M & advection--diffusion \\
FNO-2D, 5 layers & $21.11$M & heat \\
matched spectral model, 5 blocks & $22.33$M & heat \\
\bottomrule
\end{tabular}
\caption{\textbf{Model sizes.} Fixed-rate and learned-rate advection
models are listed separately.}
\label{tab:params}
\end{table}

Table~\ref{tab:params} follows the framework convention in which 1 complex
tensor entry is counted as 1 parameter. FNO spectral weights therefore contain
2 real numbers per counted complex parameter.

\subsection{GObO learns a kernel that generalizes across sources}
\label{app:TODO-generalization-results}
\label{app:heat-results}
\label{app:source-results}
\label{app:noncol}
\label{sec:exp-source}

\subsubsection{Heat Source Transfer}
\label{app:result-heat-transfer}

The full comparison in Table~\ref{tab:exp1_full} confirms the trade-off shown
in the main text. Direct models are strongest on static sources but degrade
sharply under motion. GObO changes little across the 3 source classes.

\begin{table}[ht]
\centering
\scriptsize
\setlength{\tabcolsep}{2.6pt}
\begin{tabular}{lcccccc}
\toprule
& \multicolumn{3}{c}{relative $L_2$} & \multicolumn{3}{c}{absolute $L_2$} \\
\cmidrule(lr){2-4}\cmidrule(lr){5-7}
method & static & straight & serpentine & static & straight & serpentine \\
\midrule
GObO & $0.011 \pm 0.001$ & $0.012 \pm 0.000$ & $0.005 \pm 0.000$ & $2838 \pm 160$ & $2253 \pm 85$ & $2447 \pm 12$ \\
FNO-2D & $0.001 \pm 0.000$ & $0.101 \pm 0.004$ & $0.137 \pm 0.007$ & $282 \pm 17$ & $23330 \pm 810$ & $63590 \pm 3400$ \\
FNO-3D & $0.004 \pm 0.000$ & $0.185 \pm 0.003$ & $0.426 \pm 0.008$ & $970 \pm 6$ & $46670 \pm 360$ & $199500 \pm 3900$ \\
matched spectral model & $0.003 \pm 0.001$ & $0.204 \pm 0.002$ & $0.217 \pm 0.004$ & $370 \pm 50$ & $48130 \pm 700$ & $100800 \pm 1900$ \\
\bottomrule
\end{tabular}
\caption{\textbf{Complete source-transfer comparison.} Values are
means and population standard deviations over 3 seeds. The moving
sets are zero-shot.}
\label{tab:exp1_full}
\end{table}

The aggregate errors show that the difference is not limited to 1 path type
or 1 normalization. GObO remains stable in relative and absolute error on both
moving tests, while every direct model loses its static-source advantage.
Figure~\ref{fig:exp1-kymo} shows how that failure develops during a scan.

\begin{figure}[ht]
\centering
\begin{subfigure}[b]{0.58\textwidth}
  \centering\includegraphics[width=\linewidth]{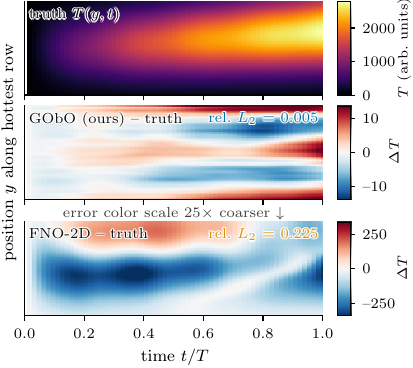}
  \caption{response and signed error}
\end{subfigure}\hfill
\begin{subfigure}[b]{0.36\textwidth}
  \centering\includegraphics[width=\linewidth]{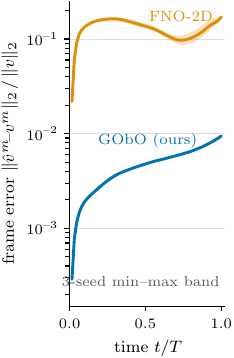}
  \caption{per-frame relative error}
\end{subfigure}
\caption{\textbf{Per-frame behavior on serpentine scans.} GObO tracks
the moving peak throughout the trajectory, whereas FNO-2D develops a
persistent displacement error. Curves show the 3-seed min--max
band.}
\label{fig:exp1-kymo}
\end{figure}

The displacement visible in the signed error explains why the relative error
grows even though the direct model remains accurate on static sources.

\subsubsection{Baseline Capacity and GObO Design}
\label{app:result-baseline-depth}
\label{app:result-gop-ablation}

Table~\ref{tab:architecture-ablation} shows that increasing the depth of either
direct baseline does not repair moving-source transfer. Within GObO, fixed
rates and shared factors give the most reliable result. Search rank 1 is
insufficient.

\begin{table}[ht]
\centering
\footnotesize
\begin{tabular}{lccc}
\toprule
baseline architecture & static & straight & serpentine \\
\midrule
FNO-2D, 4 layers & $0.001 \pm 0.000$ & $0.101 \pm 0.004$ & $0.137 \pm 0.007$ \\
FNO-2D, 5 layers & $0.001 \pm 0.000$ & $0.103 \pm 0.002$ & $0.149 \pm 0.005$ \\
matched spectral model, 4 layers & $0.003 \pm 0.001$ & $0.204 \pm 0.002$ & $0.217 \pm 0.004$ \\
matched spectral model, 5 layers & $0.004 \pm 0.001$ & $0.226 \pm 0.002$ & $0.234 \pm 0.013$ \\
\bottomrule
\end{tabular}
\caption{\textbf{Baseline-depth ablation} (relative $L_2$). Increasing
depth does not resolve the moving-source failure of either direct-output
architecture.}
\label{tab:architecture-ablation}
\end{table}

The unchanged or worse transfer error after adding depth rules out insufficient
baseline capacity as the primary explanation. Table~\ref{tab:exp1_ablation}
next varies the structure inside GObO to determine which parts of its
factorization matter.

\begin{table}[ht]
\centering
\footnotesize
\begin{tabular}{lccc}
\toprule
GObO factorization & static & straight & serpentine \\
\midrule
fixed rates, shared factors, $L=16$ & $\mathbf{0.011 \pm 0.001}$ & $\mathbf{0.012 \pm 0.000}$ & $\mathbf{0.005 \pm 0.000}$ \\
learned rates, shared factors, $L=16$ & $0.011 \pm 0.002$ & $0.012 \pm 0.002$ & $0.006 \pm 0.001$ \\
fixed rates, separate factors, $L=16$ & $0.019 \pm 0.001$ & $0.016 \pm 0.001$ & $0.008 \pm 0.000$ \\
learned rates, separate factors, $L=16$ & $0.013 \pm 0.002$ & $0.013 \pm 0.002$ & $0.006 \pm 0.001$ \\
fixed rates, shared factors, $L=1$ & $0.130 \pm 0.007$ & $0.123 \pm 0.005$ & $0.074 \pm 0.008$ \\
\bottomrule
\end{tabular}
\caption{\textbf{Rate and factor-structure ablation.} Here $L$ is the
factor search rank used for each fixed or learned rate, matching the
terminology of Tables~\ref{tab:rank} and \ref{tab:rank-full}.}
\label{tab:exp1_ablation}
\end{table}

These controls support the use of the analytic fixed rates, shared factors,
and $L=16$ in the headline heat model.

\subsubsection{Distinct Observation Surfaces}
\label{app:result-distinct-observations}

The same source-transfer pattern holds when source and observation surfaces
are distinct.

\begin{table}[ht]
\centering
\footnotesize
\begin{tabular}{llccc}
\toprule
observation & method & static & straight & serpentine \\
\midrule
horizontal plane & GObO & $0.042$ & $\mathbf{0.044}$ & $\mathbf{0.027}$ \\
horizontal plane & FNO-2D & $\mathbf{0.027}$ & $0.194$ & $0.108$ \\
vertical mid-plane & GObO & $0.064$ & $\mathbf{0.062}$ & $\mathbf{0.054}$ \\
vertical mid-plane & FNO-2D & $0.066$ & $0.240$ & $0.148$ \\
\bottomrule
\end{tabular}
\caption{\textbf{Source transfer with non-collocated observations.}
The horizontal plane is separated from the source by $L_\Omega/4$, the
vertical plane is orthogonal to the source face.}
\label{tab:noncol}
\end{table}

Table~\ref{tab:noncol} shows that separating the manifolds does not remove
GObO's transfer advantage. Figure~\ref{fig:noncol} resolves the vertical-plane
result in time and shows that the aggregate gap is sustained throughout the
moving scan.

\begin{figure}[ht]
\centering
\begin{subfigure}[b]{0.60\textwidth}
  \centering\includegraphics[width=\linewidth]{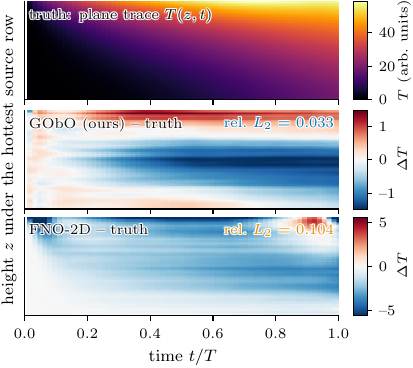}
  \caption{vertical trace and signed error}
\end{subfigure}\hfill
\begin{subfigure}[b]{0.36\textwidth}
  \centering\includegraphics[width=\linewidth]{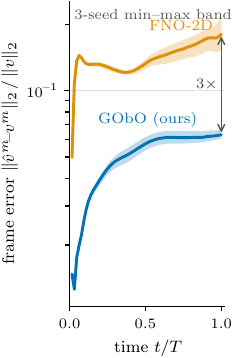}
  \caption{per-frame relative error}
\end{subfigure}
\caption{\textbf{A moving source observed on the vertical mid-plane.}
For the displayed sample, GObO has relative error $0.033$ and FNO-2D
has $0.104$, the aggregate per-frame curves show that the gap persists
over the test set.}
\label{fig:noncol}
\end{figure}

The result shows that the learned kernel transfers across sources even when
the response is observed away from and orthogonal to the source surface.

\subsubsection{Advection--Diffusion Source Transfer}
\label{app:result-advection-transfer}
\label{app:advection-results}
\label{app:ident}

Table~\ref{tab:exp4_full} shows that fixed physical rates keep GObO close to
its static error under both source shifts. Learned rates degrade all 3 source
classes. FNO-2D is competitive on static fumes but does not transfer to motion.

\begin{table}[ht]
\centering
\footnotesize
\setlength{\tabcolsep}{2.5pt}
\begin{tabular}{lcccccc}
\toprule
& \multicolumn{3}{c}{relative $L_2$} & \multicolumn{3}{c}{absolute $L_2$} \\
\cmidrule(lr){2-4}\cmidrule(lr){5-7}
method & static & straight & complex & static & straight & complex \\
\midrule
GObO, fixed rates & $0.114 \pm 0.018$ & $0.101 \pm 0.020$ & $0.139 \pm 0.035$ & $1.88 \pm 0.19$ & $2.04 \pm 0.31$ & $2.00 \pm 0.24$ \\
GObO, learned rates & $0.300 \pm 0.021$ & $0.313 \pm 0.016$ & $0.352 \pm 0.028$ & $5.19 \pm 0.19$ & $6.68 \pm 0.23$ & $6.13 \pm 0.26$ \\
FNO-2D & $0.104 \pm 0.009$ & $0.430 \pm 0.015$ & $0.364 \pm 0.007$ & $1.77 \pm 0.15$ & $9.33 \pm 0.36$ & $6.09 \pm 0.17$ \\
\bottomrule
\end{tabular}
\caption{\textbf{Complete source-generalization results on the sensor
curtain.} Values are means and population standard deviations over
3 seeds.}
\label{tab:exp4_full}
\end{table}

The smaller gap on stationary fumes shows that both models can learn the
in-distribution map. Figure~\ref{fig:exp4} illustrates the different behavior
after the source leaves the stationary training distribution.

\begin{figure}[ht]
\centering
\includegraphics[width=0.95\textwidth]{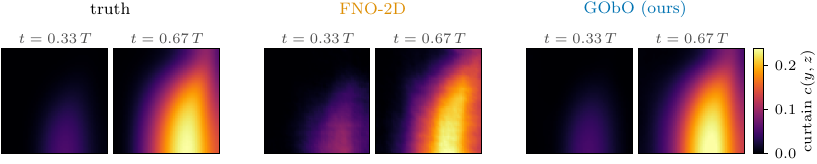}
\caption{\textbf{A moving ground-level release observed on a vertical
sensor curtain.} The qualitative comparison complements the aggregate
advection--diffusion errors in Table~\ref{tab:exp4_full}.}
\label{fig:exp4}
\end{figure}

The curtain response confirms the quantitative result without suggesting the
catastrophic qualitative failure observed for heat.

\subsection{The augmented kernel state captures non-closed submanifold dynamics}
\label{app:stream-results}
\label{sec:exp-streaming}

\subsubsection{Adapted Transient NGO}
\label{app:result-ngo}

The $S=1$ model reaches 1-step error $0.012$ and rollout errors $0.039$,
$0.043$, and $0.027$ on static, straight, and serpentine sources. The published value $S=0.8$
clips every sample and gives static rollout error $0.721$.

The $S=1$ result is therefore the informative NGO comparison used in the main
table, while the published threshold is retained to document the adaptation.

\subsubsection{One-Step Prediction from Partial Observations}
\label{app:result-one-step}

On the fully observed two-dimensional control, the best 1-frame map has
residual $4\times10^{-15}$ and additional history does not improve it. On the
three-dimensional trace, the residual is $3.2\times10^{-3}$ at $k=1$ and
$1.1\times10^{-3}$ at $k=16$. The corresponding best linear rollout errors are
$0.043$ and $0.016$. This contrast shows that the observed trace is not a
closed state.

Table~\ref{tab:impl-onestep} shows that the learned one-step GObO variants
follow the same trend inside the training window. Longer history lowers static
rollout error from $0.103$ to $0.032$. Beyond that window, the longer
recurrences become unstable, while GObO remains bounded through the reported
diagnostic.

\begin{table}[!ht]
\centering
\small
\resizebox{\linewidth}{!}{\begin{tabular}{lrcccccc}
\toprule
& params & one-step (static) & rollout, static & straight & serpentine & rescaled fraction & ms/frame \\
\midrule
GObO & 22.16\,M & -- & $0.011\pm0.001$ & $0.012\pm0.000$ & $0.005\pm0.000$ & -- & 0.13 \\
\multicolumn{8}{@{}l}{\emph{A. NGO as published}} \\
NGO, $S=0.8$ & 0.02\,M & $0.163\pm0.000$ & $0.721\pm0.000$ & $0.760\pm0.000$ & $0.552\pm0.001$ & $1.00$ & 0.09 \\
NGO, $S=1$ & 0.02\,M & $0.012\pm0.000$ & $0.039\pm0.001$ & $0.043\pm0.000$ & $0.027\pm0.001$ & $1.00$ & 0.09 \\
\multicolumn{8}{@{}l}{\emph{B. One-step GObO variants}} \\
one-step, $k=1$ & 22.28\,M & $0.021\pm0.000$ & $0.103\pm0.003$ & $0.103\pm0.009$ & $0.058\pm0.003$ & $1.00$ & 0.08 \\
one-step, $k=4$ & 22.43\,M & $0.014\pm0.007$ & $0.091\pm0.012$ & $0.108\pm0.016$ & $0.052\pm0.007$ & $1.00$ & 0.09 \\
one-step, $k=16$ & 23.05\,M & $0.015\pm0.004$ & $0.032\pm0.005$ & $0.033\pm0.006$ & $0.017\pm0.003$ & $1.00$ & 0.09 \\
\bottomrule
\end{tabular}}
\caption{One-step schemes on the trace. Relative error, mean $\pm$ population
standard deviation over 3 seeds, 128 samples per set. ``One-step'' is the error of a single
step from the true previous frames, the quantity these rows are trained on, the
3 rollout columns feed the model its own previous frames from rest. ``Rescaled
fraction'' indicates whether the norm cap is active on the first static test
sample, averaged over the 3 seeds. Frame-spacing transfer requires
retraining for every one-step row. Parameter counts are of the built models.}
\label{tab:impl-onestep}
\end{table}

Inside the training interval, longer histories partly compensate for the
missing state and the $k=16$ variant is the strongest one-step model. This
does not establish a stable evolution law. Table~\ref{tab:impl-onestep-horizon}
rolls every model beyond the interval used for training and exposes the
resulting accumulation of error.

\begin{table}[!ht]
\centering
\small
\begin{tabular}{lcccc}
\toprule
cumulative error on & $[0,1T]$ & $[0,2T]$ & $[0,3T]$ & $[0,4T]$ \\
\midrule
GObO & $0.010\pm0.000$ & $0.047\pm0.002$ & $0.105\pm0.004$ & $0.158\pm0.007$ \\
NGO, $S=0.8$ & $0.702\pm0.000$ & $0.812\pm0.000$ & $0.832\pm0.000$ & $0.841\pm0.000$ \\
NGO, $S=1$ & $0.036\pm0.001$ & $0.142\pm0.001$ & $0.237\pm0.001$ & $0.297\pm0.001$ \\
one-step, $k=1$ & $0.099\pm0.001$ & $0.227\pm0.005$ & $0.326\pm0.008$ & $0.396\pm0.012$ \\
one-step, $k=4$ & $0.090\pm0.012$ & diverged & diverged & diverged \\
one-step, $k=16$ & $0.030\pm0.006$ & $0.136\pm0.022$ & $0.58\pm0.20$ & diverged \\
\bottomrule
\end{tabular}
\caption{One-step schemes, horizon. Cumulative relative error on the first
$h\cdot64$ frames of the decaying-source test set (128 samples, source off
before $T$), mean $\pm$ population standard deviation over 3 seeds. GObO by the streaming
recurrence, the one-step rows by rollout. ``Diverged'' is a rollout that has
left the scale of the data by that horizon (a cumulative error above $1$ in at
least 1 seed). The $S=0.8$ row has decayed to a fraction of the truth
instead.}
\label{tab:impl-onestep-horizon}
\end{table}

The horizon comparison shows that lower error inside the training window does
not imply stable repeated prediction from partial observations.

\subsection{The kernel parametrization captures the spatiotemporal structure}
\label{app:TODO-kernel-parametrization-results}
\label{app:kernel-structure-results}

\subsubsection{Impulse-Response Structure}
\label{app:result-impulse-response}

Figure~\ref{fig:impl-lag} compares both learned responses with the reference
impulse response for the same medium and source location. GObO captures its
lag-dependent decay and onset in this example. The generic space--time basis
does not reproduce this structure and responds before the impulse is applied.

\begin{figure}[ht]
\centering
\begin{subfigure}[b]{0.43\textwidth}
  \centering\includegraphics[width=\linewidth]{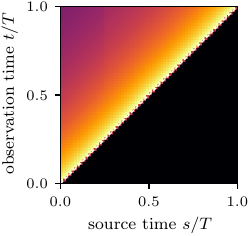}
  \caption{GObO kernel parametrization}
\end{subfigure}\hfill
\begin{subfigure}[b]{0.50\textwidth}
  \centering\includegraphics[width=\linewidth]{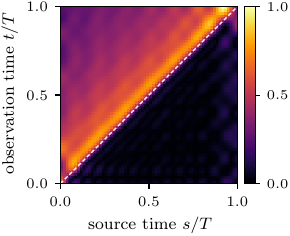}
  \caption{separable space--time parametrization}
\end{subfigure}
\caption{\textbf{Impulse-response comparison.} The panels show the
response produced by the proposed kernel parametrization and by the naive
separable space--time alternative.}
\label{fig:impl-lag}
\end{figure}

This response-level comparison links the lower GObO error to the temporal
structure imposed by its lag-dependent kernel.

\subsubsection{Naive Tensor-Decomposition Accuracy}
\label{app:result-tensor}

Table~\ref{tab:impl-spacetime} shows that removing the lag law increases static
error from $0.011$ to $0.035$ and moving error from $0.012$ to $0.045$. Its
pre-source response reaches $0.52$ of its peak. The size-matched $C=4$ control
is less accurate still, with static error $0.126$ and pre-source response
$0.91$ of its peak.

\begin{table}[!ht]
\centering
\small
\resizebox{\linewidth}{!}{\begin{tabular}{lrcccccc}
\toprule
& params & static & straight & serpentine & $\Delta/2$ & $\Delta$ (control) & $2\Delta$ \\
\midrule
GObO ($e^{-\lambda_k(t-s)}$, 16 rates) & 22.16\,M & $0.011\pm0.001$ & $0.012\pm0.000$ & $0.005\pm0.000$ & $0.010\pm0.000$ & $0.011\pm0.000$ & $0.011\pm0.000$ \\
separable space--time kernel, $C=16$ & 22.67\,M & $0.035\pm0.001$ & $0.045\pm0.002$ & $0.051\pm0.008$ & $0.031\pm0.001$ & $0.031\pm0.001$ & $0.033\pm0.001$ \\
\bottomrule
\end{tabular}}
\caption{Separable space--time basis. Relative error \eqref{eq:evaluation-metric}, mean $\pm$
population standard deviation over 3 seeds, 128 samples per set. The last 3
columns use the same independently generated frame-spacing test set.}
\label{tab:impl-spacetime}
\end{table}

The extra parameters of the separable model do not compensate for replacing
the lag-dependent kernel and do not provide a reusable streaming state.

\subsubsection{Architecture and Medium-Pathway Ablations}
\label{app:result-medium-pathway}

Table~\ref{tab:impl-medium} shows that replacing the learned volume encoder by
fixed medium projections increases static error to $0.103$ at a similar
parameter count. An MLP on the same projections reaches $0.020$ but uses
$137.23$ million parameters. Reducing GObO's basis to $16\times16$ gives static
error $0.009$, while an $8\times8$ basis gives $0.019$. The comparison therefore
cannot be explained by basis size alone.

\begin{table}[!ht]
\centering
\small
\resizebox{\linewidth}{!}{\begin{tabular}{lrccccccc}
\toprule
& params & static & straight & serpentine & $\Delta/2$ & $\Delta$ (control) & $2\Delta$ & ms/frame \\
\midrule
GObO (cosine spectral conditioning) & 22.16\,M & $0.011\pm0.001$ & $0.012\pm0.000$ & $0.005\pm0.000$ & $0.010\pm0.000$ & $0.011\pm0.000$ & $0.011\pm0.000$ & 0.13 \\
medium projections $+$ U-Net & 22.73\,M & $0.103\pm0.012$ & $0.102\pm0.011$ & $0.068\pm0.011$ & $0.110\pm0.013$ & $0.110\pm0.013$ & $0.111\pm0.013$ & 0.14 \\
\bottomrule
\end{tabular}}
\caption{Medium pathway. Relative error, mean $\pm$ population standard deviation over 3
seeds, 128 samples per set. Columns as in Table~\ref{tab:impl-spacetime},
ms/frame is measured per streaming step.}
\label{tab:impl-medium}
\end{table}

The learned volume-to-manifold pathway is therefore important independently
of the retained basis size.

\subsubsection{Controlled Curved-Surface Recovery}
\label{app:result-curved-recovery}

The recovered kernel has slice errors $9.4\times10^{-3}$,
$4.3\times10^{-3}$, and $1.6\times10^{-3}$ at lags $0.05$, $0.075$, and
$0.125$. Its held-out error on 64 Gaussian chart sources is $0.027$.
Figure~\ref{fig:toy} compares the geometry and a recovered kernel slice with
the reference.

\begin{figure}[ht]
\centering
\begin{subfigure}[b]{0.32\textwidth}
  \centering\includegraphics[width=\linewidth]{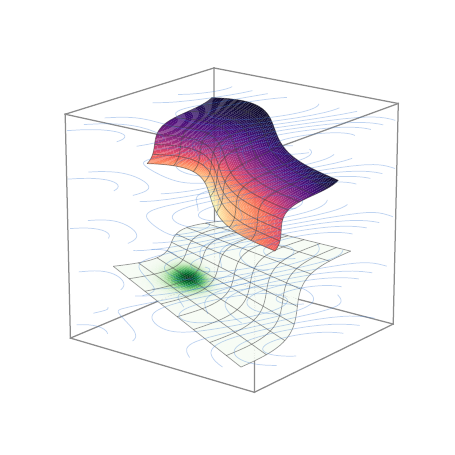}
  \caption{source and observation manifolds}
\end{subfigure}\hfill
\begin{subfigure}[b]{0.32\textwidth}
  \centering\includegraphics[width=\linewidth]{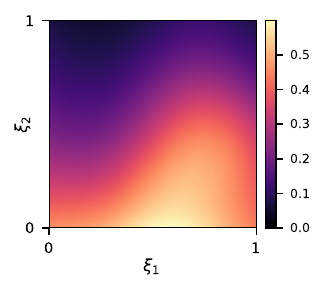}
  \caption{reference kernel slice}
\end{subfigure}\hfill
\begin{subfigure}[b]{0.32\textwidth}
  \centering\includegraphics[width=\linewidth]{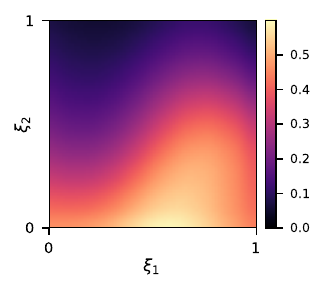}
  \caption{recovered kernel slice}
\end{subfigure}
\caption{\textbf{Kernel recovery on distinct curved manifolds.} The
source lies on a saddle and the response is observed on a separate curved
surface. The final panels compare the reference and recovered restricted
kernel at $\tau=0.05$.}
\label{fig:toy}
\end{figure}

The agreement across separated curved surfaces provides a controlled check
that the factorization can recover the intended restricted kernel.

\subsubsection{Factor-Rank Sweep}
\label{app:result-rank}

Table~\ref{tab:rank-full} reports the complete search-rank sweep for the
controlled curved-surface recovery problem and compares each factorized fit
with the coefficient-space optimum.

At $L=1$, the data-fit objective remains more than an order of magnitude above
the coefficient-space optimum. For $L\geq8$, it is within approximately one
percent of that optimum and has reached its elbow. The recovered rank remains
7 from $L=8$ through $L=64$. At $L=256$, additional singular values lie at the
reporting threshold and increase the reported rank without improving the
objective. We use $L=16$ in the main experiments to remain beyond the elbow.
The held-out error for this separate optimization experiment is $0.106$ on 16
sources.

\begin{table}[!b]
\centering
\footnotesize
\setlength{\tabcolsep}{3pt}
\begin{tabular}{lccc}
\toprule
search rank $L$ & objective & recovered rank & coefficient error \\
\midrule
coefficient optimum & $1.804\times10^{-4}$ & $7$ & n/a \\
$1$ & $2.593\times10^{-3}$ & $1$ & $0.986$ \\
$2$ & $8.714\times10^{-4}$ & $2$ & $0.730$ \\
$4$ & $2.563\times10^{-4}$ & $4$ & $0.380$ \\
$8$ & $1.826\times10^{-4}$ & $7$ & $0.215$ \\
$16$ & $1.788\times10^{-4}$ & $7$ & $0.160$ \\
$32$ & $1.801\times10^{-4}$ & $7$ & $0.120$ \\
$64$ & $1.810\times10^{-4}$ & $7$ & $0.095$ \\
$256$ & $1.808\times10^{-4}$ & $18^{\ast}$ & $0.039$ \\
\bottomrule
\end{tabular}
\caption{\textbf{Complete search-rank sweep.} Recovered rank uses the
threshold $10^{-4}\sigma_1$, coefficient error is
$\lVert A-A_{\mathrm{cvx}}\rVert_F/\lVert A_{\mathrm{cvx}}\rVert_F$.
The objective reaches its elbow at $L=8$, and $L=16$ is used in the heat
experiments. At $L=256$, the additional singular values lie at the reporting
threshold, making the rank count threshold-sensitive.}
\label{tab:rank}
\label{tab:rank-full}
\end{table}

The stable rank over the range around the elbow motivates the conservative
$L=16$ choice used in the learned heat kernels.

\subsection{GObO enables rapid source evaluation and nonlinear corrections}
\label{app:TODO-reuse-results}
\label{app:nonlinear-results}
\label{sec:exp-kirchhoff}

\subsubsection{Repeated-Source Timing}
\label{app:result-timing}
\label{app:result-complexity}

Table~\ref{tab:timing-hw} reports the operation-level measurements. All
interquartile ranges are within about $3\%$ of the median. Dense and streaming
evaluation agree to $5.2\times10^{-7}$. The matched-accuracy finite-volume
setting has relative error $0.0096$ against the production solve.

\begin{table}[ht]
\centering
\footnotesize
\setlength{\tabcolsep}{4pt}
\begin{tabular}{llccc}
\toprule
operation & frequency & CPU & GPU $b=1$ & GPU $b=64$/sample \\
\midrule
FV assembly and preconditioner & per medium & $0.6$ s & n/a & n/a \\
FV production trajectory & per source & $63.5$ s & n/a & n/a \\
FV matched-accuracy trajectory & per source & $4.6$ s & n/a & n/a \\
FNO-2D forward pass & per source & $194.1$ ms & $2.50$ ms & $0.38$ ms \\
GObO conditioning & per medium & $218.3$ ms & $10.1$ ms & $4.29$ ms \\
GObO batch response & per source & $1.45$ ms & $0.45$ ms & $0.008$ ms \\
GObO streaming response & per source & $57.0$ ms & $25.9$ ms & $0.41$ ms \\
\bottomrule
\end{tabular}
\caption{\textbf{Operation-level timing.} A batch of 64 sources shares
one conditioned GObO, the finite-volume solver was evaluated on CPU only.}
\label{tab:timing-hw}
\end{table}

The one-time conditioning cost is amortized over every later source.
Table~\ref{tab:timing-complexity} separates this one-time work from dense and
streaming queries in the implemented asymptotic costs.

\begin{table}[ht]
\centering
\footnotesize
\begin{tabular}{lll}
\toprule
method & time & memory \\
\midrule
finite-volume solver & $O(N_tn_{\mathrm{sub}}I_{\mathrm{CG}}N^3\log N)$ & $O(N^3)$ \\
FNO-2D & $O(PN^2(C^2+C\log N))$ & $O(N^2C)$ \\
GObO conditioning & $O(N^3)$ & operator factors \\
GObO batch & $O(N_t[N^2\log N+RL(D_S+D_O)]+N_t^2RL)$ & $O(N_tN^2+N_t^2R)$ \\
GObO streaming & $O(N_t[N^2\log N+RL(D_S+D_O)])$ & $O(N^2+RL)$ \\
\bottomrule
\end{tabular}
\caption{\textbf{Implemented asymptotic costs.} Streaming removes the
quadratic dependence on history length.}
\label{tab:timing-complexity}
\end{table}

Together, the measurements and costs show that conditioning is paid once and
that later sources use the explicit kernel without another network pass.

\subsubsection{Resolution Transfer}
\label{app:result-resolution-transfer}
\label{app:result-positional-encoding}
\label{sec:exp-res}

Table~\ref{tab:res} shows that coarse computation followed by fine synthesis
and coarse conditioning followed by a fine-grid query both reach the $0.026$
reference discretization gap. The second path is specific to GObO because the
direct FNO consumes the source before its learned layers.

\begin{table}[ht]
\centering
\footnotesize
\begin{tabular}{lcc}
\toprule
mode & GObO & FNO-2D \\
\midrule
M0. compare truths at 64 and 128 & \multicolumn{2}{c}{$0.026$} \\
M1. end-to-end at 128 & $0.078 \pm 0.022$ & $0.044 \pm 0.002$ \\
M2. coarse compute, fine readout & $\mathbf{0.025 \pm 0.000}$ & $0.059 \pm 0.000$ \\
M3. condition at 64, operator at 128 & $\mathbf{0.025 \pm 0.000}$ & n/a \\
\bottomrule
\end{tabular}
\caption{\textbf{Resolution transfer $64\to128$} (relative $L_2$
against truth at 128, 3 seeds, 64 paired samples). M2 and M3 reach
the discretization floor, M3 is structurally unavailable to FNO.}
\label{tab:res}
\end{table}

The numerical values separate end-to-end fine-grid evaluation from reuse of
the coarse-grid operator. Figure~\ref{fig:res} makes this distinction explicit
by tracing where the fine-grid source enters each model.

\begin{figure}[ht]
\centering
\includegraphics[width=0.48\textwidth]{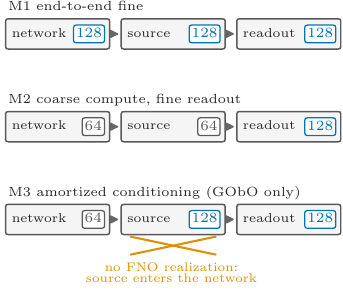}
\caption{\textbf{Resolution-transfer paths.} The source enters FNO
before its learned layers, so there is no FNO analogue of the GObO-only
M3 path.}
\label{fig:res}
\end{figure}

Table~\ref{tab:indexpe} shows that the end-to-end fine-grid path is sensitive
to how coordinates are represented. The operator-reuse paths do not evaluate
the encoder on the fine grid and are unaffected by this choice.

\begin{table}[ht]
\centering
\footnotesize
\begin{tabular}{lcc}
\toprule
M1 positional encoding & GObO & FNO-2D \\
\midrule
raw cell-index features & $0.292 \pm 0.056$ & $0.047 \pm 0.001$ \\
continuous reference features & $0.078 \pm 0.022$ & $0.044 \pm 0.002$ \\
\bottomrule
\end{tabular}
\caption{\textbf{Positional-encoding ablation for end-to-end fine
evaluation.} M2 and M3 do not run the encoder on the fine grid and are
unaffected by this choice.}
\label{tab:indexpe}
\end{table}

The grid-independent basis evaluation, rather than end-to-end fine-grid
inference, is what makes the strongest transfer modes reach the numerical
discretization floor.

\subsubsection{Radiative Feedback}
\label{app:result-radiative-feedback}

Both the explicit and linearized implicit corrections have relative error
$0.027\pm0.001$ over 128 samples and 3 seeds. Their absolute errors are
$6.7\times10^3$ and $6.5\times10^3$. The uncorrected linear trace ranges from
$0.002$ to $0.31$ relative error across samples. Figure~\ref{fig:exp2-app}
shows that the correction removes the peak overshoot and follows the final
profile.

\begin{figure}[ht]
\centering
\begin{subfigure}[b]{0.45\textwidth}
  \centering\includegraphics{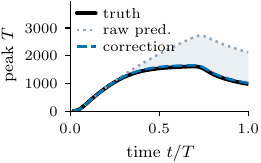}
  \caption{peak temperature over time}
\end{subfigure}\hfill
\begin{subfigure}[b]{0.45\textwidth}
  \centering\includegraphics{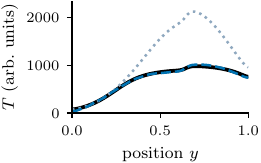}
  \caption{hottest-row profile at $t=T$}
\end{subfigure}
\caption{\textbf{Radiative feedback with the original linear
checkpoint.} (a)~Peak surface temperature over time, the shaded band
between the unchanged linear kernel and the truth is the heat that
radiation removes. (b)~Temperature at $t=T$ along the hottest row. The
unchanged kernel overshoots the peak, whereas the streamed semi-implicit
correction captures both saturation and the final profile without
retraining.}
\label{fig:exp2-app}
\end{figure}

The corrected trajectories show that feedback can be added around the frozen
linear kernel without retraining it on radiative data.

\subsubsection{Temperature-Dependent Conduction}
\label{app:result-temperature-conduction}

Table~\ref{tab:exp3} shows that the exact transform in configuration A keeps
GObO near its linear accuracy across all contrasts and shifted sources. The
trace-based update in configuration B also improves every tested setting
without retraining.

\begin{table}[ht]
\centering
\scriptsize
\setlength{\tabcolsep}{2.6pt}
\begin{tabular}{lccccccc}
\toprule
& & \multicolumn{2}{c}{mild $0.2$} & \multicolumn{2}{c}{strong $0.5$} & \multicolumn{2}{c}{severe $1.0$} \\
\cmidrule(lr){3-4}\cmidrule(lr){5-6}\cmidrule(lr){7-8}
method & retraining & static & shifted & static & shifted & static & shifted \\
\midrule
\multicolumn{8}{l}{\emph{Configuration A. exact Kirchhoff transform}} \\
FNO trained for each setting & 6 & $\mathbf{0.001}$ & $0.104$ & $\mathbf{0.001}$ & $0.116$ & $\mathbf{0.001}$ & $0.103$ \\
linear FNO + transform & 0 & $\mathbf{0.001}$ & $0.096$ & $0.002$ & $0.107$ & $0.002$ & $0.099$ \\
GObO, no correction & 0 & $0.053$ & $0.069$ & $0.122$ & $0.151$ & $0.232$ & $0.263$ \\
GObO + exact transform & 0 & $0.011$ & $\mathbf{0.009}$ & $0.010$ & $\mathbf{0.009}$ & $0.010$ & $\mathbf{0.009}$ \\
\midrule
\multicolumn{8}{l}{\emph{Configuration B. trace-based state update}} \\
FNO trained for each setting & 6 & $\mathbf{0.001}$ & $0.101$ & $\mathbf{0.001}$ & $0.104$ & $\mathbf{0.001}$ & $0.108$ \\
GObO, no correction & 0 & $0.025$ & $0.031$ & $0.054$ & $0.064$ & $0.096$ & $0.115$ \\
GObO + trace-based update & 0 & $0.014$ & $\mathbf{0.012}$ & $0.021$ & $\mathbf{0.019}$ & $0.030$ & $\mathbf{0.030}$ \\
\bottomrule
\end{tabular}
\caption{\textbf{Temperature-dependent conduction without retraining.}
Relative errors are means over 3 seeds. Every standard deviation is at most
$0.002$ for the GObO rows and at most $0.007$ for the retrained FNO. Shifted sources include scans, rasters, and sums of 2
pulses.}
\label{tab:exp3}
\end{table}

The table establishes the trend over all 3 conductivity contrasts and both
source classes. Figure~\ref{fig:exp3} shows how the corrections change the
spatial temperature profile for severe examples rather than only reducing an
aggregate norm.

\begin{figure}[ht]
\centering
\begin{subfigure}[b]{0.47\textwidth}
  \centering\includegraphics[width=\linewidth]{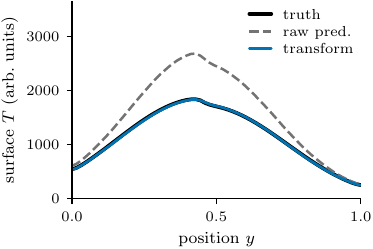}
  \caption{configuration A. exact Kirchhoff transform}
\end{subfigure}\hfill
\begin{subfigure}[b]{0.47\textwidth}
  \centering\includegraphics[width=\linewidth]{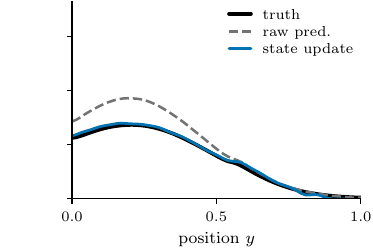}
  \caption{configuration B. trace-based state update}
\end{subfigure}
\caption{\textbf{2 nonlinear material laws corrected with 1 linear
checkpoint.} Surface temperature along the hottest row at $t=T$ for
the severe-contrast tier ($k_{\max}/k_{\min}=2$), 1 static test
sample per configuration. Each panel overlays the truth, the uncorrected
GObO prediction, and the corrected prediction. Panel~(a) uses the exact
Kirchhoff transform and panel~(b) uses the trace-based state update. Both panels
share axes, line styles, and temperature scale. Table~\ref{tab:exp3}
quantifies the contrast trends.}
\label{fig:exp3}
\end{figure}

Both constructions preserve the source-transfer benefit of the linear model
while correcting the nonlinear material response.

\subsubsection{Fixed-Point Comparison}
\label{app:result-fixed-point}

Table~\ref{tab:exp3_picard} separates convergence of the outer iteration from
the surrogate's accuracy on the effective sources generated by that iteration.
GObO improves substantially under both residual constructions, while the
direct FNO remains inaccurate on shifted sources.

\begin{table}[ht]
\centering
\scriptsize
\setlength{\tabcolsep}{2.5pt}
\begin{tabular}{llccccc}
\toprule
model & residual/state & steps & mild static & mild shifted & strong static & strong shifted \\
\midrule
FNO trained for setting & n/a & 0 & $0.001$ & $0.101$ & $0.001$ & $0.104$ \\
linear FNO + iteration & time derivative & $\le20$ & $0.011\,(2)$ & $0.097\,(29)$ & $0.026\,(6)$ & $0.093\,(26)$ \\
linear GObO + iteration & time derivative & $\le20$ & $0.011\,(3)$ & $\mathbf{0.009}\,(2)$ & $0.015\,(3)$ & $\mathbf{0.014}\,(3)$ \\
linear GObO + iteration & generator & $\le20$ & $0.011\,(3)$ & $\mathbf{0.009}\,(2)$ & $0.016\,(4)$ & $\mathbf{0.014}\,(3)$ \\
GObO + trace-based update & generator state & 1 & $0.014$ & $0.012$ & $0.021$ & $0.019$ \\
FV solver + iteration & time derivative & $\le20$ & $0.002\,(1)$ & n/a & $0.007\,(3)$ & n/a \\
\bottomrule
\end{tabular}
\caption{\textbf{Fixed-point attribution for configuration B.}
Entries are median relative errors over non-diverged samples,
parentheses give half the interquartile range in units of the last digit.}
\label{tab:exp3_picard}
\end{table}

The iteration results attribute the remaining shifted-source error to the
surrogate rather than to failure of the nonlinear outer solve.

\end{document}